\documentclass[twoside,11pt]{article}

\usepackage[preprint]{jmlr2e}

\jmlrheading{1}{2020}{1-48}{12/20}{12/20}{manita20}{Oxana A.\ Manita, Mark A.\ Peletier, Jacobus W.\ Portegies, Jaron Sanders, and Albert Senen--Cerda.\\ All authors have contributed equally}

\ShortHeadings{Universal Approximation in Dropout Neural Networks}{Manita, Peletier, Portegies, Sanders, Senen--Cerda}
\firstpageno{1}

\usepackage[utf8]{inputenc}
\usepackage[english]{babel}
\usepackage{amssymb}
\usepackage{amsmath}
\usepackage{bbm}
\usepackage{pgfplots}
\usepackage{stackengine}
\usepackage{subcaption}
\usepackage{tikz}
\usepackage{todonotes}
\usepackage{upgreek}
\usepackage{xcolor}

\usepackage{hyperref}
\hypersetup{
	colorlinks   = true, 
	urlcolor     = blue, 
	linkcolor    = blue, 
	citecolor   = red 
}

\usepackage{setspace}
\usepackage{etoolbox}
\AtBeginEnvironment{quote}{\singlespacing\small}

\newcommand{\R}[1]{\mathbb{R}^{#1}}

\newcommand{\DNN}{\mathsf{DDNN}}

\newcommand{\children}{\mathsf{children}}

\newcommand{\inp}{\mathsf{Inp}}
\newcommand{\propprob}{\mathsf{ApProp}}
\newcommand{\indnn}{\mathsf{NN}}
\newcommand{\into}{\mathsf{into}}
\newcommand{\target}{\mathsf{target}}
\newcommand{\source}{\mathsf{source}}

\newcommand{\level}{\mathsf{level}}
\newcommand{\subnorm}{\mathcal{F}} 
\newcommand{\transfer}{\mathsf{In}}
\let\wh\widehat
\newcommand{\E}[1]{\mathbb{E}_{#1}}
\DeclareMathOperator{\Var}{Var}

\newcommand{\intint}[1]{\overline{1,#1}}

\usepackage{dsfont} 

\newcommand{\expectation}[1]{ \mathbb{E} [ #1 ] }

\newcommand{\indicator}[1]{ \mathds{1} [ #1 ] }

\newcommand{\pnorm}[2]{ \| #1 \|{}_{#2} }

\newcommand{\probability}[1]{ \mathbb{P} [ #1 ] }

\newcommand{\probabilityBig}[1]{ \mathbb{P} \Bigl[ #1 \Bigr] }

\newcommand{\realNumbers}{ \mathbb{R} }

\newcommand{\refFigure}[1]{{\textnormal{Figure~\ref{#1}}}}

\newcommand{\refTheorem}[1]{{\textnormal{Theorem~\ref{#1}}}}

\newcommand{\refCorollary}[1]{{\textnormal{Corollary~\ref{#1}}}}
\newcommand{\refProposition}[1]{{\textnormal{Proposition~\ref{#1}}}}
\newcommand{\refLemma}[1]{{\textnormal{Lemma~\ref{#1}}}}

\newcommand{\QuodEratDemonstrandum}{\hfill \ensuremath{\Box}}

\def\eqcom#1{\overset{\textnormal{(#1)}}}

\def\E{{\mathbb E}}

\def\({{\Bigl(}}
\def\){{\Bigr)}}

\newcommand{\ba}{\begin{array}}
\newcommand{\ea}{\end{array}}

\newcommand{\xdeleted}[1]{\deleted{}} 

\usepackage[acronym]{glossaries}
\makeglossaries

\newacronym{ReLU}{ReLU}{Rectified Linear Unit}

\begin{document}

\title{Universal Approximation in Dropout Neural Networks}

\author{\name Oxana A.\ Manita \email o.zaal.manita@tue.nl \\
	\name Mark A.\ Peletier \email m.a.peletier@tue.nl \\
	\name Jacobus W.\ Portegies \email j.w.portegies@tue.nl \\	
	\name Jaron Sanders \email jaron.sanders@tue.nl \\
	\name Albert Senen--Cerda \email a.senen.cerda@tue.nl \\
	\addr Department of Mathematics \& Computer Science\\
	Eindhoven University of Technology\\
	Eindhoven, The Netherlands
}

\editor{~}

\maketitle

\begin{abstract}
	We prove two universal approximation theorems for a range of dropout neural networks. These are feed-forward neural networks in which each edge is given a random $\{0,1\}$-valued filter, that have two modes of operation: in the first each edge output is multiplied by its random filter, resulting in a random output, while in the second each edge output is multiplied by the expectation of its filter, leading to a deterministic output. It is common to use the random mode during training and the deterministic mode during testing and prediction. 
	
	Both theorems are of the following form: Given a function to approximate and a threshold $\varepsilon>0$, there exists a dropout network that is $\varepsilon$-close in probability and in $L^q$. The first theorem applies to dropout networks in the random mode. It assumes little on the activation function, applies to a wide class of networks, and can even be applied to approximation schemes other than neural networks. The core is an algebraic property that shows that deterministic networks can be exactly matched in expectation by random networks. 	
	The second theorem makes stronger assumptions and gives a stronger result. Given a function to approximate, it provides existence of a network that approximates in both modes simultaneously. Proof components are a recursive replacement of edges by independent copies, and a special first-layer replacement that couples the resulting larger network to the input.
	
	The functions to be approximated are assumed to be elements of general normed spaces, and the approximations are measured in the corresponding norms. The networks are constructed explicitly. Because of the different methods of proof, the two results give independent insight into the approximation properties of random dropout networks. With this, we establish that dropout neural networks broadly satisfy a universal-approximation property.
\end{abstract}


\begin{MSC2020}
	41A30; 68T05.
\end{MSC2020}

\section{Introduction}

\subsection{Universal approximation in Neural Networks}

The class of functions that are generated by neural networks satisfies a `universal approximation property': any given function can be approximated to arbitrary precision by such a neural network \citep{cybenko1989approximation,leshno1993multilayer}. This property partially explains why neural networks are so effective as approximators of implicitly given functions. 

It is commonly observed that the training of such networks improves upon introducing \emph{dropout}, the random `dropping' of edges \cite[Sec.~7.9]{GoodfellowBengioCourville16}. Dropout converts a deterministic network into a random one.  In this paper we address the question that this observation implicitly raises: Does this randomness interfere with the universal approximation property? Or, to formulate it in the affirmative: does the class of dropout neural networks still satisfy universal approximation?

To provide a first quantification of this question, let us explain the expec\-ta\-tion--variance split, which in the context of dropout goes back to the theoretical analysis by \cite{wager2013dropout}. We will think of a dropout neural network as a function $\Psi: \mathbb{R}^d \times \mathbb{R}^n \to \mathbb{R}$ together with a $\{0,1\}^n$-valued random variable $f$. We call the components of $f$ filter variables. We think of $\mathbb{R}^d$ as data space and $\mathbb{R}^n$ as parameter space (the space of weights and biases for the neural network). The parameters get multiplied elementwise with the vector of filter variables $f$. That means that when $\zeta: \mathbb{R}^d \to \mathbb{R}$ is a function we want to approximate, we try to approximate it with the stochastic function that maps $x$ to $\Psi(x, w \odot f)$. For fixed $x$, the expec\-ta\-tion--variance split reads 
\begin{equation}
	\mathbb{E}\left[\bigl(\Psi(x, w \odot f) - \zeta(x) \bigr)^2\right]
	= 
	\bigl(\mathbb{E}[ \Psi(x, w \odot f) ] - \zeta(x)\bigr)^2 
	+ \mathbb{V}[ \Psi(x, w \odot f) ].
	\label{eqn:bias_variance_decomposition}
\end{equation}
Here $\odot$ denotes element-wise multiplication, and throughout the paper $\mathbb{P}$, $\mathbb{E}$, $\mathbb{V}$ stand for probability, expectation, variance with respect to the filter variables $f$, respectively. As both terms on the right-hand side are nonnegative, both terms have to be small in order to have a good approximation. Is this possible?

\cite{foong2020expressiveness} showed that deep \gls{ReLU} neural networks with node-dropout still can approximate functions arbitrary well, by showing that both the expectation term and the variance term in the expec\-ta\-tion--variance split can be made arbitrary small (see their Theorem 3). In fact, the two terms are arbitrary small uniformly over $x$ in the unit cube in $\mathbb{R}^d$. With this statement, Foong et al.\ effectively showed a universal-approximation result.

In this paper we show two universal-approximation results for wider classes of dropout neural networks.  Where Foong et al.\ made specific use of the properties of \gls{ReLU} activation, we show that the property of universal approximation does not require more of the activation functions in dropout networks than in the underlying deterministic ones \citep{leshno1993multilayer}. In addition, our main theorems allow for very general classes of filters, including the original node-based dropout \citep{hinton2012improving}, the edge-based dropconnect \citep{wan2013regularization}, and many others. With Theorem~\ref{thm:approximation_in_probability-intro} below we show that the class of dropout networks can \emph{exactly} match a given deterministic network, at least in expectation. With Theorem~\ref{th:main-average} below we show that we can construct networks that approximate a given function arbitrarily well, both as a random network and as a deterministic filtered network. Finally, we provide control of the error both in probability and in~$L^q$.

\subsection{Approximation by random neural networks}

In a deterministic context, a universal-approximation theorem for some class $\mathsf C$ is a density statement, stating that any function $\zeta$ can be approximated to arbitrary precision by neural networks in $\mathsf C$, where the approximation is measured in some seminormed  function space $(\mathcal{F}, \| \cdot \|_{\subnorm})$. 
Such approximation statements can be generalized to a stochastic context in multiple ways. We will focus on two of these, approximation in probability and  in $L^q$ for $q \in [1, \infty)$.

Universal approximation in probability is the property that for every function $\zeta \in \mathcal{F}$ to approximate, and every $\epsilon > 0$, there exists a neural network $\Psi$, a weight vector $w$ and a random vector $f$ (all with certain extra properties to make the statement nontrivial), such that 
\[
	\mathbb{P} \left[ \|\zeta(\cdot) - \Psi( \cdot , w \odot f ) \|_{\subnorm} > \epsilon \right] 
	< 
	\epsilon.
\]
A stronger statement involves approximation in $L^q$ for $q \in [1, \infty)$: there exist $\Psi$, $w$ and $f$ such that
\[
	\mathbb{E} \left[ \left\|\zeta(\cdot) - \Psi(\cdot, w \odot f) \right\|_{\subnorm}^q \right]^{\frac{1}{q}} 
	<
	\epsilon.
\]
In this article we indeed show such universal approximation statements for certain classes of deep dropout neural networks. 

We prove two main classes of approximation results, corresponding to the two main ways that dropout networks are used in practice. In the first class of results  the network is a random object as described above, and the same network is used during training and prediction; we call this \emph{random-approximation} dropout.\footnote{This has also been called \emph{Monte Carlo dropout} because of the close connection with Monte Carlo estimation of integrals \citep{gallicchio2020deep}.} In the second class of results the network is trained with random networks of the form $\Psi(\cdot,w\odot f)$, but for prediction the deterministic network $\Psi(\cdot,w\odot \E f)$ is used in which the filters are replaced by their expectations. We call this type of dropout \emph{expectation-replacement}.

\subsection{Main results 1: Random-approximation}
\label{sec:intro:Rand-Approx}

We start with uniform \emph{random-approximation}, that is the property that any function $\zeta$ in an appropriate set can be approximated  by random networks of the form $\Psi(\cdot,w\odot f)$.
At the highest level the proof strategy is the same as in Foong et al.\, and consists of the following three steps. Given a function $\zeta\in \mathcal F$ to be approximated:
\begin{enumerate}
	\item \label{steps:1}
	Approximate $\zeta$ by a neural network $\Psi(\cdot,w)$ using classical deterministic universal approximation results (e.g. \cite{leshno1993multilayer});
	\item \label{steps:2}
	Use $\Psi(\cdot,w)$ to construct  a larger, random dropout neural network $\tilde \Psi(\cdot,\tilde w\odot \tilde f)$ that matches  $\Psi(\cdot,w)$ in expectation;
	\item \label{steps:3}
	Construct an even larger random neural network $\wh \Psi(\cdot,\wh w\odot \wh f)$ consisting of many independent copies of the network $\tilde \Psi(\cdot,\tilde w\odot \tilde f)$ to obtain an approximation of $\zeta$ that is close in expectation and also has small variance. 
\end{enumerate}
We consider Step~\ref{steps:1} as given by existing results, and Step~\ref{steps:3} is a standard procedure. The novelty of this paper for \emph{random-approximation} lies in Step~\ref{steps:2}, which we describe in the rest of this section.

Step~\ref{steps:2} is based on an {algebraic} property of common classes of neural networks, which is illustrated by the following simpler version of the central theorem. We write $2^n$ for the collection of subsets of $\intint n = \{1,\dots,n\}$, and for any such a subset $U$, we write $\mathbf{1}_U \in \{0,1\}^n$ for the vector with entries $(1_{j \in U})_{j \in \intint{n}}$.

\begin{theorem}
\label{thm:approximation_in_probability-intro}
	Let $\Psi : \mathbb{R}^d \times \mathbb{R}^n \to \mathbb{R}$ be any given function.  Let $(f^U)_{U \in 2^n}$ be a collection of $\{0,1\}^n$-valued random variables indexed by subsets $U \in 2^n$ such that for every $U$
	\[
	\mathbb{P}[f^U = (1,\dots, 1)] > 0.
	\]
	Then there exist constants $(a_U)_{U \in 2^n}$, independent of $w$, such that for all $w$,
	\begin{equation}
	\mathbb{E}\left[\sum_{U \in 2^n} a_U \Psi\bigl( \cdot, (w \odot \mathbf{1}_U) \odot f^U\bigr)\right] = \Psi(\cdot, w).
	\label{eqn:Decomposition_of_Psi_in_a_linear_combination_of_expectations-intro}
	\end{equation}
\end{theorem}

This theorem should be read as follows. The right-hand side in \eqref{eqn:Decomposition_of_Psi_in_a_linear_combination_of_expectations-intro} plays the role of a deterministic function that we want to approximate. The left-hand side is the expectation of a linear combination of many copies of $\Psi(\cdot,w)$. Each copy has two `dropout' modifications: the vector $\mathbf 1_U$ implements a deterministic dropout, and the random filter variables $f^U$ a stochastic one. With a view to generality, the random filter vector $f^U$ is allowed to be a different random vector for each subset $U$ of edges, but note that the distribution of $f^U$ on $\{0,1\}^n$ can be completely unrelated to the subset $U\subset \intint n$; the subset $U$ only serves as label. 

The theorem establishes the following fact: for \emph{any} collection of random filter variables $f^U$, for \emph{any} function $\Psi$, for \emph{any} parameter point $w$, the function $\Psi(\cdot,w)$ can be matched exactly by the expectation of a sum of filtered versions of the same function. The important caveat is that one needs to take into account all reduced versions of the functions $\Psi$, i.e., the whole hierarchy of deterministically modified versions indexed by subsets $U$.

\medskip
Theorem~\ref{thm:approximation_in_probability-intro} suggests a special role for `classes of networks', with the property that given a `network' $\Psi(\cdot,w)$ we can   in some sense define a new (random) network $\tilde \Psi(\cdot,\tilde w\odot \tilde f)$ by
\begin{equation}
\label{eqdef:tilde-Psi}	
\tilde \Psi(\cdot,\tilde w\odot \tilde f) := \sum_{U \in 2^n} a_U \Psi( \cdot, w \odot \mathbf{1}_U \odot f^U).
\end{equation}
To formalize this, we assume that we have chosen a set $\DNN$ (a `set of random neural networks'),  which can be any collection of tuples $(n, \Psi, f)$ that satisfy the following properties:
\begin{itemize}
	\item[(i)] $n \in \mathbb{N}$ is a natural number;
	\item[(ii)]  $\Psi : \mathbb{R}^d \times \mathbb{R}^n \to \mathbb{R}$ is a function such that for every $w \in \mathbb{R}^n$, $\Psi(\cdot, w) \in \mathcal{F}$;
	\item[(iii)]  $f$ is a $\{0, 1\}^n$-valued random variable such that 
\begin{equation}
	\mathbb{P}[ f = (1, \dots, 1) ]
	>
	0.
	\label{eqn:probability_filters_on_is_positive}
\end{equation}
\end{itemize}
Moreover, we assume that $\DNN$ is closed under linear, independent combinations. By this we mean that whenever $a, b \in \mathbb{R}$ and $(m, \Phi, f)$ and $(n, \Psi, g)$ are in $\DNN$, then also $(m + n, a \Phi + b \Psi, h) \in \DNN$ where $h$ is an $\{0,1\}^{m + n}$-valued random variable that is the independent concatenation of $f$ and $g$, 
and $a \Phi + b \Psi : \mathbb{R}^d \times \mathbb{R}^{m+n} \to \mathbb{R}$ is given by
\[
	(x, (w_1, w_2)) 
	\mapsto a \Phi(x, w_1) + b \Psi(x, w_2).
\]
This closure assumption implies that a definition of the form \eqref{eqdef:tilde-Psi} is meaningful.

The range of possible classes $\DNN$ satisfying these requirements is vast. Typical examples are neural networks with node-dropout, as originally introduced by \cite{hinton2012improving}, and \emph{dropconnect}, as introduced by \cite{wan2013regularization}, but many other choices also are possible. Note that the function $\Psi$ may be extremely general, implying that there are  no restrictions on e.g.\ the form of the activation function or the structure of the network. In fact, nothing in the requirements on $\DNN$ restricts to functions $\Psi$ generated by neural networks; other approximation methodologies may also be used, for instance based on Fourier or wavelet expansions. See Section~\ref{sec:specification_of_DNN}  for a detailed description of $\DNN$.

\medskip

By combining Theorem \ref{thm:approximation_in_probability-intro} with the law of large numbers we then find \refCorollary{cor:main-intro} below, which expresses the following insight:  if the class $\DNN$ is rich enough to approximate any function in $\mathcal{F}$ when all filter variables are set to 1, then any function in $\mathcal{F}$ can also be approximated by a (random) dropout neural network in $\DNN$.
\begin{corollary}
\label{cor:main-intro}
Let $\zeta \in \mathcal{F}$ and $\epsilon > 0$. Assume there exists a $(m, \Phi, g) \in \DNN$ and a $v \in \mathbb{R}^m$ such that $\|\Phi(\cdot, v) - \zeta\|_{\subnorm} < \epsilon/2$. Then there exists a $(n, \Psi, f) \in \DNN$ and a $w \in \mathbb{R}^n$ such that 
\begin{equation}
\mathbb{P}[ \pnorm{\Psi(\cdot, w \odot f) - \zeta }{\subnorm} > \epsilon]
< \epsilon
\end{equation}
and
\[
\mathbb{E} \left[ \left\| \Psi(\cdot, w \odot f) - \zeta\right\|_{\subnorm}^q \right]^{\frac{1}{q}} < \epsilon.
\]
\end{corollary}
\noindent
Section~\ref{sec:Use_of_average_filter_variables_for_prediction} is devoted to these results, but develops them in more generality. There we also give some examples and calculate the coefficients $a_U$ explicitly for the case of independent Bernoulli filters. 

\subsection{Main results 2: Expectation-replacement}
\label{ss:intro:Expectation-Replacement}

In the previous section we considered dropout neural network to be a random object, both during training and during prediction. 
By contrast, it is common practice to choose the filter variables  during prediction to be \emph{deterministic} and equal to their expectations; see e.g.\ Section 7.12 in \cite{GoodfellowBengioCourville16}'s paper. We call this \emph{expectation-replacement} dropout, and Corollary~\ref{cor:main-intro} above does not say anything about this situation. In fact, we show in Example~\ref{ex:MC-bad-regular} that the construction at the heart of Corollary~\ref{cor:main-intro} may lead to networks that are `bad approximators' in this specific sense: given a function $\zeta$,  the constructed networks approximate $\zeta$ with high probability with random filters, but do not approximate $\zeta$ at all when replacing the filters by their expectations.

To address this, we describe in Section~\ref{sec:Use_of_average_filter_variables_for_prediction}
the construction of dropout neural networks that approximate not only in probability and in $L^q$, but also in this expectation-replacement sense. As in the case of Corollary~\ref{cor:main-intro}, the construction builds on existing density results for deterministic networks: we  start with a given deterministic neural network $\Psi(\cdot,w)$ that is close to a given function $\zeta$. Differently from Corollary~\ref{cor:main-intro}, however, the nonlinearity of $\Psi$ forces us to apply the law of large numbers  to each edge (or weight in this context) separately, instead of simultaneously for the whole network. 

In the construction in Section~\ref{sec:Use_of_average_filter_variables_for_prediction} we therefore iteratively replace each edge in the deterministic neural network $\Psi$ by a set of parallel edges, with edge-weights $w$ taken from the original edge, and with independent filter variables on each of them. In this way we can use the law of large numbers to obtain convergence estimates for each edge separately, and then combine these estimates into a single convergence estimate for the whole network. 
  
  The convergence estimate for a single edge arises from  the following statement (which is a simplified version of Lemma~\ref{lem:copy-input}). It describes how the error encountered by averaging  $N$  independently filtered edges can be controlled in probability. At the same time it also allows for small perturbations of the inputs to this edge. This latter perturbation freedom is needed in order to apply this lemma progressively, moving from edge to edge through the network.
  \begin{lemma}
  		Consider any continuous function $\sigma : \mathbb{R}^m \to \mathbb{R}^m$ and let $W \in \mathbb{R}^{m \times n}$, $b \in \mathbb{R}^m$. Let $\{ F^i \}_{i \in \intint{N}}$ be a collection of independent copies of a random matrix $F \in \{0,1\}^{m \times n}$. Then for every $K > 0$ there exists a $\delta > 0$ such that 
	\begin{equation}
			\sup_{x \in \overline{B(0, K)}} 
			\sup_{(\tilde{x}^i) \in B(x, \delta)^N} 
			\biggl|
				\sigma
				\biggl(
					\frac1N \sum_{i=1}^N (W \odot F^i) \tilde{x}^i + b 
				\biggr) 
				- \sigma\Bigl( (W\odot \E F^i) x + b \Bigr) 
			\biggr| 
	\end{equation}
	converges to zero in probability as $N \to \infty$.
  \end{lemma}
  In Section~\ref{sec:Use_of_average_filter_variables_for_prediction} this construction is described in detail. A separate part of this description is how to connect the resulting dropout neural network to the inputs of the original layer; for this we introduce a single additional layer that implements this connection.

 The main Theorem~\ref{th:main-average} allows for a wide range of choices of activation functions and filter-variable distributions. The following is a simple, concrete corollary for a \gls{ReLU} activation function.

\begin{corollary}
	\label{co:dropconnect-relu-average}
	Take $\mathcal{F}$ to be the space of continuous functions $\mathbb{R}^d \to \mathbb{R}$, and endow it with a seminorm $\| \cdot \|_{\subnorm}$ equal to supremum of the function  on the unit cube.
	Then for every $\zeta \in \mathcal{F}$ and every $\epsilon > 0$ there exists a dropconnect \gls{ReLU} neural network $(\Psi, f)$ and a parameter vector $w$ such that
	\begin{equation}
		\mathbb{P}\Bigl[ \Bigl\| \Psi(\cdot, w \odot f) - \zeta \Bigr\|_{\subnorm} > \epsilon \Bigr]
		< \epsilon
	\end{equation}
	and
\[
\mathbb{E} \left[ \left\| \Psi(\cdot, w \odot f) - \zeta\right\|_{\subnorm}^q \right]^{\frac{1}{q}} < \epsilon,
\]
	while
	\[
	\left\|\Psi( \cdot, w \odot \mathbb{E}[f]) - \zeta \right\|_{\subnorm} < \epsilon.
	\]
\end{corollary}

Note that where the construction of the previous section applied to a very wide class of functions $\Psi$---not only those generated by neural networks---the construction underlying Corollary~\ref{co:dropconnect-relu-average} depends in a detailed manner on the fact that $\Psi$ has the structure of a neural network. 

\subsection{Random-approximation \emph{vs.}\ expectation-replace\-ment dropout}

Using a \emph{random} neural network to approximate a given \emph{deterministic} function  is non-trivial; the variance of the network needs to be reduced while matching the expectation, as described in Section~\ref{sec:intro:Rand-Approx}.

In expectation-replacement dropout, however, the networks used in prediction are deterministic,  and this difficulty is absent.
In fact, the difference between training and prediction is the reason we include expectation-replacement in this paper. 

This difference between training and prediction poses an intriguing question. Suppose that the dropout training algorithm yields a parameter point $w^*$. 
During this training, the algorithm has observed random networks $\Psi(\cdot,w\odot f)$, but   it has never observed the deterministic network $\Psi(\cdot,w\odot \mathbb Ef)$. Why, then,  should the result $w^*$ of the dropout training then generate a good deterministic network $\Psi(\cdot,w^*\odot \mathbb Ef)$? Example~\ref{ex:MC-bad-regular} confirms that this method may fail badly. 

At the same time, expectation-replacement dropout  is both very widespread and very successful; see e.g.\ Section 7.12\ in \cite{GoodfellowBengioCourville16}'s paper, or \cite{LabachSalehinejadValaee19}'s paper. How can these two observations be reconciled?

The results of Section~\ref{sec:Use_of_average_filter_variables_for_prediction} and e.g.\ Corollary~\ref{co:dropconnect-relu-average}  provide a partial answer to this question. We show that dropout neural network have sufficient representational capacity to approximate well simultaneously in probability, in $L^q$, \emph{and} in the expectation-replacement sense. While this does not explain why any given training algorithm finds parameter points that approximate well in the expectation-replacement sense, at least it shows that the contrast between random training and deterministic prediction is not an obstacle to good performance.

\subsection{Related literature}
\label{sec:related-literature}

The universal approximation property for neural networks is one of the fundamental properties and essentially  determines whether the whole training process of the network makes sense: if the algorithmically generated functions don't form a dense set in the function space of interest, the approximation problem is ill-posed. Therefore  establishing  the universal approximation property for different classes of networks has been an active research area in the last decades.  However, most classes of networks for which there is a universal approximation property established do not include, for example, node-dropout or dropconnect neural network.s

The first universal approximation theorem for neural networks with a sigmoidal activation function can be found in \cite{cybenko1989approximation}'s paper, and this canonical work led to much follow-up research. Several years later \cite{hornik1991approximation} showed that the universal approximation property relies more on a neural network's architecture than on the specific use of sigmoid activation functions. Moreover, \cite{leshno1993multilayer} established that deep, feed-forward neural networks require a nonpolynomial activation function in order for a universal approximation theorem to hold. 
\cite{makovoz1996random,makovoz1998uniform} used the so-called probabilistic method to prove the existence of a deterministic function that suitably approximates a target  function in deterministic neural networks. 

Approximately at the same time the study of random networks started. \cite{white1989additional}'s paper on ``QuickNet'' is one of the first works where universal approximation is mentioned (but not proved) side by side with a neural network algorithm in which random hidden nodes are placed.

The class of networks with random weights and biases, called Random Vector Functional-Link Nets, was introduced in 1994 by \cite{pao1994learning}.
\cite{igelnik1995stochastic} proved a universal approximation property of these networks, by showing that the span of the node functions is almost surely asymptotically dense in the many-node limit. This result does not apply to dropout schemes since  in the dropout setup the randomness is applied after choosing coefficients.

\cite{gelenbe1999function,gelenbe1999approximation} introduced a class of neural networks that relies on a fixed neural network topology on top of which neurons forward positive and negative signals (spikes) at random points in time based on their own ``potential''. 
Specifically, they gave a constructive proof of the universal approximation theorem for such stochastic neural networks networks in steady state. 
This class of networks also doesn't cover the node-dropout or dropconnect cases due to the different dynamics assumed; moreover a dropout neural network is trained randomly, but typically operated deterministically. 

\cite{rahimi2008uniform} investigated uniform approximation of functions with random bases. This is a particular case of a so-called random feature method, in which  the parameters are split in two groups: parameters in one group are taken randomly (and not tuned), and the other part is trained to achieve best approximation. Therefore these results also don't cover node-dropout or dropconnect since for the latter algorithms all parameters are trained.

Another commonly used class of neural networks is the \emph{mixture of experts} model. The idea is that for different input regions different, typically simpler,  networks (learners) are used for prediction. The choice is performed by the \emph{gating network}; training of the model consists then of training individual learners together with training the gating network. 
\cite{nguyen2016universal} proved a universal approximation theorem for a mixture-of-experts model, and \cite{nguyen2017universal} subsequently generalized their findings to allow for so-called Gaussian gating. 

\cite{debie2018stochastic} considered a network architecture that can handle probability measures as input and output. They proved the universal approximation in Wasserstein metric for continuous maps from the space of measures into itself.  Our results are more specific, and not covered by this result, since we study a different (more restricted) approximation scheme. 

As mentioned in the introduction, \cite{foong2020expressiveness} show a universal approximation property for random-approximation dropout networks (see their Theorem 3).  We recover this result as Corollary~\ref{cor:main-intro} when identifying $h \equiv 0$. Another difference is that we allow for activation functions other than \gls{ReLU} activation functions, and consider a stronger sense of approximation.

Finally, we refer interested readers to the following surveys to fully complete their picture of known results.
A survey of approximation-theoretic problems was written by \cite{pinkus1999approximation}; a recent survey by \cite{Grohs_UA_deep} contains a comparison of approximation properties for finite-width and finite-depth networks.  Several uniform approximation results for random neural networks can be found in   \cite{timotheou2010random}'s Section 5.4. Approximation literature for random neural networks was also summarized by \cite{yin2019random}.

\subsection{Structure of this paper} 

Definitions of dropout neural networks are given in Section~\ref{sec:specification_of_DNN}. 
In Section~\ref{sec:Universal_approximation_for_random_approximation_dropout} we show universal approximation results for random-approximation dropout, whereas Section~\ref{sec:Use_of_average_filter_variables_for_prediction} is devoted to universal approximation results for expectation-replacement dropout.
We discuss our results in Section~\ref{sec:discussion} and conclude in Section~\ref{sec:conclusion}.

\section{Specification of dropout neural networks}
\label{sec:specification_of_DNN}

In the introduction, we considered general functions $\Psi: \mathbb{R}^d \times \mathbb{R}^n \to \mathbb{R}$ together with a $\{0,1\}^n$-valued random variable $f$ (Section~\ref{sec:intro:Rand-Approx}), and more specific functions $\Psi$ that arise from a neural network (Section~\ref{ss:intro:Expectation-Replacement}). In this section, we specify this neural network structure and introduce the corresponding notation.

\subsection{Neural networks}
\label{sec:neural-networks}

We specify a (feedforward) neural network as a special type of parametrized function $\Psi : \mathbb{R}^{d} \times \mathbb{R}^n \to \mathbb{R}$ from an input vector space $\mathbb{R}^{d}$ to $\mathbb{R}$, parametrized by vectors in $\mathbb{R}^n$. The function is special in that it is assumed to be the composition of multiple functions of much simpler type
\begin{equation}
	\Psi( \cdot , w ) = 
	\Psi_L\left( \cdot  , w^{(L)}\right) \circ 
	\Psi_{L-1}\left( \cdot  , w^{(L-1)}\right) \circ
	\cdots \circ
	\Psi_1\left( \cdot  , w^{(1)}\right).
	\label{eqn:definition_Psi}
\end{equation}
Here, $L$ is an integer, the parameter $w$ is the concatenation of the individual parameter vectors $w^{(j)} = (W^{(j)}, b^{(j)})$ for $j = \intint{L}$, which in turn consist of a $d_j \times d_{j-1}$ \emph{weight matrix} $W^{(j)}$ and a bias vector $b^{(j)} \in \mathbb{R}^{d_j}$. We set $d_0 = d$ and $d_L=1$.

In \eqref{eqn:definition_Psi} every $\Psi_j$ is a function from $\mathbb{R}^{d_{j-1}}$ to $\mathbb{R}^{d_{j}}$ given by
\begin{equation}
	\Psi_j(x , w^{(j)}) 
	:= \sigma_j\left( W^{(j)} x + b^{(j)} \right),
	\label{eqn:definition_Psi_j}
\end{equation}
where the function $\sigma_j : \mathbb{R} \to \mathbb{R}$ is called the \emph{activation function}. 
The activation function is applied elementwise.

\subsection{Dropout neural networks}
\label{sec:dropout_neural_networks}

A dropout neural network consists of a neural network $\Psi : \mathbb{R}^{d} \times \mathbb{R}^n \to \mathbb{R}$ as above together with a random vector $f \in \{0, 1\}^n$.
The components of $f$ are called filter variables. The network $\Psi$, the filter variables $f$, and a parameter vector $w \in \mathbb{R}^n$ together  form  a stochastic function from $\R d$ to $\realNumbers$ given by
\[
x \mapsto \Psi(x, w \odot f).
\]

For the constructions later in the article, we recall what we precisely mean by random variables. Throughout the article, $(\Omega, \mathcal{F}_\Omega, \mathbb{P})$ is an arbitrary, rich enough, probability space. Whenever we write \emph{random variable}, \emph{random vector} or \emph{random matrix}, we mean a measurable function defined on this probability space.

\subsubsection{Node-dropout}
\label{sec:node-dropout}

In the original version of dropout filter variables acted on \emph{nodes} of the network \citep{hinton2012improving}. In this paper the filter variables act on edges instead. The original version, which we call \emph{node-dropout}, can be represented in the edge-based version of this paper as follows. 

The filter variables are partitioned into various blocks: filter variables are in the same block if and only if they multiply an element in the same column in the same weight matrix, or they multiply elements of the same bias vector. Filter variables in the same block always attain the same value, i.e., with probability one. Filter variables in different blocks are independent. We will use the convention that filter variables that multiply biases are always on, whereas filter variables that multiply elements of weight matrices are on, i.e., equal to $1$, with probability $1-p$ for some $0 \leq p < 1$. 

We can understand node-dropout from the previous description in the notation of \eqref{eqn:definition_Psi_j}. For any $j = 1, \dots, L$, choose probabilities $p^j$, and let $f_1^{j}, \ldots, f_{d_{j}}^j$ be independent Bernoulli filters with probability $1-p^j$. Let $D_j \in \R{d_j \times d_j}$ be the diagonal matrix with entries $f_1^{j}, \ldots, f_{d_{j}}^j$ in the diagonal. If we then arranging all nodes per block, then node-dropout implements for $j=1, \ldots, L$,
\begin{equation}
	\Psi_j(\cdot , w^{(j)} \odot f^{(j)}) 
	= \sigma_j\left( W^{(j)} D^{j} ( \cdot ) + b^{(j)} \right).
	\label{eqn:definition_Psi_j_node_dropout}
\end{equation}
Note that if $p^1>0$ then with positive probability an input is masked.
For this reason we call the case $p^1>0$ node-dropout \emph{with} dropout on the inputs. We call the case $p^1=0$ node-dropout \emph{without} dropout on the inputs.

\subsubsection{Dropconnect}
\label{sec:dropconnect}

Dropconnect is another dropout regularization scheme \citep{wan2013regularization}. Although \cite{wan2013regularization} also allowed for dropout of biases, we will use the term \emph{dropconnect} for the dropout neural network in which only the matrices $W^{(j)}$ are filtered. This is achieved by choosing the filter variables multiplying the biases to be equal to~$1$ with probability one.

We can understand dropconnect in the notation of \eqref{eqn:definition_Psi_j}. For $j = 1, \ldots, L$, let $F^{(j)} \in \R{d_{j +1} \times d_{j}}$ be random matrices composed of entries $(F^{j})_{ik}$, all of which are mutually independent Bernoulli random variables with the same success probability $1-p$. Dropconnect then implements for $j=1, \ldots, L$,
\begin{equation}
	\Psi_j(\cdot , w^{(j)} \odot f^{(j)}) 
	= \sigma_j\left( (W^{(j)} \odot F^{(j)}) ( \cdot ) + b^{(j)} \right).
	\label{eqn:definition_Psi_j_dropconnect}
\end{equation}

\section{Universal approximation for random app\-rox\-imation dropout}
\label{sec:Universal_approximation_for_random_approximation_dropout}

The aim of this section is to derive the abstract universal approximation statement for random-approximation dropout already mentioned in the introduction (Corollary~\ref{cor:main-intro}).

\subsection{Key approximation result}

The following theorem is  Theorem~\ref{thm:approximation_in_probability-intro} in the introduction, extended with a convergence statement. 
\begin{theorem}
\label{thm:approximation_in_probability}
	Let $(\mathcal{F}, \| \cdot \|_{\subnorm})$ be a seminormed vector space of functions from $\mathbb{R}^d$ to $\mathbb{R}$. Let $\Psi : \mathbb{R}^d \times \mathbb{R}^n \to \mathbb{R}$ be a given function such that $\Psi( \cdot, w) \in \mathcal{F}$ for every $w \in \mathbb{R}^n$. Let $(f^U)_{U \in 2^n}$ be a collection of $\{0,1\}^n$-valued random variables indexed by subsets $U \in 2^n$, such that for every $U$
	\begin{equation}
		\mathbb{P}[f^U = (1,\dots, 1)] 
		> 
		0.
		\label{eqn:probability_all_filters_one_positive}
	\end{equation}
	Then there exist constants $(a_U)_{U \in 2^n}$ independent of $w$ such that
	\begin{equation}
		\mathbb{E}\left[\sum_{U \in 2^n} a_U \Psi( \cdot, (w \odot \mathbf{1}_U) \odot f^U)\right] 
		= 
		\Psi(\cdot, w).
		\label{eqn:Decomposition_of_Psi_in_a_linear_combination_of_expectations}
	\end{equation}
	In particular, by the weak law of large numbers, if $f^{i, U}$ are independent copies of $f^U$, then as $M \to \infty$,
	\begin{equation}
		\frac{1}{M} \sum_{i=1}^M \sum_{U \in 2^n} a_{U} \Psi( \cdot, (w \odot \mathbf{1}_U) \odot f^{i, U})
		\to 
		\Psi(\cdot, w)
		\label{eqn:average_converges_in_prob_and_Lq}
	\end{equation}
	in probability in $(\mathcal{F}, \|\cdot \|_{\subnorm})$ and in $L^q$ for every $q \in [1,\infty)$.
\end{theorem}

A proof of \refTheorem{thm:approximation_in_probability} can be found in Appendix~\ref{secappendix:proof_of_approximation_in_probability}. The main observation in \refTheorem{thm:approximation_in_probability} is the existence of the constants $(a_U)_{U \in 2^n}$.
This purely algebraic statement follows by induction, as explained by Lemma \ref{lemma:function_in_span_dropout_expectations} in Appendix~\ref{secappendix:proof_of_approximation_in_probability}. From \refTheorem{thm:approximation_in_probability}, it follows that one can see a dropout neural network as a linear combination of dropout networks with weights $(w \odot \mathbf{1}_U)_{U \in 2^n}$, such that the linear combination equals the original neural network in expectation as shown in \eqref{eqn:Decomposition_of_Psi_in_a_linear_combination_of_expectations}.

To get a dropout neural network that is close to the original network in probability, in \eqref{eqn:average_converges_in_prob_and_Lq} one makes a large average of independent copies of the dropout network that approximates the original network in expectation. The convergence in probability of \eqref{eqn:average_converges_in_prob_and_Lq} follows then from the weak law of large numbers. The convergence in $L^q$ finally follows because the expectation is uniformly bounded in $\mathcal{F}$ for any realization of the filter variables $f^{i,U}$, so that the convergence in probability immediately implies the convergence in $L^q$ by dominated convergence.

\subsection{Examples}
We further illustrate the construction of  \refTheorem{thm:approximation_in_probability} with the following examples:

\begin{example}[One-hidden-layer dropconnect networks]
	Consider the function $\Psi : \mathbb{R}^d \times \mathbb{R}^n \to \mathbb{R}$ given by
	\begin{equation}
		\label{eqn:Cybenko_base_NN}
	\Psi(x, w) := \sum_{j=1}^N c^j \sigma(w^j x +b^j)
	\end{equation}
	where the activation function $\sigma : \mathbb{R} \to \mathbb{R}$ is continuous with $\sigma(x) \to 0$ as $x \to -\infty$ and $\sigma(x) \to 1$ as $x \to \infty$.
	In \eqref{eqn:Cybenko_base_NN} we have biases $b_j \in \mathbb{R}$ and weights made up from the constants $c^j \in \mathbb{R}$ and the $1 \times d$-matrices $w^j$. 

	The well-known result by \cite{cybenko1989approximation} implies that the class of all such functions
	is dense in $C([0,1]^d)$ endowed with the supremum norm. An example of an approximation by functions in \eqref{eqn:Cybenko_base_NN} is depicted in Figure \ref{fig:Cybenko_base_NN}.

	\begin{figure}[!hptb]
		\centering
%

\def\layersep{2cm}
\begin{tikzpicture}[node distance=\layersep]
	\tikzstyle{every edge}=[shorten >=1pt,->,draw=black!50];
    \tikzstyle{every pin edge}=[<-,shorten <=1pt];
    \tikzstyle{neuron}=[circle,draw=black,fill=black!100,minimum size=9pt,inner sep=0pt];
    \tikzstyle{input neuron}=[neuron, fill=black!50];
    \tikzstyle{output neuron}=[neuron, fill=black!50];
    \tikzstyle{hidden neuron}=[neuron, fill=black!25];
    \tikzstyle{annot} = [text width=4em, text centered]

\node[input neuron, pin=left:$x$] (I-1) at (0,-2) {};
\node[hidden neuron, pin={[pin distance=0.25cm]below right:\footnotesize $1$}] (I-2) at (0,-4) {};
\path node[hidden neuron] (H-1) at (\layersep,-2 cm) {$\sigma$};
\path node[hidden neuron] (H-2) at (\layersep,-4 cm) {$\sigma$};
\node[output neuron,pin={[pin edge={->}]right:$\Psi(x,w) = \displaystyle \sum_{j=1}^n c^j \sigma( w^j x + b^j )$}, right of=H-1] (O) {$+$};
\path (I-1) edge node[near start, above] {$w^1$} (H-1);
\path (I-1) edge node[near start, above] {$w^2$} (H-2);
\path (I-2) edge node[near start, above] {$b^1$} (H-1);
\path (I-2) edge node[near start, above] {$b^2$} (H-2);        
\path (H-1) edge node[near start, above] {$c^1$}  (O);
\path (H-2) edge node[near start, above] {$c^2$}  (O);

   \begin{axis}[
        width = 0.5*\columnwidth, 
        height = 0.5*0.618\columnwidth,	
        xshift = 0.4\textwidth,
        yshift = -6cm,
        xmin = -10, 
        xmax = 10,
        ymin = -0.5, 
        ymax = 0.5,
        xtick = {-10,0,10},
        every axis x label/.style=
            {at={(ticklabel cs:0.5)},anchor=north},
        ytick={-0.5,0,0.5},
        every axis y label/.style=
            {at={(ticklabel cs:0.5)},rotate=0,anchor=east},
        scaled ticks=true,
        xlabel={$x$},  			
    	ylabel={},	
        legend style = {at={(-0.25,0.25)},anchor=north east},   	
        legend entries=
        {
        	{\footnotesize Truth $\zeta(x)$},
        	{\footnotesize Cybenko NN $\Psi(x,v)$},
        	{\footnotesize Cybenko NN $\pm \varepsilon$},
        	{},
        	{\footnotesize Construction $\Psi(x,w \odot f)$}
        }
        ]

\addplot[color=black, thick, mark=none] plot coordinates {
(-10.,-0.000599094) (-9.9,-0.000582943) (-9.8,-0.000550331) (-9.7,-0.000498335) (-9.6,-0.000423811) (-9.5,-0.00032342) (-9.4,-0.000193653) (-9.3,-0.0000308754) (-9.2,0.000168624) (-9.1,0.000408566) (-9.,0.000692602) (-8.9,0.00102421) (-8.8,0.00140661) (-8.7,0.00184257) (-8.6,0.00233434) (-8.5,0.00288338) (-8.4,0.00349025) (-8.3,0.00415434) (-8.2,0.00487363) (-8.1,0.00564446) (-8.,0.00646118) (-7.9,0.00731591) (-7.8,0.00819818) (-7.7,0.00909458) (-7.6,0.00998842) (-7.5,0.0108594) (-7.4,0.0116831) (-7.3,0.0124311) (-7.2,0.0130698) (-7.1,0.013561) (-7.,0.0138613) (-6.9,0.0139217) (-6.8,0.0136877) (-6.7,0.0130994) (-6.6,0.0120913) (-6.5,0.0105927) (-6.4,0.00852824) (-6.3,0.00581816) (-6.2,0.00237946) (-6.1,-0.00187318) (-6.,-0.00702595) (-5.9,-0.0131643) (-5.8,-0.0203707) (-5.7,-0.0287223) (-5.6,-0.0382881) (-5.5,-0.0491256) (-5.4,-0.0612766) (-5.3,-0.0747635) (-5.2,-0.089584) (-5.1,-0.105706) (-5.,-0.12306) (-4.9,-0.141537) (-4.8,-0.160976) (-4.7,-0.181161) (-4.6,-0.20181) (-4.5,-0.222571) (-4.4,-0.243009) (-4.3,-0.2626) (-4.2,-0.280725) (-4.1,-0.296657) (-4.,-0.30956) (-3.9,-0.318477) (-3.8,-0.322329) (-3.7,-0.319909) (-3.6,-0.309882) (-3.5,-0.290786) (-3.4,-0.261035) (-3.3,-0.218927) (-3.2,-0.162657) (-3.1,-0.090333) (-3.,0.) (-2.9,0.090333) (-2.8,0.162657) (-2.7,0.218927) (-2.6,0.261035) (-2.5,0.290786) (-2.4,0.309882) (-2.3,0.319909) (-2.2,0.322329) (-2.1,0.318477) (-2.,0.30956) (-1.9,0.296657) (-1.8,0.280725) (-1.7,0.2626) (-1.6,0.243009) (-1.5,0.222571) (-1.4,0.20181) (-1.3,0.181161) (-1.2,0.160976) (-1.1,0.141537) (-1.,0.12306) (-0.9,0.105706) (-0.8,0.089584) (-0.7,0.0747635) (-0.6,0.0612766) (-0.5,0.0491256) (-0.4,0.0382881) (-0.3,0.0287223) (-0.2,0.0203707) (-0.1,0.0131643) (0.,0.00702595) (0.1,0.00187318) (0.2,-0.00237946) (0.3,-0.00581816) (0.4,-0.00852824) (0.5,-0.0105927) (0.6,-0.0120913) (0.7,-0.0130994) (0.8,-0.0136877) (0.9,-0.0139217) (1.,-0.0138613) (1.1,-0.013561) (1.2,-0.0130698) (1.3,-0.0124311) (1.4,-0.0116831) (1.5,-0.0108594) (1.6,-0.00998842) (1.7,-0.00909458) (1.8,-0.00819818) (1.9,-0.00731591) (2.,-0.00646118) (2.1,-0.00564446) (2.2,-0.00487363) (2.3,-0.00415434) (2.4,-0.00349025) (2.5,-0.00288338) (2.6,-0.00233434) (2.7,-0.00184257) (2.8,-0.00140661) (2.9,-0.00102421) (3.,-0.000692602) (3.1,-0.000408566) (3.2,-0.000168624) (3.3,0.0000308754) (3.4,0.000193653) (3.5,0.00032342) (3.6,0.000423811) (3.7,0.000498335) (3.8,0.000550331) (3.9,0.000582943) (4.,0.000599094) (4.1,0.000601476) (4.2,0.000592541) (4.3,0.000574503) (4.4,0.000549338) (4.5,0.000518793) (4.6,0.000484397) (4.7,0.000447469) (4.8,0.000409138) (4.9,0.000370351) (5.,0.000331893) (5.1,0.0002944) (5.2,0.000258375) (5.3,0.000224205) (5.4,0.000192171) (5.5,0.000162467) (5.6,0.000135207) (5.7,0.000110441) (5.8,0.0000881664) (5.9,0.0000683337) (6.,0.0000508595) (6.1,0.0000356324) (6.2,0.0000225207) (6.3,0.0000113782) (6.4,0.00000204952) (6.5,-0.0000056252) (6.6,-0.0000118069) (6.7,-0.0000166544) (6.8,-0.0000203219) (6.9,-0.0000229567) (7.,-0.0000246985) (7.1,-0.0000256776) (7.2,-0.0000260146) (7.3,-0.0000258196) (7.4,-0.0000251928) (7.5,-0.0000242237) (7.6,-0.0000229919) (7.7,-0.0000215673) (7.8,-0.0000200106) (7.9,-0.000018374) (8.,-0.0000167015) (8.1,-0.00001503) (8.2,-0.0000133895) (8.3,-0.000011804) (8.4,-0.0000102923) (8.5,-0.00000886841) (8.6,-0.00000754212) (8.7,-0.00000631975) (8.8,-0.0000052046) (8.9,-0.0000041974) (9.,-0.00000329682) (9.1,-0.00000249982) (9.2,-0.00000180206) (9.3,-0.00000119816) (9.4,-0.000000682055) (9.5,-0.000000247159) (9.6,0.000000113377) (9.7,0.000000406508) (9.8,0.000000639146) (9.9,0.000000818048) (10.,0.000000949716) 
};

\addplot[color=blue, densely dotted, thick, mark=none] plot coordinates {
(-10.,0.00988746) (-9.9,0.00987244) (-9.8,0.00985503) (-9.7,0.00983477) (-9.6,0.00981188) (-9.5,0.0097847) (-9.4,0.0097537) (-9.3,0.00971746) (-9.2,0.00967574) (-9.1,0.00962734) (-9.,0.00957131) (-8.9,0.00950646) (-8.8,0.0094316) (-8.7,0.00934505) (-8.6,0.0092442) (-8.5,0.00912762) (-8.4,0.00899339) (-8.3,0.00883722) (-8.2,0.00865698) (-8.1,0.0084486) (-8.,0.00820756) (-7.9,0.00792837) (-7.8,0.00760555) (-7.7,0.00723195) (-7.6,0.00680017) (-7.5,0.00630116) (-7.4,0.00572371) (-7.3,0.00505638) (-7.2,0.00428557) (-7.1,0.00339484) (-7.,0.00236607) (-6.9,0.00117898) (-6.8,-0.000191927) (-6.7,-0.0017724) (-6.6,-0.00359535) (-6.5,-0.0056951) (-6.4,-0.00811386) (-6.3,-0.0108964) (-6.2,-0.0140951) (-6.1,-0.0177689) (-6.,-0.0219829) (-5.9,-0.0268104) (-5.8,-0.0323319) (-5.7,-0.0386362) (-5.6,-0.0458186) (-5.5,-0.0539813) (-5.4,-0.0632324) (-5.3,-0.0736818) (-5.2,-0.0854375) (-5.1,-0.098603) (-5.,-0.113264) (-4.9,-0.129482) (-4.8,-0.14728) (-4.7,-0.166621) (-4.6,-0.187385) (-4.5,-0.20934) (-4.4,-0.232107) (-4.3,-0.255118) (-4.2,-0.277578) (-4.1,-0.298422) (-4.,-0.316289) (-3.9,-0.329523) (-3.8,-0.336215) (-3.7,-0.334313) (-3.6,-0.32182) (-3.5,-0.297086) (-3.4,-0.259169) (-3.3,-0.20822) (-3.2,-0.145757) (-3.1,-0.0747371) (-3.,0.000687599) (-2.9,0.0757157) (-2.8,0.145621) (-2.7,0.206458) (-2.6,0.255544) (-2.5,0.291637) (-2.4,0.314819) (-2.3,0.326174) (-2.2,0.327419) (-2.1,0.320548) (-2.,0.307563) (-1.9,0.290299) (-1.8,0.270326) (-1.7,0.248925) (-1.6,0.227087) (-1.5,0.20555) (-1.4,0.184831) (-1.3,0.165277) (-1.2,0.147092) (-1.1,0.13038) (-1.,0.115169) (-0.9,0.101432) (-0.8,0.089106) (-0.7,0.0781068) (-0.6,0.0683359) (-0.5,0.0596895) (-0.4,0.0520633) (-0.3,0.0453556) (-0.2,0.0394698) (-0.1,0.0343159) (0.,0.0298107) (0.1,0.0258787) (0.2,0.0224515) (0.3,0.0194676) (0.4,0.0168723) (0.5,0.014617) (0.6,0.0126585) (0.7,0.010959) (0.8,0.00948497) (0.9,0.00820722) (1.,0.00710007) (1.1,0.00614112) (1.2,0.0053108) (1.3,0.00459208) (1.4,0.00397011) (1.5,0.003432) (1.6,0.00296653) (1.7,0.00256396) (1.8,0.00221585) (1.9,0.00191488) (2.,0.00165468) (2.1,0.00142977) (2.2,0.00123537) (2.3,0.00106736) (2.4,0.000922163) (2.5,0.000796693) (2.6,0.000688276) (2.7,0.000594597) (2.8,0.000513657) (2.9,0.000443727) (3.,0.00038331) (3.1,0.000331115) (3.2,0.000286023) (3.3,0.000247069) (3.4,0.000213418) (3.5,0.000184349) (3.6,0.000159238) (3.7,0.000137546) (3.8,0.000118808) (3.9,0.000102623) (4.,0.0000886418) (4.1,0.0000765652) (4.2,0.0000661336) (4.3,0.000057123) (4.4,0.0000493401) (4.5,0.0000426174) (4.6,0.0000368106) (4.7,0.0000317949) (4.8,0.0000274626) (4.9,0.0000237206) (5.,0.0000204884) (5.1,0.0000176966) (5.2,0.0000152853) (5.3,0.0000132024) (5.4,0.0000114034) (5.5,0.00000984953) (5.6,0.00000850738) (5.7,0.00000734811) (5.8,0.00000634681) (5.9,0.00000548195) (6.,0.00000473494) (6.1,0.00000408972) (6.2,0.00000353242) (6.3,0.00000305106) (6.4,0.0000026353) (6.5,0.00000227619) (6.6,0.00000196602) (6.7,0.00000169811) (6.8,0.00000146671) (6.9,0.00000126684) (7.,0.00000109421) (7.1,0.000000945104) (7.2,0.000000816315) (7.3,0.000000705076) (7.4,0.000000608995) (7.5,0.000000526008) (7.6,0.000000454329) (7.7,0.000000392418) (7.8,0.000000338943) (7.9,0.000000292755) (8.,0.000000252861) (8.1,0.000000218404) (8.2,0.000000188642) (8.3,0.000000162936) (8.4,0.000000140732) (8.5,0.000000121555) (8.6,0.00000010499) (8.7,0.0000000906834) (8.8,0.0000000783259) (8.9,0.0000000676526) (9.,0.0000000584335) (9.1,0.0000000504707) (9.2,0.0000000435931) (9.3,0.0000000376526) (9.4,0.0000000325217) (9.5,0.00000002809) (9.6,0.0000000242621) (9.7,0.000000020956) (9.8,0.0000000181003) (9.9,0.0000000156337) (10.,0.0000000135033) 
};
\end{axis}

\end{tikzpicture}

		\caption{A Cybenko neural network as in \eqref{eqn:Cybenko_base_NN}, trained to  approximate a function $\zeta$; here, $n=2,d=1$, and $\zeta(x) =  \sin{(x+3)} \exp{|x-3|}$.}
		\label{fig:Cybenko_base_NN}
	\end{figure}
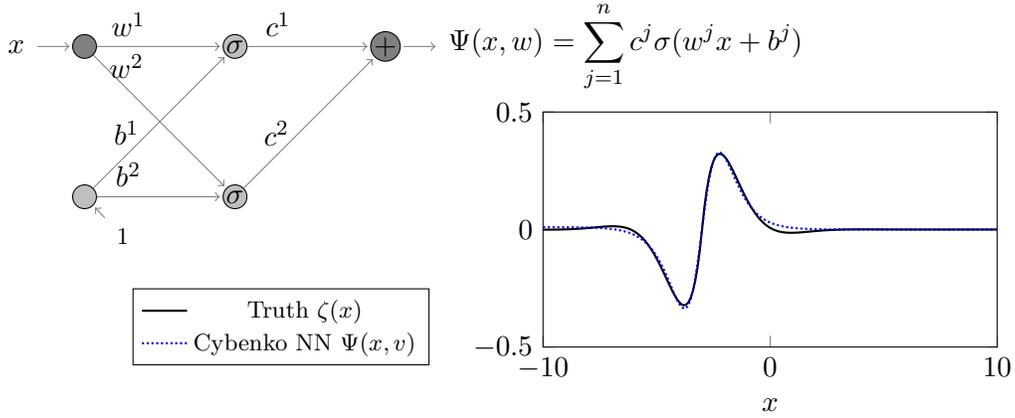

	\begin{figure}[!hptb]
		\centering
		\input{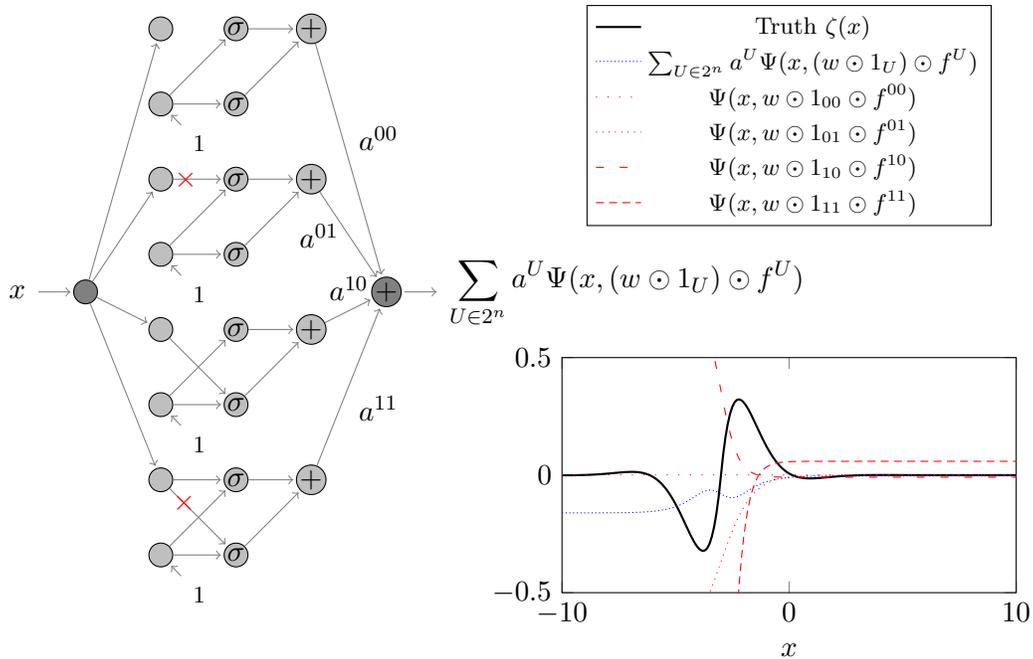}
		\caption{A single realization of the random neural network in \eqref{eqn:Decomposition_of_Psi_in_a_linear_combination_of_expectations} using Dropconnect. Based off the trained Cybenko neural network in Figure \ref{fig:Cybenko_base_NN}, for simplicity, we only apply dropout to the weights $w^{j}$ of \eqref{eqn:Cybenko_base_NN}, which we denote by $w$ and correspond to the edges joining the nodes connected to the input $x$. With $n=2$ and $d = 1$, there are four different random neural networks with their respective independent filters. All of them use as base $\Psi( \cdot, w)$ in Figure \ref{fig:Cybenko_base_NN}. In this realization, some of the edges are filtered, which are depicted with red crosses. The explicit coefficients $a_{U}$ used for dropconnect are computed in Section~\ref{sec:explicit-computation}.}
		\label{fig:A_Dropconnect_RNN_based_on_Cybenko__One_realization}
	\end{figure}

	We suppose that the distribution of the filters follows the case of dropconnect, as described in Section~\ref{sec:dropconnect}. Theorem \ref{thm:approximation_in_probability} directly yields that by choosing appropriate weights $c^{j, U}$ and weight matrices $w^{j, U}$, the one-hidden-layer dropconnect network given by
	\begin{equation}
		\frac{1}{M} \sum_{i=1}^M \sum_{U \in 2^n} \sum_{j=1}^N a_U c^{j, U} g^{j, U} \sigma\left( (w^{j, U} \odot f^{j, U})x + b^j \right)
		\label{eqn:A_Dropconnect_RNN_based_on_Cybenko__LLN_construction}
	\end{equation}
	can be chosen to be close to $\Psi$ in $L^q$ for large $M$. Here $g^{j, U}$ are independent Bernouilli random variables, and $f^{j, U}$ are random vectors with independently Bernoulli-distributed components, all with success probability $1-p$. This result is illustrated by Figures \ref{fig:A_Dropconnect_RNN_based_on_Cybenko__One_realization} and \ref{fig:A_Dropconnect_RNN_based_on_Cybenko__LLN_construction}, where for simplicity  we have used filters only on the weights $w^{j, U}$, while leaving the biases $b^j$ and $c^j$ with constant filters $1$. Figure \ref{fig:A_Dropconnect_RNN_based_on_Cybenko__One_realization} shows a single realization of the neural network in \eqref{eqn:Decomposition_of_Psi_in_a_linear_combination_of_expectations} with dropconnect while in Figure \ref{fig:A_Dropconnect_RNN_based_on_Cybenko__LLN_construction} a `blow up'---the average of $M$ independent copies of the network in \eqref{eqn:A_Dropconnect_RNN_based_on_Cybenko__LLN_construction}---of the previous construction is depicted.

	\begin{figure}[!hptb]
		\centering
		\input{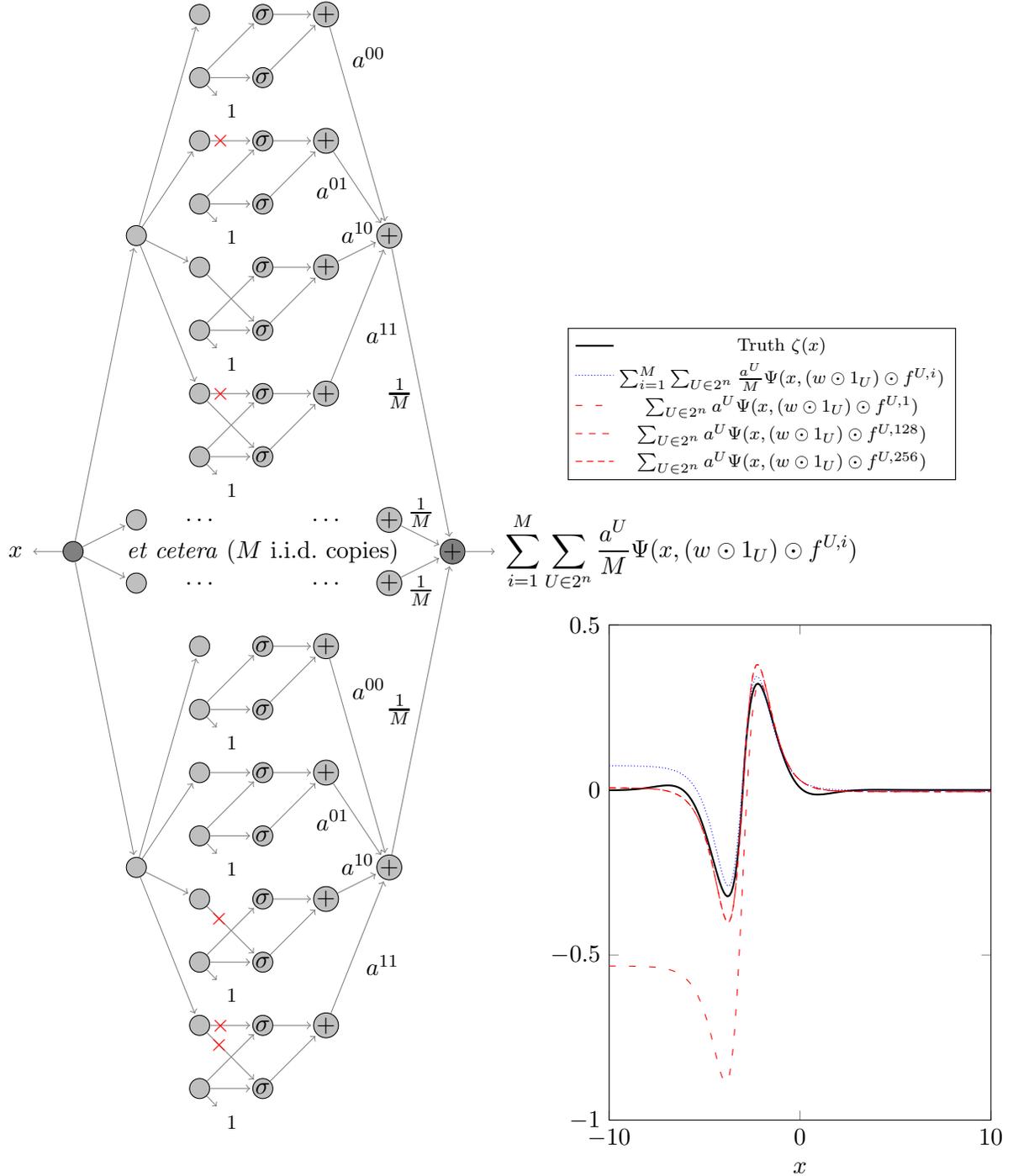}
		\vspace{-1em}
		\caption{An approximation of $\zeta$ with a large neural network using dropconnect, based off the base Cybenko neural network in \eqref{eqn:Cybenko_base_NN} depicted in Figure \ref{fig:Cybenko_base_NN}. Adding many independent copies of the network from Figure \ref{fig:A_Dropconnect_RNN_based_on_Cybenko__One_realization}, we are leveraging the law of large numbers as in  \eqref{eqn:A_Dropconnect_RNN_based_on_Cybenko__LLN_construction}. Different independent copies of the network may have a different realization of the filters, which is here depicted by the red crosses on the edges.}
		\label{fig:A_Dropconnect_RNN_based_on_Cybenko__LLN_construction}
	\end{figure}
\end{example}

In a similar way, we can also consider more general dropconnect networks.

\begin{example}[Dropconnect networks] 
	Consider a deep neural network $ \Psi : \mathbb{R}^d \times \mathbb{R}^n \to \mathbb{R}$ as introduced in \eqref{eqn:definition_Psi} with dropconnect filters as described in Section~\ref{sec:dropconnect}. Here, the filter variables, i.e., the components of $f^{i,U}$ in \eqref{eqn:definition_Psi}, are i.i.d.\ Bernoulli distributed with success probability $1-p$ if they multiply elements of the weight matrices $W^{(j)}$, and are equal to $1$ if they multiply biases $b^{(j)}$.

	Let $w \in \R{n}$. We choose for $\mathcal{F}$ the vector space of continuous functions on $\mathbb{R}^d$, endowed with the supremum seminorm over the closed unit cube. Then the dropconnect random network in \eqref{eqn:average_converges_in_prob_and_Lq} is for large $M$ close to the network $\Psi( \cdot, w )$ in $L^q$. 
\end{example}

\begin{example}[Node-dropout networks]
	Consider again the deep neural network in \eqref{eqn:definition_Psi} with node-dropout as described in Section~\ref{sec:node-dropout}.
	The random neural network in \eqref{eqn:average_converges_in_prob_and_Lq}
	is then again a node-dropout neural network.
	In this way, we recover \cite{foong2020expressiveness}'s Theorem 3 (with $h \equiv 0$), which for \gls{ReLU} activation functions and a target function $\zeta$ bounds
	\begin{equation}
	\sup_{x \in [0,1]^d} \Var(\zeta(x) - \Psi(x, w \odot f)).
	\label{eqn:result_foong_bound_variance}
	\end{equation}
	When $\mathcal{F}$ is the space of continuous functions with supremum norm, \eqref{eqn:result_foong_bound_variance} can be bounded by a constant times the square of the $L^2$-norm. Hence, 
	\refTheorem{thm:approximation_in_probability} approximates in a stronger sense, namely, in $L^q$ for any $q \in [1, \infty)$. Moreover, \refTheorem{thm:approximation_in_probability} also allows for activation functions other than \gls{ReLU}.
\end{example}

\begin{example}[Dropout networks with dropout on \emph{input}]
	\label{example:dropout_on_input}
	In contrast, if there is also dropout on the input, then the neural network in \eqref{eqn:average_converges_in_prob_and_Lq} is \emph{not} again a dropout neural network with dropout on the inputs. 
	Results by \cite{foong2020expressiveness} imply that in general neural networks with dropout on the \emph{input} cannot satisfy a universal approximation property. 

	We remark that this kind of stochastic network is \emph{not} a dropout neural network as defined in Section~\ref{sec:dropout_neural_networks} as the following example shows: Suppose that $\Psi_1, \Psi_2: \R{d} \times \R{n} \to \R{}$ are two different dropout neural networks with weights $w_1, w_2$ and with respective filter random variables $f, g$ with values in $\{0,1 \}^{n}$.  Then we can define the dropout neural network $\Psi$ with value
	\begin{equation}
	\Psi(x, (w_1 \odot f, w_2 \odot g)) = \Psi_1(x , w_1 \odot f) + \Psi_2(x , w_1 \odot g).
	\label{eqn:counter_example_foong_addition_is_not}
	\end{equation} 
	Suppose that, additionally, we add independent filters $h_1$ and $h_2$ with values in $\{0,1\}^{d}$ to $\Psi_1, \Psi_2$ for their respective inputs. Then, $\Psi_1(x \odot h_1, w_1) + \Psi_2(x \odot h_2, w_2)$ is not necessarily of the type $\Psi(x \odot h, (w_1, w_2))$ for some random variable $h$ with values in $\{0,1\}^{d}$.
\end{example}

As the above examples illustrate, a crucial aspect of whether a certain class of dropout neural networks (such as dropconnect or node-dropout) satisfy a universal approximation property, is whether linear, independent combinations of such networks are again networks in the same class. On the other hand, many details of the neural networks, such as them being a composition of simpler functions, are irrelevant for the proof of Theorem \ref{thm:approximation_in_probability}.

\subsection{The classes \texorpdfstring{$\DNN$}{DDNN}}

In the introduction we introduced classes $\DNN$ of tuples $(n, \Psi, f)$ that are closed under linear, independent combinations as the basic objects with which we want to approximate a given function $\zeta\in\mathcal F$.

The convergence statement \eqref{eqn:average_converges_in_prob_and_Lq} of Theorem~\ref{thm:approximation_in_probability} then immediately implies \refCorollary{cor:main}, which was already given as Corollary~\ref{cor:main-intro}. It expresses that if the class $\DNN$ is rich enough to approximate any function in $\mathcal{F}$ when all filter variables are set to 1 in the event in \eqref{eqn:probability_filters_on_is_positive}, then for every function in $\mathcal{F}$ there exists a dropout neural network such that with high probability with regards to the filter variables, the dropout neural network also approximates the function.

\begin{corollary}
\label{cor:main}
	Let $\zeta \in \mathcal{F}$ and $\epsilon > 0$. Assume there exists a $(m, \Phi, f) \in \DNN$ and a $v \in \mathbb{R}^m$ such that $\|\Phi(\cdot, v) - \zeta\|_{\subnorm} < \epsilon/2$. Then there exists a $(n, \Psi, g) \in \DNN$ and a $w \in \mathbb{R}^n$ such that 
	\begin{equation}
	\mathbb{P}[ \pnorm{\Psi(\cdot, w \odot g) - \zeta }{\subnorm} > \epsilon]
	< \epsilon
	\end{equation}
	and
	\[
	\mathbb{E} \left[ \left\|\zeta(\cdot) - \Psi(\cdot, w \odot g) \right\|_{\subnorm}^q \right]^{\frac{1}{q}} < \epsilon.
	\]
\end{corollary}

This corollary can be combined with deterministic universal approximation properties of certain classes of neural networks to obtain concrete universal approximation properties of dropout neural networks. For instance, because both the class of node-dropout networks and the class of dropconnect networks defined in Sections \ref{sec:node-dropout} and \ref{sec:dropconnect} form examples of a set $\DNN$, we obtain the following universal approximation property by combining Corollary \ref{cor:main} with the universal approximation result in \cite{leshno1993multilayer}'s Proposition 1.

\begin{corollary}
Assume $\mu$ is a nonnegative probability measure on $\mathbb{R}^n$ with compact support, absolutely continuous with respect to the Lebesgue measure. Take $\mathcal{F} = L^r(\mu)$ for some $r \in [1,\infty)$. Assume that the activation function $\sigma : \mathbb{R} \to \mathbb{R}$ is not equal to a polynomial almost everywhere. Then for every $\epsilon > 0$ there exists a one-hidden-layer dropconnect neural network $(n, \Psi, g)$ such that
\begin{equation}
\mathbb{P}[ \pnorm{ \Psi(\cdot, w \odot g) - \zeta }{L^r(\mu)} > \epsilon ]
< \epsilon,
\end{equation}
and
\[
\mathbb{E} \left[ \left\|\zeta(\cdot) - \Psi(\cdot, w \odot h) \right\|_{L^r(\mu)}^q \right]^{\frac{1}{q}} < \epsilon.
\]
There also exists a one-hidden-layer node-dropout neural network with the same properties.
\end{corollary}
Certainly, many variations of the above corollary can be constructed. 

To further illustrate Corollary \ref{cor:main}, in Figure \ref{fig:Function_approximation_in_probability}
we look at the approximation in probability of our construction from Theorem \ref{thm:approximation_in_probability}.

\begin{figure}[!hbtp]
	\centering
	\input{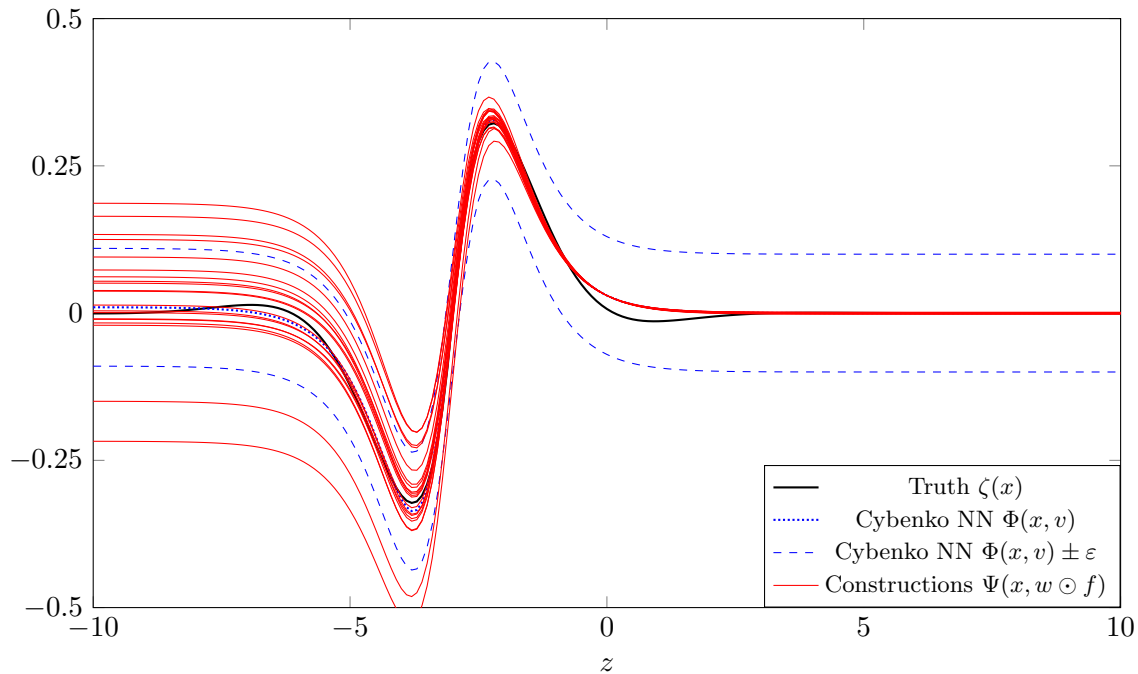}
	\caption{An illustration of function approximation in probability with our construction. Here, $M = 256$, and $20$ independent runs of the construction are shown in red. Here, $\epsilon = 0.1$ was chosen for illustrative purposes. Most of the runs lie within $\epsilon$ distance around the Cybenko neural network from \eqref{eqn:Cybenko_base_NN}, which we approximate with our construction in \eqref{eqn:A_Dropconnect_RNN_based_on_Cybenko__LLN_construction}. Altogether, we are approximating the target function $\zeta$.}
	\label{fig:Function_approximation_in_probability}
\end{figure}

\subsection{Explicit computation of coefficients}
\label{sec:explicit-computation}

To further illustrate Theorem \ref{thm:approximation_in_probability}, we will compute the coefficients $a_U$ in \eqref{eqn:Decomposition_of_Psi_in_a_linear_combination_of_expectations} explicitly for a special case of dropout neural networks for which the filter variables are partitioned into independent blocks. All variables in one block $i$ are all simultaneously off with probability $p_i$ and simultaneously on with probability $1-p_i$. Both node-dropout and dropconnect are special cases.

\begin{proposition}
	Let $f$ be a $\{0,1\}^n$-valued random variable with a distribution specified as follows. Let $\intint{n} = I_{1} \cup \ldots \cup I_{r}$ be a disjoint partition and suppose that $f_{i} = f_{j}$ whenever $i,j \in I_{s}$ for any $i,j \in \overline{1,n}$ and $s \in \overline{1,r}$. Let $f = (f_{I_{1}}, \ldots, f_{I_{r}})$ denote the random variables ordered as blocks and suppose that $\mathbb{P}(f_{I_{s}} = 1) = 1 - p_s > 0$ for all $s \in \overline{1,r}$ and that $\{f_{I_{i}}\}_{i \in \intint{r}}$ are mutually independent. Then we have
	\begin{equation}
		\Psi( \cdot , w ) 
		= 
		\sum_{V \in 2^r} \prod_{i \in V} \left(\frac{1}{1 - p_i}\right) \prod_{i \in \intint{r} \setminus V} \left(- \frac{p_i}{1 - p_i} \right) \E{} \Big[ \Psi( \cdot , (w \odot \mathbf{1}_{\iota(V)}) \odot f^V )\Big]
		\nonumber
	\end{equation} 
	where $\iota: 2^r \to 2^n$ is the embedding characterized by $j \in \iota(V)$ if $j \in I_i$ for some $i \in V$.
	\label{prop:Coefficients_aU_under_independence_conditions}
\end{proposition}

We prove Proposition \ref{prop:Coefficients_aU_under_independence_conditions} in Appendix~\ref{sec:proof_explicit_computation}. Note that as $p_i \to 1$, the coefficients $a_U$ become large. From this fact, together with the observation that the sum is taken over the large set $2^r$, it is clear that the construction is computationally strenuous. Still, small examples in the case of dropconnect are shown in Figures \ref{fig:A_Dropconnect_RNN_based_on_Cybenko__One_realization},  \ref{fig:A_Dropconnect_RNN_based_on_Cybenko__LLN_construction} and \ref{fig:Function_approximation_in_probability}.

\subsection{Why the results in this section are only for random-approximation dropout}

In this section, we have shown a random-approximation universal approximation result, i.e., a universal approximation result that is relevant when the dropout neural network is also used at prediction time with a stochastic output. In practice, the filter variables are usually replaced by their average values at prediction time. The following example shows that the construction in this section can lead to a bad approximation when doing expectation-replacement.

\begin{example}
	\label{ex:MC-bad-regular}
	Let $\sigma$ be the standard \gls{ReLU} activation function.
	The approximation procedure in Corollary \ref{cor:main} would yield that the function $\zeta: \mathbb{R} \to \mathbb{R}$ given by $\zeta(x) := \sigma(x - 1)$ can be well approximated by an average of many independent copies of the dropout neural network
	\[
	x \mapsto 4 f_1 \sigma (f_2 x - 1)
	\]
	where $f_1$ and $f_2$ are i.i.d.\ Bernoulli random variables with success probability $1/2$.
	However, replacing $f_1$ and $f_2$ by $1/2$, we just obtain the function
	\[
	x \mapsto 2 \sigma(x/2 - 1)
	\]
	which is not a good approximation to the function $\zeta$ at all.
\end{example}

\section{Use of average filter variables for prediction}
\label{sec:Use_of_average_filter_variables_for_prediction}

We will now approximate a neural network by a larger dropout neural network that is also close to the original neural network if the filter variables are replaced by their expected values. The replacement of the filter random variables $f$ by their expected values $\mathbb{E}[f]$ is common practice after having trained dropout neural networks for prediction. 
Informally, the main Theorem~\ref{th:main-average} below states that for any base neural network $\Psi(\cdot ,w)$, there exists a larger neural network $\indnn_{\Gamma,\Xi}(\cdot ,v)$ and filter variables $f$ such that
\begin{equation}
	\indnn_{\Gamma, \Xi}(x, v \odot f ) 
	\approx \Psi(x,w)	
	\approx \indnn_{\Gamma, \Xi}(x,  v \odot \mathbb{E}[f]).
\end{equation}

\subsection*{Global variables} 

In order to improve readability of this section, we fix for the entire section a few (otherwise arbitrary) variables. Throughout this section:
\begin{itemize}
	\item The base neural network $\Psi$ is assumed to be a fixed $(L-1)$-hidden layer neural network as described in Section~\ref{sec:neural-networks}. We assume that its activation functions $\sigma_j$ are continuous. We also keep the weights $W^{(j)}$ and biases $b^{(j)}$ fixed.  
	\item We fix a number $R > 0$, which will play the role of the radius of a ball in the input space.
	\item We fix a number $\beta \in (0,1)$ and assume that for every random filter matrix $F$ in this section, each one entry is on with a probability that is larger than or equal to $\beta$, i.e., for all $r,c$,
	\[	
		\mathbb{P}[ 
			F_{rc} = 1  
		]
		\geq \beta
		> 0.
	\]
	\item We fix a number $Q > 1$, whose role will become clear later.
\end{itemize}

\begin{figure}[hbtp]
	\centering
	\def\layersep{1.75cm}
\def\heightsep{1cm}

\def\boxwidth{0.5cm}
\def\boxheight{0.7cm}
\setstackgap{S}{2pt}

\begin{tikzpicture}[node distance=\layersep]
    \tikzstyle{every edge} = [shorten >=1pt,->,draw=black!50, semithick];
    \tikzstyle{every pin edge} = [<-,shorten <=1pt];
    \tikzstyle{neuron} = [circle,draw=black,fill=black!100,minimum size=9pt,inner sep=0pt];
    \tikzstyle{input neuron} = [neuron, fill=black!50];
    \tikzstyle{output neuron} = [neuron, fill=black!50];
    \tikzstyle{hidden neuron} = [neuron, fill=black!25];
    \tikzstyle{annot} = [text width=4em, text centered];
    \tikzstyle{box} = [draw=black!25, dashed, shape=rectangle, minimum width=0.75*\layersep,
    minimum height = 4*\heightsep];
    
    \tikzstyle{sbox} = [
        draw =black!100,
        shape = rectangle,
        minimum width = \boxwidth,
        minimum height = \boxheight
    ];

    \node[box] (X1) at (0*\layersep,-2.5*\heightsep) {};
    \node[box] (X2) at (1*\layersep,-2.5*\heightsep) {};
    \node[box] (X3) at (2*\layersep,-2.5*\heightsep) {};
    \node[box] (X4) at (3*\layersep,-2.5*\heightsep) {};
    \node[box] (X5) at (4*\layersep,-2.5*\heightsep) {};

    \node[input neuron, pin=left:$x$] (A1) at (0*\layersep,-2.5*\heightsep) {};
    \path node[hidden neuron] (B1) at (1*\layersep,-2.0*\heightsep) {$\sigma_1$};
    \path node[hidden neuron] (B2) at (1*\layersep,-3.0*\heightsep) {$\sigma_1$};
    \path node[hidden neuron] (C1) at (2*\layersep,-1.0*\heightsep) {$\sigma_2$};
    \path node[hidden neuron] (C2) at (2*\layersep,-2.0*\heightsep) {$\sigma_2$};
    \path node[hidden neuron] (C3) at (2*\layersep,-3.0*\heightsep) {$\sigma_2$};
    \path node[hidden neuron] (C4) at (2*\layersep,-4.0*\heightsep) {$\sigma_2$};
    \path node[hidden neuron] (D1) at (3*\layersep,-1.5*\heightsep) {$\sigma_3$};
    \path node[hidden neuron] (D2) at (3*\layersep,-2.5*\heightsep) {$\sigma_3$};
    \path node[hidden neuron] (D3) at (3*\layersep,-3.5*\heightsep) {$\sigma_3$};
    \node[output neuron,pin={[pin edge={->}]right:$\Psi(x,w)$}] (E1) at (4*\layersep,-2.5*\heightsep) {$\sigma_4$};

    \path (A1) edge node[near start, above] {} (B1);
    \path (A1) edge node[near start, above] {} (B2);
    \path (B1) edge node[near start, above] {} (C1);
    \path (B1) edge node[near start, above] {} (C2);
    \path (B1) edge node[near start, above] {} (C3);
    \path (B1) edge node[near start, above] {} (C4);
    \path (B2) edge node[near start, above] {} (C1);
    \path (B2) edge node[near start, above] {} (C2);
    \path (B2) edge node[near start, above] {} (C3);
    \path (B2) edge node[near start, above] {} (C4);
    \path (C1) edge node[near start, above] {} (D1);
    \path (C1) edge node[near start, above] {} (D2);
    \path (C1) edge node[near start, above] {} (D3);
    \path (C2) edge node[near start, above] {} (D1);
    \path (C2) edge node[near start, above] {} (D2);
    \path (C2) edge node[near start, above] {} (D3);
    \path (C3) edge node[near start, above] {} (D1);
    \path (C3) edge node[near start, above] {} (D2);
    \path (C3) edge node[near start, above] {} (D3);
    \path (C4) edge node[near start, above] {} (D1);
    \path (C4) edge node[near start, above] {} (D2);
    \path (C4) edge node[near start, above] {} (D3);
    \path (D1) edge node[near start, above] {} (E1);
    \path (D2) edge node[near start, above] {} (E1);
    \path (D3) edge node[near start, above] {} (E1);
    
    \begin{scope}[shift={(0,-5*\heightsep)}]   
    \node[scale=0.7] at (0.5*\layersep, 0) {$(W^{(1)},b^{(1)})$};
    \node[scale=0.7] at (1.5*\layersep, 0) {$(W^{(2)},b^{(2)})$};
    \node[scale=0.7] at (2.5*\layersep, 0) {$(W^{(3)},b^{(3)})$};
    \node[scale=0.7] at (3.5*\layersep, 0) {$(W^{(4)},b^{(4)})$};
    \end{scope}
    
 	\begin{scope}[shift={(0,-6*\heightsep)}]   
    
    \node[sbox, pin=left:$x$] (A) at (0*\layersep,0) {$\Shortstack{ . }$};
    \path node[sbox] (B) at (1*\layersep,0) {$\Shortstack{ . . }$};
    \path node[sbox] (C) at (2*\layersep,0) {$\Shortstack{ . . . . }$};
    \path node[sbox] (D) at (3*\layersep,0) {$\Shortstack{ . . . }$};
    \node[sbox,pin={[pin edge={->}]right:$\Psi(x,w)$}] (E) at (4*\layersep,0) {$\Shortstack{ . }$};

    \tikzstyle{every edge} += [shorten >=4pt, shorten <=3pt, ultra thick];
    \path (A) edge node[near start, above] {} (B);
    \path (B) edge node[near start, above] {} (C);
    \path (C) edge node[near start, above] {} (D);
    \path (D) edge node[near start, above] {} (E);

	\end{scope}
\end{tikzpicture}
	\caption{An example base neural network $\Psi$ where $L=4$. The top diagram indicates the individual activation functions and the dimensions of each layer of nodes: $d = d_0 = 1$, $d_1 = 2$, $d_2 = 4$, $d_3 = 3$, and $d_4 = 1$. The set of all  arrows connecting layer $k-1$ of nodes to layer $k$ correspond to the parameters $(W^{(k)},b^{(k)})$. The bottom diagram  shows the same network, with the edges between layers compressed to single arrows; this compressed notation is the basis for the diagram in Figure~\ref{fig:Tree_construction} below. }
	\label{fig:Base_NN_psi}
\end{figure}
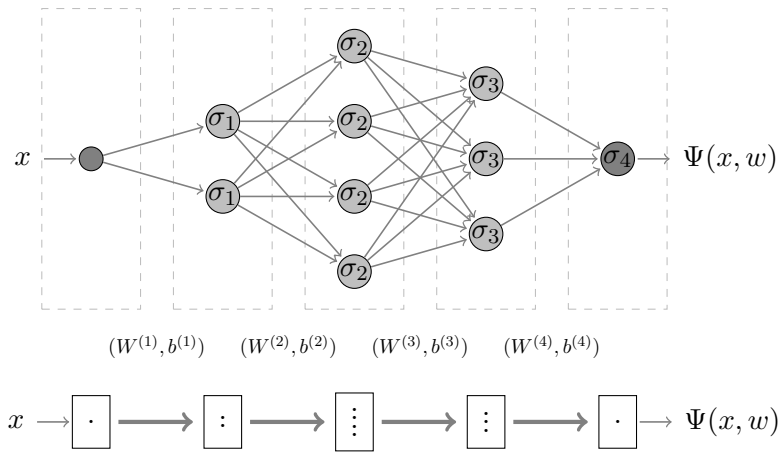

\subsection{Heuristic description of the construction}

In this section we describe the construction of the larger dropout neural network $\indnn_{\Gamma,\Xi}$ in heuristic terms; the full details are given in the subsequent sections. The construction starts at the last layer of the base neural network $\Psi$, which is a function $\Psi_L : \mathbb{R}^{d_{L-1}} \to \mathbb{R}^{d_{L}}$ given by
\begin{equation}
	x \mapsto \sigma_L \bigl( W^{(L)} x + b^{(L)} \bigr).
	\label{eqn:Last_layer}
\end{equation}
We construct the last layer of the larger dropout neural network such that it remains close to \eqref{eqn:Last_layer} as follows. Let $N \in \mathbb{N}$ and consider any collection $\{ F^{(L),i} \}_{i \in \intint{N}}$ of i.i.d.\ random filter matrices such that for each $i$, $F^{(L),i}$ has the same dimension as $W^{(L)}$.
By the law of large numbers, we can expect that the function
\begin{equation}
	x
	\mapsto
	\sigma_L
	\biggl(
	\frac{1}{N}\sum_{i=1}^N
	\Bigl(
		\bigl( W^{(L)} \div \mathbb{E}[F^{(L),i}] \bigr) \odot F^{(L),i}
		\Bigr)
	x
	+ b^{(L)}
	\biggr)
	\label{eqn:Last_layer_replacement}
\end{equation}
will be close to the function $\Psi_L$ for sufficiently large $N$. Here we write $\div$ for element-wise division. 

Viewed as a one-layer neural network, the function \eqref{eqn:Last_layer_replacement} is a  one-layer dropout neural network with $N$ times as many edges as $\Psi_L$, and it can  replace \eqref{eqn:Last_layer}, i.e., $\Psi_L$, while staying close to~$\Psi_L$. 

A further adaptation is necessary, however, because in \eqref{eqn:Last_layer_replacement} each copy $W^{(L)} \div \mathbb{E}[F^{(L),i}]$ takes the \emph{same} input $x\in \R{d_{L-1}}$. To make \eqref{eqn:Last_layer_replacement} a \emph{bona fides} dropout network, different edges should take different inputs, and therefore we generalize \eqref{eqn:Last_layer_replacement} to 
\begin{equation}
	(\R{d_{L-1}})^N \ni
	(x^i)_{i\in \intint N}
	\mapsto
	\sigma_L
	\biggl(
	\frac{1}{N}\sum_{i=1}^N
	\Bigl(
		\bigl( W^{(L)} \div \mathbb{E}[F^{(L),i}] \bigr) \odot F^{(L),i}
		\Bigr)
	x^i
	+ b^{(L)}
	\biggr).
	\label{eqn:Last_layer_replacement-generalized}
\end{equation}
By precomposing each of the inputs $x^i$ with $\Psi^{(L-1)}$, and performing the same construction as above (copying the input to these copies of $\Psi^{(L-1)}$), we can inductively create our larger dropout neural network $\indnn_{\Gamma,\Xi}$ that will be close to~$\Psi$.  

There are now three points of attention:
\begin{itemize}
	\item The intuitive statement `repeating this construction' needs a formalization by an inductive construction. This requires a mathematical object that can record the intermediate stages of the construction.
	\item We need to show inductively that the resulting intermediate neural networks are close to (a network closely related to) the original network. In particular, we need to introduce a mathematical specification of `close' that is compatible with an inductive argument.
	\item The input space to the neural network in this construction grows with each step, whereas we still aim to have a final neural network with data space $\mathbb{R}^d$. This requires us to deal with the first layer of the network differently.
\end{itemize}
These points are the topics of the subsequent sections.

\subsection{Dropout-trees}

We will encode the intermediate stages of our inductive construction by a mathematical object that we will refer to as a \emph{dropout-tree}. 
The idea is that we start with a root, then attach incoming edges labeled with random filter matrices to it (creating leaves), and then recursively attach even more edges to the leaves. To be consistent with the numbering of layers in Section \ref{sec:neural-networks}, here, we will prefer to speak about the \emph{level} $j$ of a vertex or an edge in a tree rather than its depth $L-j$ (the latter is also established jargon in graph theory, and this aligns our notation with that of the neural network). In this numbering, the root is therefore at level $L$.

\begin{definition}
	A vertex $v$ of a rooted tree is at \emph{level} $j \in \{ 0, 1, \ldots, L \}$ if the path from $v$ to the root $v_0$ has length $L-j$.
	An edge $e = (u,v)$ of a rooted tree is at \emph{level} $j \in \{0,1, \ldots, L \}$ if its target vertex $v$ is at level $j$.
\end{definition}

From now on we will write $\sigma_v$ for $\sigma_{\level(v)}$, $W^{v}$ for $W^{(\level(v))}$, $b^{v}$ for $b^{(\level(v))}$, \emph{et cetera}. This simplifies the notation at only a minor cost of abuse of notation.

\begin{definition}
	A \emph{dropout-tree} $\Gamma$ of an $(L-1)$-hidden layer neural network $\Psi$ is a directed graph $\mathcal{G}$ together with a labeling of the edges such that:
	\begin{itemize}
		\item the graph $\mathcal{G}$ is connected and acyclic;
		\item one of the vertices, say $v_0$, is designated as the \emph{root};
		\item the depth of the tree is at most  $L-1$; 
		\item all directed edges point towards the root;
		\item every edge $e$ is labeled with a random matrix $F^e$; for each $e$
		\begin{itemize}
			\item[(a)] $F^e$ has the same dimension as $W^{\target(e)}$
			\item[(b)] $F^e$'s entries are $\{0,1\}$-valued
			\item[(c)] for all $r,c$, $\probability{ F^e_{rc} = 1 } \geq \beta > 0$;
		\end{itemize} 
		\item for every vertex $v$ that is not a leaf, $\{F^e\}_{e\in \into(v)}$ is a collection of mutually independent, identically distributed random matrices.
	\end{itemize}
\end{definition}

For convenience we recall some terminology. A directed edge points from a \emph{source} to a \emph{target}, and for an edge $e$ we identify them by $\source(e)$ and $\target(e)$; we write $\into(v)$ for the set of  all edges with target vertex $v$. In the trees in this paper, all edges point towards the root of the tree. A vertex $v$ is a \emph{child} of a vertex $w$ if there is an edge pointing from $v$ to $w$; $w$ then is the \emph{parent} of $v$. A \emph{leaf} is a vertex without children.

\medskip

Dropout-trees can be constructed iteratively by starting from the trivial dropout-tree consisting only of a root and then performing a so-called \emph{$\mu$-input-copy} construction. This allows us to inductively create a larger dropout-tree from a smaller dropout-tree.

\begin{definition}
	\label{de:mu-clone}
	Let $\Gamma$ be a dropout-tree and let $\ell $ be a leaf of $\Gamma$ at level $k$. 
	Let $\mu$ be a distribution of a random matrix $F \in \{0,1\}^{ d_k \times d_{k-1} }$ that satisfies for all $r,c$, $\probability{ F_{rc} = 1 } \geq \beta > 0$.	
	A dropout-tree $\Gamma'$ is a \emph{$\mu$-input-copy to the leaf $\ell $ of $\Gamma$} if one can obtain $\Gamma'$ from $\Gamma$ by:
	(a) attaching child vertices to $\ell$, and
	(b) labeling each edge going into $\ell$ by an independent copy of $F$.
	The \emph{size of a $\mu$-input-copy to $\ell $ at $\Gamma$} refers to the number of children of $\ell $ in $\Gamma'$.
\end{definition}

Let us describe the precise meaning of procedure (b) in Definition \ref{de:mu-clone}. For that, it may be useful to recall that random matrices are nothing but measurable functions defined on the probability space $(\Omega, \mathcal{F}_\Omega, \mathbb{P})$. The procedure (b) precisely means that the sigma-algebras generated by the filter variables $F^e$ with $e \in \into(\ell)$ are independent, and that for every $e \in \into(\ell)$ the law of $F^e$ equals $\mu$. 
In particular, this condition allows for some correlation between filter variables labeling edges in the dropout-tree that do not go into $\ell$. Moreover, in general there can be many different dropout trees $\Gamma'$ that are $\mu$-input-copies of~$\Gamma$.

We will now describe how a dropout-tree encodes a dropout neural network.

\subsubsection{Dropout neural networks encoded by dropout-trees}

Each dropout-tree $\Gamma$ will induce a stochastic function $\Phi^{v_0}_\Gamma$---a dropout neural network---as follows. For any edge $e = (u,v)$ of $\Gamma$, let
\begin{equation}
	V^e_\Gamma
	:= W^e \div \mathbb{E}[F^e]
	\label{eqn:Definition_of_Vegamma}
\end{equation}
be rescaled weights for the dropout neural network. We define
\[
	\Phi^{v}_\Gamma
	:=
	\begin{cases}
		\sigma_v
		\Bigl(
			\frac{1}{\# \into(v)}
			\sum_{e \in \into(v)}
			(V^e \odot F^e) \Phi^{\source(e)}_\Gamma
			+ b^v
		\Bigr)
		&
		\textnormal{if } v \textnormal{ is not a leaf},
		\\
		\mathrm{Identity}_{\mathbb{R}^{d_v}}
		&
		\textnormal{if $v$ is a leaf}.
	\end{cases}
\]
Figure~\ref{fig:Tree_construction} illustrates this construction, based on the network $\Psi$ depicted in Figure~\ref{fig:Base_NN_psi}. 

\subsection{Dropout neural networks induced by dropout-trees are close to their deterministic counterpart}

We will give an inductive argument that $\Phi^{v_0}_\Gamma$ is close to $\Phi^{v_0}_{\Gamma_{\mathsf{det}}}$.  Here, $\Gamma_{\mathsf{det}}$ denotes the same dropout-tree as $\Gamma$ except for the fact that we have replaced each and every filter variable deterministically by its expectation. Loosely speaking, the inductive argument implies  that dropout neural networks induced by dropout-trees are close to their deterministic counterparts.

As a technical preparation, we define a sequence of radii $R_0, R_1, \dots, R_L$. The idea is that these provide bounds on the output after applying several layers, no matter the
choice of filter variables or weights in the upcoming construction.
Given the radius $R > 0$ defined at the start of this section,  
we set 
\[
	R_0 
	:= \frac{Q}{\beta} R + 1
\]
and then choose $R_j$ inductively such that for all $j \in \intint{L}$, $x \in B(0, \allowbreak \beta^{-1} \|W^{(j)}\|_{\mathrm{HS}} R_{j-1} + 1)$ it holds that 
\begin{equation}
	\label{eq:definition_R_j}
	\Bigl|\Psi_j 
	\Bigl( 
		x, 
		\bigl( 
			I, b^{(j)} 
		\bigr) 
	\Bigr)\Bigr|
	< R_{j} - 1.
\end{equation}
for all $j = \intint{L}$, where $\beta \in (0,1)$ and $Q > 1$ were two of the global variables that we defined at the beginning of the section. Here $\|\cdot\|_{\mathrm{HS}}$ denotes the Hilbert--Schmidt norm of a matrix and $I$ denotes an identity matrix of the corresponding size.

We denote the input space to a network induced by a dropout-tree $\Gamma$ by $\inp_\Gamma$. That is, $\inp_\Gamma$ is the vector space 
\[
	(
		x^\ell \in \mathbb{R}^{d_{\level(\ell)}} 
		\ | \ 
		\ell \in \mathsf{leaves}(\Gamma) 
	) .
\]
We endow $\inp_\Gamma$ with the norm
\[
	\| (x^\ell) \|_{\inp_\Gamma} 
	:= \max_{l \in \mathsf{leaves}(\Gamma)} \|x^\ell \|_{\mathbb{R}^{d_{\level(\ell)}}}.
\]
We define $\transfer_{\Gamma}: \R d\to \inp_\Gamma$ to be the collection of functions, indexed by leaves $\ell $ of $\Gamma$, that are  generated by those layers in the base network $\Psi$ that are \emph{not} represented in $\Gamma$ at leaf $\ell $:
\begin{equation}
	\label{eqn: InGamma}
	\transfer_\Gamma^\ell(x) 
	:= \left(\Psi_{\level(\ell)} (\cdot,  (W^{(\level(\ell))}, b^{(\level(\ell))}) )\circ \dots \circ \Psi_1(\cdot, (W^{(1)}, b^{(1)}) )\right) (x).
\end{equation}
Note that by the definitions \eqref{eq:definition_R_j} of the radii $R_j$ we have 
\begin{equation}
\label{est:transfer-Rjs}
\transfer_\Gamma^\ell\bigl(\,\overline {B(0,R)}\,\bigr) \subset B(0,R_{\level(\ell)}-1).
\end{equation}

We say that a dropout-tree $\Gamma$ satisfies property $\propprob_\Gamma(\delta,\epsilon)$ if
\begin{equation}
	\mathbb{P}
	\Bigl[
		\sup_{x \in \overline{B(0, R)}} \sup_{\tilde{x} \in \overline{B(\transfer_\Gamma(x), \delta)}}
		\bigl| \Phi^{v_0}_\Gamma(\tilde{x}) - \Phi^{v_0}_{\Gamma_{\det}}(\transfer_{\Gamma}(x)) \bigr|
		> \frac{\epsilon}{2}
		\Bigr]
	< 
	\Bigl( \frac{\epsilon}{4 R_L} \Bigr)^q.
\label{de:approp}
\end{equation}
We will prove \refLemma{lem:Propprob_inheritance}: its message is that one can always construct a full dropout-tree that satisfies $\propprob_\Gamma(\delta,\epsilon)$ for some $\delta > 0$, by copying inputs at vertices. 
\begin{lemma}
	\label{lem:Propprob_inheritance}
	Let $\Gamma$ be a dropout-tree and let $\ell$ be a leaf of $\Gamma$ at level $k > 1$.
	Let $\mu$ be the distribution of a random matrix $F \in \{0,1\}^{ d_k \times d_{k-1} }$ that satisfies for all $r,c$, $\probability{ F_{rc} = 1 } \geq \beta > 0$.
	The following now holds:
	if $\Gamma$ satisfies $\propprob_\Gamma(\delta,\epsilon)$ in \eqref{de:approp} for some $\delta,\epsilon > 0$, then for every sufficiently large $\mu$-input-copy $\Gamma'$ at $\ell$ there exists a $\delta'> 0$ such that $\Gamma'$ satisfies $\propprob(\Gamma')(\delta',\epsilon)$.
\end{lemma}

The proof of \refLemma{lem:Propprob_inheritance} is relegated to Section~\ref{sec:Proof_of_lemma_propprob_inheritance}. There, we show that \refLemma{lem:Propprob_inheritance} follows from \refLemma{lem:copy-input}, which is displayed next and proved in Section~\ref{sec:Proof_of_lemma_copy-input}.

\begin{lemma}
	\label{lem:copy-input}
	Consider any continuous function $\sigma : \mathbb{R}^m \to \mathbb{R}^m$ and let $W \in \mathbb{R}^{m \times n}$, $b \in \mathbb{R}^m$. Let $\{ F^i \}_{i \geq 1}$  be a sequence of mutually independent copies of a random matrix $F \in \realNumbers^{m \times n}$ that satisfies: for $r \in \intint{m}$ and $c \in \intint{n}$, $0 < \expectation{ F_{rc} } < \infty$ and $0 \leq F_{rc} \leq M < \infty$ w.p.\ one. Let $V := W \div \mathbb{E}[F] $. The following now holds: for every $0 \leq K < \infty$ and $\rho > 0$ there exists a $\delta > 0$ such that
	\begin{equation}
		\mathbb{P}
		\Bigl[
			\sup_{x \in \overline{B(0, K)}} 
			\sup_{(\tilde{x}^i) \in \overline {B(x, \delta)}^N} 
			\Bigl|
				\sigma
				\Bigl(
					\frac1N \sum_{i=1}^N (V \odot F^i) \tilde{x}^i + b 
				\Bigr) 
				- \sigma( W x + b ) 
			\Bigr| 
			> \rho 
		\Bigr] 
		\to 0
		\label{eqn:Probability_bound_for_the_last_layer}
	\end{equation}
	as $N \to \infty$.
\end{lemma}

\subsection{Replacing the first layer}

Assume now that we have constructed a full dropout-tree, that is, a dropout-tree of which all leaves are at level $1$ (i.e., at depth $L-1$). 
This means that we have constructed suitable replacements for almost every layer of the neural network, except for the first layer. This layer contains the edges that have the global input as source.
Replacing the first layer requires a different construction: if we would outright drop edges in the first layer, then we can not control the error with the current technique. We now describe how we replace the first layer.

For every leaf $\ell$ in the full dropout-tree, we precompose every input at~$\ell$ with a stochastic function $\Xi^\ell: \mathbb{R}^{d_0} \to \mathbb{R}^{d_1}$. We record this information in what we call a \emph{precomposition} for a dropout-tree. Figure~\ref{fig:Tree_construction} illustrates this precomposition.

\begin{definition}
	A \emph{precomposition} for a full dropout-tree $\Gamma$ is a map $\Xi : \mathbb{R}^{d_0} \to \mathbb{R}^{d_1} $ from leaves to stochastic functions.
\end{definition}

Let $\Delta : \mathbb{R}^{d_0} \to (\mathbb{R}^{d_0})^{\mathsf{leaves}_\Gamma}$ be the diagonal map sending $x$ to copies of $x$. The neural network induced by the full dropout-tree $\Gamma$ and a precomposition $\Xi$ that we consider is given by 
\begin{equation}
	\indnn_{\Gamma, \Xi} 
	:= \Phi^{v_0}_{\Gamma, \Xi} \circ \Delta
	\label{eqn:Seed_of_the_recursion}
\end{equation} 
where
\begin{equation}
	\Phi^v_{\Gamma, \Xi}
	=
	\begin{cases}
		\sigma_v
		\Bigl(
			\frac{1}{ \# \into(v) }
			\sum_{e \in \into(v)} (V^e \odot F^e) \Phi^{\source(e)}_{\Gamma, \Xi} + b^e
		\Bigr)
		&
		\textnormal{if } v \textnormal{ is not a leaf},
		\\
		\Xi^v
		&
		\textnormal{if $v$ is a leaf}.
	\end{cases}	
	\label{eqn:Recursion_inducing_a_NN_of_a_full_dropout_tree_and_a_precomposition}
\end{equation}
We also define $\indnn^{\mathsf{avg-filt}}_{\Gamma, \Xi}$ as being almost the same neural network as $\indnn_{\Gamma, \Xi}$, with the only difference being that we replace each random filter variable $F^e$ in \eqref{eqn:Recursion_inducing_a_NN_of_a_full_dropout_tree_and_a_precomposition} with its expectation $\expectation{F^e}$. Recall for \eqref{eqn:Seed_of_the_recursion} that $v_0$ designates the root of the dropout-tree, and note that \eqref{eqn:Recursion_inducing_a_NN_of_a_full_dropout_tree_and_a_precomposition} constructs the neural network recursively (layer by layer).

\begin{figure}[p]
	\centering
\setstackgap{S}{2pt}

\def\sepx{1.8cm}
\def\sepy{1.0cm}
\def\boxwidth{0.5cm}
\def\boxheight{0.7cm}
\begin{tikzpicture}[node distance=\sepx,scale=0.8]

    \tikzstyle{every pin edge} = [<-,shorten <=1pt];
    \tikzstyle{neuron} = [circle,draw=black,fill=black!100,minimum size=9pt,inner sep=0pt];
    \tikzstyle{input neuron} = [neuron, fill=black!50];
    \tikzstyle{output neuron} = [neuron, fill=black!50];
    \tikzstyle{hidden neuron} = [neuron, fill=black!25];
    \tikzstyle{annot} = [text width=4em, text centered];

    \tikzstyle{box} = [
        draw =black!100,
        shape = rectangle,
        minimum width = \boxwidth,
        minimum height = \boxheight
    ];
    \tikzstyle{root}    = [box, fill=black!50];
    \tikzstyle{vertex}  = [box, fill=black!50];
    \tikzstyle{label}   = [text width=4em, text centered];
    
    \path node[box, pin={[pin edge={->}, pin distance=0.7cm]right:$\indnn_{\Gamma, \Xi}(x)$}] (R00000) at (-0*\sepx,1.0*\sepy)  {$\Shortstack{   .   }$};
    \path node[box] (A10000) at (-1*\sepx,-7.0*\sepy)  {$\Shortstack{ . . . }$};
    \path node[box] (A20000) at (-1*\sepx, 1.0*\sepy)  {$\Shortstack{ . . . }$};
    \path node[box] (A30000) at (-1*\sepx, 8.0*\sepy)  {$\Shortstack{ . . . }$};
    \path node[box] (A21000) at (-2*\sepx,-2.0*\sepy)  {$\Shortstack{. . . .}$};
    \path node[box] (A22000) at (-2*\sepx, 4.0*\sepy)  {$\Shortstack{. . . .}$};
    \path node[box] (A21100) at (-3*\sepx,-5.0*\sepy)  {$\Shortstack{  . .  }$};
    \path node[box] (A21200) at (-3*\sepx,-3.0*\sepy)  {$\Shortstack{  . .  }$};
    \path node[box] (A21300) at (-3*\sepx,-1.0*\sepy)  {$\Shortstack{  . .  }$};
    \path node[box] (A21400) at (-3*\sepx, 1.0*\sepy)  {$\Shortstack{  . .  }$};
    \path node[box] (L11110) at (-4*\sepx,-0.3*\sepy)  {$\Shortstack{   .   }$};
    \path node[box] (L11120) at (-4*\sepx,-1.7*\sepy)  {$\Shortstack{   .   }$};
    \path node[box, pin={[pin distance=0.7cm]left:$x$}] (D11111) at (-5*\sepx, 1.0*\sepy)  {$\Delta$};

    \tikzstyle{every edge} += [shorten <=2pt];
    \path node[box, draw=none] (A31000) at (-2*\sepx, 7.0*\sepy)  {};
    \path node[box, draw=none] (A32000) at (-2*\sepx, 9.0*\sepy)  {};
    \path node[box, draw=none] (A11000) at (-2*\sepx,-8.0*\sepy)  {};
    \path node[box, draw=none] (A12000) at (-2*\sepx,-6.0*\sepy)  {};

    \tikzstyle{every edge} += [shorten >=4pt, shorten <=3pt,->,draw=black!50, ultra thick];    
    \path (A20000) edge node[above] {} (R00000);
    \path (A30000) edge node[above] {} (R00000);
    \path (A10000) edge node[] {} (R00000);
    \path (A21000) edge node[above] {} (A20000);
    \path (A22000) edge node[above] {} (A20000);
    \path (A21100) edge node[above] {} (A21000);
    \path (A21200) edge node[above] {} (A21000);
    \path (A21300) edge node[above] {} (A21000);
    \path (A21400) edge node[above] {} (A21000);
	\path (L11110) edge node[above] {} (A21300);
	\path (L11120) edge node[above] {} (A21300);

    \tikzstyle{every edge} += [<-,dashed];    
	\draw (A22000) edge +(-0.8*\sepx,-1.3*\sepy);
	\draw (A22000) edge +(-0.8*\sepx,-0.5*\sepy);
	\draw (A22000) edge +(-0.8*\sepx, 0.5*\sepy);
	\draw (A22000) edge +(-0.8*\sepx, 1.3*\sepy);
    \path (A21100) edge +(-0.7*\sepx,-0.5*\sepy);
    \path (A21100) edge +(-0.7*\sepx, 0.5*\sepy);
    \path (A21200) edge +(-0.7*\sepx,-0.5*\sepy);
    \path (A21200) edge +(-0.7*\sepx, 0.5*\sepy);
    \path (A21400) edge +(-0.7*\sepx,-0.5*\sepy);
    \path (A21400) edge +(-0.7*\sepx, 0.5*\sepy);
	\tikzstyle{every edge} += [->];
    \path (A31000) edge node[near start, above] {} (A30000);
    \path (A32000) edge node[near start, above] {} (A30000);
    \path (A11000) edge node[near start, above] {} (A10000);
    \path (A12000) edge node[near start, above] {} (A10000);

    \tikzstyle{every edge} += [->,thin,solid];    
    \path (D11111) edge node[near start, above] {} (L11110);
    \path (D11111) edge node[near start, above] {} (L11120);
	\path (D11111) edge +(0.7*\sepx, 8*\sepy);
	\path (D11111) edge +(0.7*\sepx, 6*\sepy);
	\path (D11111) edge +(0.7*\sepx, 4*\sepy);
	\path (D11111) edge +(0.7*\sepx, 2*\sepy);
	\path (D11111) edge +(0.7*\sepx, 0*\sepy);
	\path (D11111) edge +(0.7*\sepx,-4*\sepy);
	\path (D11111) edge +(0.7*\sepx,-6*\sepy);
	\path (D11111) edge +(0.7*\sepx,-8*\sepy);

	\tikzstyle{vertline} = [dotted];
	\def\totheight{9*\sepy}
	\def\totdepth{-11*\sepy}
	\draw[vertline] ( 0.5*\sepx,\totheight) to ( 0.5*\sepx,\totdepth);
	\draw[vertline] (-0.5*\sepx,\totheight) to (-0.5*\sepx,\totdepth);
	\draw[vertline] (-1.5*\sepx,\totheight) to (-1.5*\sepx,\totdepth);
	\draw[vertline] (-2.5*\sepx,\totheight) to (-2.5*\sepx,\totdepth);
	\draw[vertline] (-3.5*\sepx,\totheight) to (-3.5*\sepx,\totdepth);
	\draw[vertline] (-4.5*\sepx,\totheight) to (-4.5*\sepx,\totdepth);
	
	\begin{scope}[shift={(0,-9)}]

    \path node[label] (l20) at (-0.0*\sepx, 0.0*\sepy) {$\sigma_4$};
    \path node[label] (l21) at (-1.0*\sepx, 0.0*\sepy) {$\sigma_3$};
    \path node[label] (l22) at (-2.0*\sepx, 0.0*\sepy) {$\sigma_2$};    
    \path node[label] (l23) at (-3.0*\sepx, 0.0*\sepy) {$\sigma_1$};
    \path node[label] (l24) at (-4.0*\sepx, 0.0*\sepy) {$\sigma_0$};
    \path node[label] (l25) at (-5.0*\sepx, 0.0*\sepy) {};    

    \path node[label] (l24) at (-3.5*\sepx,-0.25*\sepy) {$\underbrace{~~~~~~~~~~~~}{}$};
    \path node[label] (l24) at (-3.5*\sepx,-0.5*\sepy) {$\Xi$};
    \end{scope}
    
	\begin{scope}[shift={(0,-9.5)}]	
	\small
    \path node[label] (l30) at (-0.0*\sepx,-1.0*\sepy) {root};
    \path node[label] (l31) at (-1.0*\sepx,-1.0*\sepy) {level $3$};
    \path node[label] (l32) at (-2.0*\sepx,-1.0*\sepy) {level $2$};    
    \path node[label] (l33) at (-3.0*\sepx,-1.0*\sepy) {leaves at level $1$};
    \path node[label] (l34) at (-4.0*\sepx,-1.0*\sepy) {level $0$};
    \path node[label] (l35) at (-5.0*\sepx,-1.0*\sepy) {copying};
	\end{scope}

\end{tikzpicture}
	\caption{An illustration of how we use the base neural network $\Psi$ from \refFigure{fig:Base_NN_psi} and a dropout-tree (indicated with thicker arrows) to ultimately construct the larger neural network $\indnn_{\Gamma, \Xi}$ in which edges are being dropped stochastically. Here, $L=4$.}
	\label{fig:Tree_construction}
\end{figure}
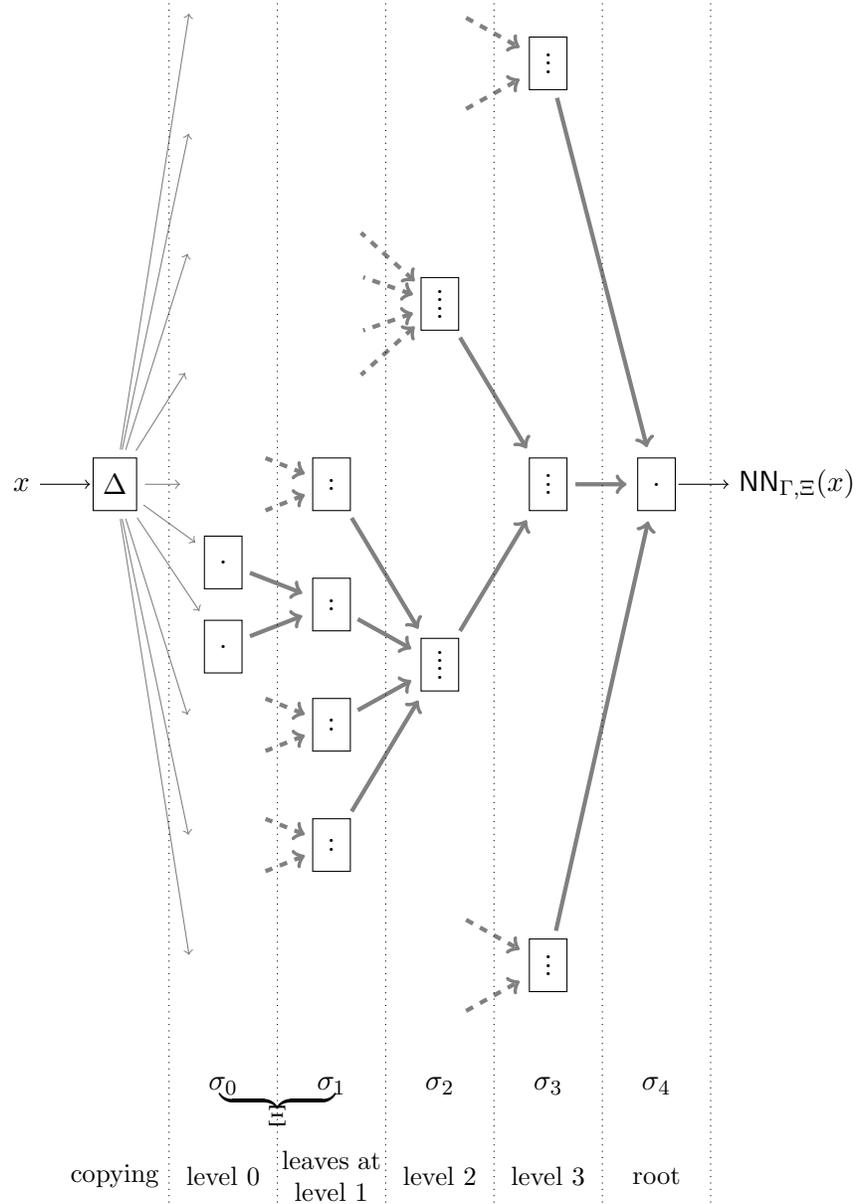

\begin{example}
	One natural precomposition $\Xi$ is the assignment of the function $\Psi^{1}( \cdot , (W^{(1)}, b^{(1)}))$ to every leaf. This precomposition yields a neural network in which edges in the first layer, i.e., the input edges of the neural network, are never dropped. In this case, $\indnn^{\mathsf{avg-filt}}_{\Gamma, \Xi}(w)$  actually coincides with the original neural network $\Psi( \cdot, w)$.
\end{example}

\subsubsection{Precompositions in which inputs are dropped}

We will now construct precompositions that allow for the possibility of dropping edges in the first layer and applying e.g.\ the \gls{ReLU} function to them immediately. Concretely, we will add a zeroth layer with an activation function~$\sigma_0 : \mathbb{R} \to \mathbb{R}$. We assume that $\sigma_0(0) = 0$ and that $\sigma_0$ has one-sided derivatives $\sigma_-$ and $\sigma_+$ in the point $0 \in \mathbb{R}$:
\begin{equation}
	\sigma_-
	:= \lim_{\alpha \downarrow 0} \frac{ \sigma_0( - \alpha) - \sigma_0(0) }{\alpha}
	,
	\quad
	\sigma_+
	:= \lim_{\alpha \downarrow 0} \frac{ \sigma_0(+ \alpha) - \sigma_0(0) }{\alpha}.
	\label{eqn:Directional_derivative}
\end{equation}
Define the sign function
\begin{equation}
	S(x) 
	=
	\begin{cases}
		-
		& 
		\textnormal{if } x < 0,
		\\
		+
		&
		\textnormal{if } x \geq 0
	\end{cases}
	\label{eqn:Condition_on_the_directional_derivative}
\end{equation}
component-wise. If $x \neq 0$, then $\sigma_{S(-x)} + \sigma_{S(x)} = \sigma_- + \sigma_+$ does not depend on $x$---a critical fact that we will leverage in our construction. 

\begin{example}
	Consider a zeroth layer that is the identity function, i.e., $\sigma_0(y) = y$. Then $\sigma_{\pm} = 1$.
	Choosing $\sigma_0$ as the identity function is allowed here, and means that the layer is not adapted.
\end{example}

\begin{example}
	Consider a zeroth layer with \gls{ReLU} activation function, $\sigma_0 := \mathrm{ReLU}(z) := z \indicator{z \geq 0}$. Then $\sigma_- = 0$ and $\sigma_+ = 1$.
\end{example}

Here are the precompositions that we employ: we call $\Xi$ an \emph{$(\alpha, N)$-pre\-comp\-o\-si\-tion associated to a set of distributions $\{ \mu^\ell, \nu^\ell \}_{l \in \mathrm{leaves}(\Gamma) }$} if for each leaf $\ell $,
\begin{equation}
	\Xi^\ell(x) 
	:= 
	\sigma_1
	\Bigl( 
		\frac{1}{N}
		\sum_{i=1}^{2N} 
		(-1)^i (V^\ell \odot F^{\ell ,i}) 
		\sigma_0
		\bigl(
			(-1)^i \alpha (I \odot G^{\ell ,i}) x
		\bigr) 
		+ b^\ell 
	\Bigr)
	\label{eqn:precomposition_associated_with_a_set_of_distributions}
\end{equation}
where element-wise
\begin{equation}
	V^\ell_{rc}
	= \frac{ W^{(1)}_{rc} }{ \alpha \bigl( \sigma_- + \sigma_+ \bigr) \,\expectation{F^\ell_{rc}}\, \expectation{G^\ell_{cc}} },
	\label{eqn:Definition_Vlrc_and_alr_component_wise}
\end{equation}
and $\{ F^{\ell,i} \}_{i \geq 1}$, $\{ G^{\ell ,i} \}_{i \geq 1}$ are sequences of mutually independent copies of random matrices $F^\ell, G^\ell$ that have distributions $\mu^\ell$, $\nu^\ell$, respectively. Furthermore, the $F^\ell$ are presumed to have the same size as $W^{(1)}$, and the $G^\ell$ to have size $d_0 \times d_0$. Note that these assumptions allow us to place unit mass on any particular outcome and thus to replace $F^\ell, G^\ell$ by deterministic counterparts.

The idea of \eqref{eqn:precomposition_associated_with_a_set_of_distributions} is that it represents two layers of a dropout neural network that satisfies the approximation properties we are after. The functions $\sigma_1$, $\sigma_0$ can be understood as their activation functions, the matrices $V^\ell$, $\alpha I$ as their weights, $b^\ell$ as a bias, and the matrices $F^{\ell }$, $G^{\ell }$ as describing which edges and inputs are randomly removed. By scaling the weights $V^\ell$ by $1/(2N)$ and generating $2N$ independent copies of the first layer, we are preparing for an application of the law of large numbers. Furthermore, by allowing for arbitrarily small $\alpha$, we are preparing for a linearization of $\sigma_0$ around $0$. Finally, the alternatingly positive and negative multiplicative factors $(-1)^i$ allow us to cover directional derivatives such as that of the \gls{ReLU} activation function. All together, the construction allows us to prove the following theorem.

\begin{theorem}
	\label{th:main-average}
	Fix\/ $0<\epsilon < 1$. 
	Let $\Gamma$ be a full dropout-tree satisfying $\propprob_\Gamma(\delta,\epsilon)$ for some $\delta > 0$. 
	Let $\Xi$ be an $(\alpha, N)$-precomposition associated to a set of distributions $\{ \mu^\ell, \nu^\ell \}_{\ell \in \mathrm{leaves}(\Gamma) }$. Assume that for every $\ell$, if $F$ is a matrix of filter variables distributing according to $\mu^\ell$ or $\nu^\ell$, then for every $r,c$,
	\[
		\mathbb{P}[F_{r c} = 1] 
		\geq \beta
		> 0
		.
	\]	
	Let $\sigma_0 : \mathbb{R} \to \mathbb{R}$ be a continuous function with one-sided derivatives $\sigma_-$ and $\sigma_+$ in $0$, such that $\sigma_- + \sigma_+ \neq 0$ and such that $\sigma_0(0) = 0$. Assume moreover that $\sigma_-$ and $\sigma_+$ satisfy the following inequality with respect to the global variable $Q$:
	\begin{equation}
		\label{eq:rel-dir-der-to-Q}
		4 \frac{|\sigma_-| + |\sigma_+|}{|\sigma_- + \sigma_+ |} 
		< Q.
	\end{equation}

	The following inequalities now hold for $\alpha > 0$ small enough and $N \in \mathbb{N}$ large enough:
	\begin{equation}
		\label{eq:main-average-prob}
		\mathbb{P}
		\Bigl[ \sup_{x \in \overline{B(0, R)}}
			\bigl| \indnn_{\Gamma, \Xi}(x) - \Psi(x,w) \bigr|
			> \epsilon
			\Bigr]
		< \epsilon
	\end{equation}
	and
	\begin{equation}
		\label{eq:main-average-Lq}
		\mathbb{E}
		\Bigl[
			\sup_{x \in \overline{B(0,R)}}
			\bigl| \indnn_{\Gamma, \Xi}(x) - \Psi(x,w) \bigr|^q
			\Bigr]^{1/q}
		< \epsilon
		,
	\end{equation}
	while
	\begin{equation}
		\label{eq:main-average-avg-filt}
		\sup_{x \in \overline{B(0, R)}}
		\Bigl| \indnn^{\mathsf{avg-filt}}_{\Gamma, \Xi}(x) - \Psi(x, w) \Bigr|
		< \epsilon.
	\end{equation}
\end{theorem}

An important consequence of \refTheorem{th:main-average} is that we obtain for instance a universal approximation result for dropconnect neural networks with \gls{ReLU} activation functions that also guarantees a good approximation when filter variables are replaced by their averages, as formalized by Corollary \ref{co:dropconnect-relu-average} in the introduction. \refTheorem{th:main-average} is proven in Section~\ref{sec:Proof_of_theorem_main-average}. There, we show that \refTheorem{th:main-average} follows from the following \refLemma{lem:Main-average_stepping_stone}, which in turn is proven in Section~\ref{sec:Proof_of_lemma_main-average_stepping_stone} using compactness arguments and the law of large numbers. 

\begin{lemma}
	\label{lem:Main-average_stepping_stone}
	Let $\sigma_0 : \realNumbers\to\realNumbers$ be a continuous function with $\sigma(0) = 0$ and with two one-sided derivatives $\sigma_-$ and $\sigma_+$ in $0$ satisfying $|\sigma_- + \sigma_+| > 0$.
	Let $\Xi$ be an $(\alpha, N)$-precomposition associated to a set of distributions $\{ \mu^\ell, \nu^\ell \}_{l \in \mathrm{leaves}(\Gamma) }$ such that 
		for all $\ell, r,c$, $\expectation{ F^\ell_{rc} } > 0, \expectation{ G^\ell_{rc} } > 0$, $0 \leq F^\ell_{rc} \leq M$ w.p.\ one, and $0 \leq G^\ell_{rc} \leq M < \infty$ w.p.\ one.				
	The following now holds: for every leaf $\ell$, $0 \leq K < \infty$, and $\rho > 0$, for $\alpha$ small enough and $N$ large enough,
	\begin{equation}
		\mathbb{P}
		\Bigr[
			\sup_{x \in \overline{B(0, K)}} 
			\bigl| \Xi^\ell(x) - \Psi_1(x ; (W^{(1)}, b^{(1)})) \bigr| 
			> \rho 
		\Bigr] 
		< \rho
		\label{eqn:Probability_bound_on_alpa_N_precompositions_associated_to_a_distribution}
	\end{equation}
	and
	\begin{equation}
		\sup_{x \in \overline{B(0,K)}} 
		\bigl|  \Xi^{\ell ,\mathsf{avg-filt}}(x) - \Psi_1(x ; (W^{(1)}, b^{(1)})) \bigr| 
		< \rho
		\label{eqn:sup_bound_on_alpa_N_precompositions_associated_to_a_distribution}
	\end{equation}
	where $\Xi^{\ell ,\mathsf{avg-filt}}$ denotes the function $\Xi^\ell$ in \eqref{eqn:precomposition_associated_with_a_set_of_distributions} but with each filter variable $F^{\ell ,i}$ replaced by its expectation $\expectation{F^\ell}$ .
\end{lemma}

\section{Discussion}
\label{sec:discussion}

In this article, we showed that dropout neural networks are rich enough for a universal approximation property to hold, both for random-approximation and expectation-replacement dropout.
It is further evidence that the representational capacity of neural networks is so large that approximations are possible despite significant additional constraints. In the case of dropout in general, these additional constraints are the implicit symmetry constraints enforced by the turning on and off of the filter variables: in dropconnect, for instance, for most realizations of the filter variables, the output of the dropconnect neural network still approximates the original neural network well after the filter variables are randomly permuted. Despite the enforced invariance with respect to this operation, there is enough room in the parameter space for the weights of the network to have a good approximation for the overwhelming majority of realizations of the filter variables.

Our proof of the universal approximation property for random-approx\-i\-ma\-tion dropout explicitly works with this symmetry.
By this, we mean the following.
The universal approximation property that we show even works when edges from the input nodes are dropped out at random.
The output in the first hidden layer is then inherently \emph{random}, and in no way close to deterministic.
This is in contrast with for instance the universal approximation property by \cite{foong2020expressiveness} in which the layers are all very close to deterministic. Yet even though the values in the nodes are random, we do have a good understanding of the \emph{distribution} of the values in the nodes, and two stochastic realizations are most likely almost permutations of each other. By blowing up the first layer, i.e., repeating it many times in parallel, we then know the output very well up to this permutation symmetry and this turns out to be enough for us to show a universal approximation property.

\subsection{Limitations of our results}

Our results and methods have several limitations.

We only show the \emph{existence} of dropout neural networks close to a given function. It is a completely separate question whether an algorithm such as dropout stochastic gradient descent would actually be able to find such an approximation. The main message of our result is that at least there is no theoretical obstruction to approximating functions with dropout neural networks.

In the proofs, we used very explicitly that filter variables only take on the values zero or one, while other forms of dropout also exist (for instance with Gaussian filter variables). Our algebraic proof does not readily generalize to this more general case, but it is possible that parts of the proof could be reused.

We also see that universal approximation goes hand in hand with blowing up the size of the neural network. As illustrated by the explicit computations in Section \ref{sec:explicit-computation}, one would likely need so many nodes that the constructions in the article are not feasible in practice.

\section{Conclusion}
\label{sec:conclusion}

We showed two types of universal approximation results for dropout neural networks, one for random-approximation dropout, in which case the random filter variables are also used at prediction time, and one for expectation-replacement dropout, in which case the filter variables are replaced by their averages at prediction time. Our results allow for dropout of edges from the input layer, allow for a wide class of distributions on filter variables, including dropout of edges from the input layer, and for a wide class of activation functions.

By making the difference between random-approximation and expectation-replacement dropout explicit, our results also highlight the following mystery: How is it that expectation-replacement dropout performs so well on prediction time?

\section*{Acknowledgments}

J.W.\ Portegies was supported by the Electronic Component Systems for European Leadership Joint Undertaking under grant agreement No 737459 (project Productive 4.0). This Joint Undertaking receives support from the European Union Horizon 2020 research and innovation program and Germany, Austria, France, Czech Republic, Netherlands, Belgium, Spain, Greece, Sweden, Italy, Ireland, Poland, Hungary, Portugal, Denmark, Finland, Luxembourg, Norway, Turkey.

\newpage
\bibliography{RelatedLiterature}

\begin{thebibliography}{25}
\providecommand{\natexlab}[1]{#1}
\providecommand{\url}[1]{\texttt{#1}}
\expandafter\ifx\csname urlstyle\endcsname\relax
  \providecommand{\doi}[1]{doi: #1}\else
  \providecommand{\doi}{doi: \begingroup \urlstyle{rm}\Url}\fi

\bibitem[Cybenko(1989)]{cybenko1989approximation}
George Cybenko.
\newblock Approximation by superpositions of a sigmoidal function.
\newblock \emph{Mathematics of control, signals and systems}, 2\penalty0
  (4):\penalty0 303--314, 1989.

\bibitem[De~Bie et~al.(2018)De~Bie, Peyr{\'e}, and Cuturi]{debie2018stochastic}
Gwendoline De~Bie, Gabriel Peyr{\'e}, and Marco Cuturi.
\newblock Stochastic deep networks.
\newblock \emph{arXiv preprint arXiv:1811.07429}, 2018.

\bibitem[Elbr{\"a}chter et~al.(2020)Elbr{\"a}chter, Perekrestenko, Grohs, and
  B{\"o}lcskei]{Grohs_UA_deep}
Dennis Elbr{\"a}chter, Dmytro Perekrestenko, Philipp Grohs, and Helmut
  B{\"o}lcskei.
\newblock Deep neural network approximation theory.
\newblock \emph{IEEE Transactions on Information Theory, submitted Jan. 2019,
  revised}, June 2020.
\newblock URL \url{http://www.nari.ee.ethz.ch/pubs/p/deep-it-2019}.

\bibitem[Foong et~al.(2020)Foong, Burt, Li, and
  Turner]{foong2020expressiveness}
Andrew Foong, David Burt, Yingzhen Li, and Richard Turner.
\newblock On the expressiveness of approximate inference in {B}ayesian neural
  networks.
\newblock \emph{Advances in Neural Information Processing Systems}, 33, 2020.

\bibitem[Gallicchio and Scardapane(2020)]{gallicchio2020deep}
Claudio Gallicchio and Simone Scardapane.
\newblock Deep randomized neural networks.
\newblock In \emph{Recent Trends in Learning From Data}, pages 43--68.
  Springer, 2020.

\bibitem[Gelenbe et~al.(1999{\natexlab{a}})Gelenbe, Mao, and
  Li]{gelenbe1999function}
Erol Gelenbe, Zhi-Hong Mao, and Yan-Da Li.
\newblock Function approximation with spiked random networks.
\newblock \emph{IEEE Transactions on Neural Networks}, 10\penalty0
  (1):\penalty0 3--9, 1999{\natexlab{a}}.

\bibitem[Gelenbe et~al.(1999{\natexlab{b}})Gelenbe, Mao, and
  Li]{gelenbe1999approximation}
Erol Gelenbe, Zhi-Wong Mao, and Yan-Da Li.
\newblock Approximation by random networks with bounded number of layers.
\newblock In \emph{Neural Networks for Signal Processing IX: Proceedings of the
  1999 IEEE Signal Processing Society Workshop (Cat. No. 98TH8468)}, pages
  166--175. IEEE, 1999{\natexlab{b}}.

\bibitem[Goodfellow et~al.(2016)Goodfellow, Bengio, and
  Courville]{GoodfellowBengioCourville16}
Ian Goodfellow, Yoshua Bengio, and Aaron Courville.
\newblock \emph{Deep learning}.
\newblock MIT press, 2016.

\bibitem[Hinton et~al.(2012)Hinton, Srivastava, Krizhevsky, Sutskever, and
  Salakhutdinov]{hinton2012improving}
Geoffrey~E. Hinton, Nitish Srivastava, Alex Krizhevsky, Ilya Sutskever, and
  Ruslan~R. Salakhutdinov.
\newblock Improving neural networks by preventing co-adaptation of feature
  detectors.
\newblock \emph{arXiv preprint arXiv:1207.0580}, 2012.

\bibitem[Hornik(1991)]{hornik1991approximation}
Kurt Hornik.
\newblock Approximation capabilities of multilayer feedforward networks.
\newblock \emph{Neural networks}, 4\penalty0 (2):\penalty0 251--257, 1991.

\bibitem[Igelnik and Pao(1995)]{igelnik1995stochastic}
Boris Igelnik and Yoh-Han Pao.
\newblock Stochastic choice of basis functions in adaptive function
  approximation and the functional-link net.
\newblock \emph{IEEE Transactions on Neural Networks}, 6\penalty0 (6):\penalty0
  1320--1329, 1995.

\bibitem[Labach et~al.(2019)Labach, Salehinejad, and
  Valaee]{LabachSalehinejadValaee19}
Alex Labach, Hojjat Salehinejad, and Shahrokh Valaee.
\newblock Survey of dropout methods for deep neural networks.
\newblock \emph{arXiv preprint arXiv:1904.13310}, 2019.

\bibitem[Leshno et~al.(1993)Leshno, Lin, Pinkus, and
  Schocken]{leshno1993multilayer}
Moshe Leshno, Vladimir~Ya Lin, Allan Pinkus, and Shimon Schocken.
\newblock Multilayer feedforward networks with a nonpolynomial activation
  function can approximate any function.
\newblock \emph{Neural networks}, 6\penalty0 (6):\penalty0 861--867, 1993.

\bibitem[Makovoz(1996)]{makovoz1996random}
Yuly Makovoz.
\newblock Random approximants and neural networks.
\newblock \emph{Journal of Approximation Theory}, 85\penalty0 (1):\penalty0
  98--109, 1996.

\bibitem[Makovoz(1998)]{makovoz1998uniform}
Yuly Makovoz.
\newblock Uniform approximation by neural networks.
\newblock \emph{Journal of Approximation Theory}, 95\penalty0 (2):\penalty0
  215--228, 1998.

\bibitem[Nguyen(2017)]{nguyen2017universal}
Hien Nguyen.
\newblock A universal approximation theorem for {G}aussian-gated mixture of
  experts models.
\newblock \emph{Available at SSRN 2946964}, 2017.

\bibitem[Nguyen et~al.(2016)Nguyen, Lloyd-Jones, and
  McLachlan]{nguyen2016universal}
Hien~D. Nguyen, Luke~R. Lloyd-Jones, and Geoffrey~J. McLachlan.
\newblock A universal approximation theorem for mixture-of-experts models.
\newblock \emph{Neural computation}, 28\penalty0 (12):\penalty0 2585--2593,
  2016.

\bibitem[Pao et~al.(1994)Pao, Park, and Sobajic]{pao1994learning}
Yoh-Han Pao, Gwang-Hoon Park, and Dejan~J. Sobajic.
\newblock Learning and generalization characteristics of the random vector
  functional-link net.
\newblock \emph{Neurocomputing}, 6\penalty0 (2):\penalty0 163--180, 1994.

\bibitem[Pinkus(1999)]{pinkus1999approximation}
Allan Pinkus.
\newblock Approximation theory of the {M}{L}{P} model in neural networks.
\newblock \emph{Acta numerica}, 8:\penalty0 143--195, 1999.

\bibitem[Rahimi and Recht(2008)]{rahimi2008uniform}
Ali Rahimi and Benjamin Recht.
\newblock Uniform approximation of functions with random bases.
\newblock In \emph{2008 46th Annual Allerton Conference on Communication,
  Control, and Computing}, pages 555--561. IEEE, 2008.

\bibitem[Timotheou(2010)]{timotheou2010random}
Stelios Timotheou.
\newblock The random neural network: a survey.
\newblock \emph{The computer journal}, 53\penalty0 (3):\penalty0 251--267,
  2010.

\bibitem[Wager et~al.(2013)Wager, Wang, and Liang]{wager2013dropout}
Stefan Wager, Sida Wang, and Percy~S Liang.
\newblock Dropout training as adaptive regularization.
\newblock In \emph{Advances in Neural Information Processing Systems}, pages
  351--359, 2013.

\bibitem[Wan et~al.(2013)Wan, Zeiler, Zhang, Le~Cun, and
  Fergus]{wan2013regularization}
Li~Wan, Matthew Zeiler, Sixin Zhang, Yann Le~Cun, and Rob Fergus.
\newblock Regularization of neural networks using {D}ropconnect.
\newblock In \emph{International Conference on Machine Learning}, pages
  1058--1066, 2013.

\bibitem[White(1989)]{white1989additional}
Halbert White.
\newblock An additional hidden unit test for neglected nonlinearity in
  multilayer feedforward networks.
\newblock In \emph{Proceedings of the international joint conference on neural
  networks}, volume~2, pages 451--455. Washington, DC, 1989.

\bibitem[Yin(2019)]{yin2019random}
Yonghua Yin.
\newblock Random neural network methods and deep learning.
\newblock \emph{Probability in the Engineering and Informational Sciences},
  pages 1--31, 2019.

\end{thebibliography}

\newpage
\appendix

\section{Proofs of Section~\ref{sec:Universal_approximation_for_random_approximation_dropout}}

\subsection{Proof of \refTheorem{thm:approximation_in_probability}}
\label{secappendix:proof_of_approximation_in_probability}
We require the following algebraic lemma, which lies at the heart of Theorem \ref{thm:approximation_in_probability}, as it implies the existence of the constants $(a_U)$.

\begin{lemma}
	\label{lemma:function_in_span_dropout_expectations}
	Let $\Psi : \mathbb{R}^d \times \mathbb{R}^n \to \mathbb{R}$ be a function.
	Let $(f^U)$ be a collection of $\{0,1\}^n$-valued random variables indexed by subsets $U \subset \intint{n}$, such that for every $U$
	\[
		\mathbb{P}[f^U = (1,\dots, 1)] > 0.
	\]
	Then for every subset $V\subset \intint{n}$, it holds that
	\[
	\begin{split}
		\Psi( \cdot , w \odot \mathbf{1}_V) 
		\in 
		\mathrm{span} \left\{\mathbb{E} \left[\Psi(\cdot, (w \odot \mathbf{1}_U) \odot f^U )\right] : U \in 2^n \right\}
	\end{split}
	\]
	where for any subset $S\in 2^d$, i.e., $S \subset \{1, \dots, d\}$, we denote by $\mathbf{1}_S$ the characteristic function of $S$.
\end{lemma}

\begin{proof}
	The proof is by induction on the cardinality of $V$ and follows from the equality
	\begin{align}
		\Bigl( 
			\sum_{S : V \subset S} \mathbb{P}[f^V = \mathbf{1}_S] 
		\Bigr)
		\Psi(\cdot, w \odot \mathbf{1}_V) 
		&
		= 
		\mathbb{E}
		\Bigl[ 
			\Psi(\cdot, (w \odot \mathbf{1}_V) \odot f^V ) 
		\Bigr]
		\nonumber \\ &
		\phantom{=}
		- \sum_{S : V \setminus S \neq \emptyset} \mathbb{P}[f^V = \mathbf{1}_S] \Psi(\cdot, w \odot \mathbf{1}_{S \cap V}).
	\end{align}
	In particular, for the base case in which $V$ is empty, the last term vanishes.
	In the induction step, the functions $\Psi(\cdot, w \odot \mathbf{1}_{S \cap V})$ are by the induction hypothesis all in the required span.
\end{proof}

\begin{proof}(of \refTheorem{thm:approximation_in_probability})
	By \refLemma{lemma:function_in_span_dropout_expectations}, we can find constants $a_{U}$ for $U \in 2^n$ such that \eqref{eqn:Decomposition_of_Psi_in_a_linear_combination_of_expectations} holds. We look now at \eqref{eqn:average_converges_in_prob_and_Lq}. By the law of large numbers, convergence in probability in the normed vector space $(\mathcal{F}, \| \cdot \|)$ follows. Moreover, for any $V \in 2^{n}$ we have
	\begin{equation}
		\pnorm{\Psi(\cdot, (w \odot \mathbf{1}_V) \odot f^V)}{\mathcal{F}} 
		\leq 
		\max_{U \in 2^n} \pnorm{\Psi(\cdot, (w \odot \mathbf{1}_U) \odot f^U)}{\mathcal{F}} 
		=: 
		C_{w},
	\end{equation}
	so that for any $q \in [1, \infty)$ and $M$, 
	\begin{equation}
		\mathbb{E} 
		\Bigl[ 
			\Bigl\| \frac{1}{M} \sum_{i=1}^M \sum_{U \in 2^n} a_{U} \Psi( \cdot, (w \odot \mathbf{1}_U) \odot f^{i, U}) \Bigr\|_{\subnorm}^q 
		\Bigr]^{\frac{1}{q}} 
		\leq 
		\max_{U \in 2^n} |a_{U}| C_w.
	\end{equation}
	With uniform boundedness for all $M$, we can use the dominated convergence theorem which implies then convergence in $L^{q}$ of the $\mathcal{F}$-valued random variables as $M \to \infty$.
\end{proof}

\subsection{Proof of \refCorollary{cor:main}}

\begin{proof}
	Let $\zeta \in \mathcal{F}$ and let $\epsilon > 0$. Assume there exists a $(m, \Phi, f) \in \DNN$ and a $v \in \mathbb{R}^m$ such that $\| \Phi(\cdot, v) - \zeta \|_{\subnorm} < \epsilon$.
	Define $\eta := \epsilon - \| \Phi(\cdot, v) - \zeta \|_{\subnorm} > 0$.
	Define the collection $( f^{U} )$ of $\{0,1\}^m$-valued filter variables, each being specifically an independent copy of $f$. By Theorem \ref{thm:approximation_in_probability}, there exist constants $(a_U)$, a number $M \in \mathbb{N}$ and $2^m M$ independent copies $(f^{i, U})$ of $f$ such that
	\[
		\mathbb{P}\Bigl[\Bigl\|\frac{1}{M} \sum_{i=1}^M  \sum_{U \in 2^m} a_U \Phi(\cdot, (v \odot \mathbf{1}_U) \odot f^{i, U}) - \Phi(\cdot, v) \Bigr\|_{\subnorm} > \eta \Bigr] < \eta
	\]
	and
	\[
		\mathbb{E}\Bigl[ \Bigl\|\frac{1}{M} \sum_{i=1}^M  \sum_{U \in 2^m} a_U \Phi(\cdot, (v \odot \mathbf{1}_U) \odot f^{i, U}) - \Phi(\cdot, v) \Bigr\|_{\subnorm}^q\Bigr]^{1/q} < \eta
		.
	\]
	Hence by the triangle inequality, in fact
	\[
		\mathbb{P}\Bigl[\Bigl\|\frac{1}{M} \sum_{i=1}^M  \sum_{U \in 2^m} a_U \Phi(\cdot, (v \odot \mathbf{1}_U) \odot f^{i, U}) - \zeta \Bigr\|_{\subnorm} > \epsilon \Bigr] < \epsilon
	\]
	and
	\[
		\mathbb{E}\Bigl[ \Bigl\|\frac{1}{M} \sum_{i=1}^M  \sum_{U \in 2^m} a_U \Phi(\cdot, (v \odot \mathbf{1}_U) \odot f^{i, U}) - \zeta \Bigr\|_{\subnorm}^q\Bigr]^{1/q} < \epsilon.
	\]
	We then  define the tuple $(n, \Psi, g) \in \DNN$ as an independent finite linear combination of $2^m M$ copies of $(m, \Phi, f)$, with coefficients $a_U/M$. Setting $\tilde v\in \R{2^mm}$ to be the concatenation of the $2^m$ modified vectors $(v\odot \mathbf 1_U)_{U\subset \overline{1,m}}$, we then set $w\in \R{2^mMm}$ to be the subsequent concatenation of $M$  copies of $\tilde v$. The combination of $(n,\Psi,g)$ and $w$ achieves the assertion.
\end{proof}

\subsection{Proof of \refProposition{prop:Coefficients_aU_under_independence_conditions}}
\label{sec:proof_explicit_computation}

As we have seen in Lemma \ref{lemma:function_in_span_dropout_expectations} , we can find the map $\Psi(\cdot, w)$ in the span of $\expectation{ \Psi( \cdot , (w \odot \mathbf{1}_V) \odot f^V }$ for $V \in 2^n$. We will look now at specific cases where we can explicitly compute the linear combination. In particular we will look in the case that we use \emph{dropout} \citep{hinton2012improving}, that is, we drop nodes independently with the same probability, and \emph{dropconnect} \citep{wan2013regularization}, where we drop individual weights independently with the same probability. 

In both cases the filter variables $f_{i}$ take the same values for some disjoint subsets of $ \overline{1,n}$, where $n$ is the number of weights where we apply filters. That is, if we have a disjoint set decomposition $\overline{1,n} = I_{1} \cup \ldots \cup I_{r}$ with $I_{k} \cap I_{s} = \emptyset$ whenever $k \neq s$, then $f_{i} = f_{j}$ for all $i,j \in I_{k}$ and $f_{i}, f_{l}$ are independent if they belong to disjoint sets, $f_{i} \in I_{k}$ and $f_{l} \in I_{s}$ with $s \neq k$. In this section we drop the index $U \in 2^{n}$ of the random variable $f^{U}$ for notational convenience as they are identically distributed. We will use this property to obtain an explicit decomposition in \eqref{eqn:Decomposition_of_Psi_in_a_linear_combination_of_expectations} and in a more general setting where the probability of the filters may differ depending on which disjoint set they belong to. For $K,S \in 2^n$, we denote $K \subseteq S$ to be the usual set inclusion, that is, $i \notin K$ whenever $i \notin S$ for all $i \in \overline{1,n}$ holds.

We need the following lemmas:

\begin{lemma}
	\label{lemma:example_mu_identity}
	Let $1 \geq q_1, \ldots, q_r > 0$ and $r \in \mathbb{N}$. For $K \in 2^{r}$, 
	\begin{equation}
		\mu_{K} 
		:= 
		\sum_{S: K \subseteq S} \prod_{i \in S} q_i \prod_{i \in \intint{r} \setminus S} (1 - q_i)
		\prod_{i \in K} \Bigl(\frac{1}{q_i}\Bigr) \prod_{i \in S \backslash K} \Bigl(1 - \frac{1}{q_i} \Bigr)
	\end{equation}
	satisfies
	\begin{equation}
		\mu_{K} = \begin{cases}
			1 & \text{ if } K = \intint{r} \\
			0 & \text{ otherwise.} \\
		\end{cases}
	\end{equation}
\end{lemma}

\begin{proof}
Let $x_1, \ldots, x_r$ be free variables. For $S \in 2^r$ we denote the monomial $x^S = \prod_{i \in S} x_i$. We will prove the identity by comparing coefficients of two equal polynomials. We have
\begin{equation}
q(x_1, \ldots, x_r) = \prod_{i=1}^r (q_i x_i + (1-q_i)) = \sum_{S \in 2^r} x^{S} \prod_{i \in S} q_{i} \prod_{i \in \intint{r} \setminus S} (1-q_{i})
\label{eqn:example_mu_identity_1}
\end{equation}
where we have expanded all $|2^r|$ monomials appearing in the decomposition of $q$. Now we set $x_i = (1/q_i) y_i - (1-q_i)/q_i$ in \eqref{eqn:example_mu_identity_1}. We have
\begin{align}
	q \Bigl( \frac{1}{q_1} y_1 - \frac{1-q_1}{q_1}, \ldots, \frac{1}{q_r} y_r - \frac{1-q_r}{q_r} \Bigr) 
	&
	= \prod_{i=1}^r \Bigl( q_i \Bigl( \frac{1}{q_i} y_i - \frac{1-q_i}{q_i}\Bigr) + (1-q_i)\Bigr) 
	\nonumber \\ &
	= \prod_{i=1}^r y_i 
	= y^{\intint{r}}.
\end{align} 
On the other hand, if we substitute $x_i = (1/q_i) y_i - (1-q_i)/q_i$ in the monomials $x^{S}$ in \eqref{eqn:example_mu_identity_1} we have
\begin{align}
	\sum_{S \in 2^r} x^{S} \prod_{i \in S} q_{i} \prod_{i \in \intint{r} \setminus S} (1-q_{i}) 
	&
	= \sum_{S \in 2^r} \prod_{i \in S} \Bigl( \frac{1}{q_i} y_i - \frac{1-q_i}{q_i} \Bigr) q_i \prod_{i \in \intint{r} \setminus S} (1-q_i) 
	\nonumber \\ &
	= \sum_{S \in 2^r} \sum_{K \in 2^r :K \subseteq S} \prod_{i \in S}  q_i \prod_{i \in \intint{r} \setminus S} (1-q_i) \prod_{i \in K } y_i \Bigl( \frac{1}{q_i}\Bigr) \prod_{i \in S \setminus K } \Bigl( - \frac{1-q_i}{q_i}\Bigr) 
	\nonumber \\ &
	= \sum_{K \in 2^r} y^{K} \sum_{S: K \subseteq S} \prod_{i \in S} q_i \prod_{i \in \intint{r} \setminus S} (1 - q_i) 
	\prod_{i \in K} \Bigl(\frac{1}{q_i}\Bigr) \prod_{i \in S \backslash K} \Bigl(1 - \frac{1}{q_i} \Bigr) \nonumber \\ &
	= \sum_{K \in 2^r}  \mu_{K} y^{K}
\end{align}
so that we must have $\mu_{K} = 1$ if $K = \intint{r}$ and zero otherwise.
\end{proof}

Let $\overline{1,n} = I_1 \cup \ldots, \cup I_{r}$ be a disjoint partition of $\overline{1,n}$, i.e., $I_{j} \cap I_{i} = \emptyset$ if $i \neq j$. We consider $S \in 2^{r}$ also as an element of $2^{n}$ via the inclusion $\iota: 2^r \to 2^n$ given by $j \in \iota(S)$ if $j \in I_{i}$ and $i \in S$, i.e., we consider the index $i$ as the set of all indices $j \in I_{i}$. Note then that $\iota(\intint{r}) = \intint{n}$. Recall now that the filter random variable with values in $\{0,1\}^{n}$ are denoted by $f = (f_1, \cdots, f_n)$ . We suppose now that the filter random variables satisfy that $f_i = f_j$ whenever $i,j \in I_s$ for some $s \in \overline{1,r}$. We denote by $B_{s}$ the $\{0,1\}$ valued random variable corresponding to the $I_{s}$ part of $\overline{1,n}$. We suppose that $\mathbb{P}(B_s = 1) = q_s = 1 - p_s$ for all $s \in \overline{1,r}$, where $q_s$ is the probability of success and $p_s$ the dropping probability. Moreover, we suppose that the $( B_s )_{s \in \intint{r}}$ are mutually independent. With this notation we have:
\begin{equation}
	\E{}\big[\Psi(\cdot,  w  \odot f)\big] 
	= 
	\sum_{L \in 2^{r}} \prod_{i \in L} q_i \prod_{i \in \intint{r} \setminus L } (1-q_i) \Psi(\cdot,  w \odot \mathbf{1}_{\iota(L)})
	.
\end{equation}

In the following Lemma, we embed $2^r$ into $2^n$ as blocks according to a partition of $\overline{1,n}$ using $\iota$:

\begin{lemma}
	For $S \in 2^{r}$, 
	\begin{align}
		\E{}\big[
			\Psi(\cdot,  (w \odot \mathbf{1}_{\iota(S)}) \odot f)
		\big]
		& 
		= 
		\sum_{K \in 2^{r}: K \subseteq S} \prod_{i \in K} q_i \prod_{i \in S \setminus K} (1 - q_i) \Psi\left(\cdot,  w \odot \mathbf{1}_{\iota(K)}\right)
		.
		\label{eqn:Expectation_of_Psi_w_g}
	\end{align}
	\label{lemma:example_expectation_subset}
\end{lemma}

\begin{proof}
	Let $S \in 2^r$. Observe that
	$
		f 
		= \sum_{s=1}^r \indicator{ B_s = 1 } \mathbf{1}_{I_s}
	$
	and note in particular that
	\begin{equation}
		g 
		:= \mathbf{1}_{\iota(S)} \odot f 
		= \bigl( \sum_{s \in S} + \sum_{s \in S^c} \bigr) \indicator{ B_s = 1 } \mathbf{1}_{\iota(S) \cap I_s}
		= \sum_{s \in S} \indicator{ B_s = 1 } \mathbf{1}_{I_s}.
		\label{eqn:Random_variable_g}
	\end{equation}
	Hence, $g$ depends only on $( B_s )_{s \in S}$ and is thus moreover independent of $( B_t )_{t \in S^c}$ by assumption. Consequently $\Psi(\cdot, w \odot g )$ also depends only on $( B_s )_{s \in S}$ and is also independent of $( B_t )_{t \in S^c}$. The result then follows. 

	To see this in detail, suppose that $S = \{ s_1, \ldots, s_m \}$ and $S^c = \{ t_1, \ldots, t_{r-m} \}$ say. Use the law of the unconscious statistician together with (i) independence to conclude that
	\begin{align}
		\expectation{ \Psi( \cdot, w \odot g ) }
		&		
		\eqcom{\ref{eqn:Random_variable_g}}=
		\sum_{b_1 = 0}^1 \cdots \sum_{b_r = 0}^1 \Psi( \cdot, w \odot \sum_{s \in S} \indicator{ b_s = 1 } \mathbf{1}_{I_s} ) \probability{ B_1 = b_1, \ldots, B_r = b_r }
		\nonumber \\ &
		\eqcom{i}=
		\sum_{b_{s_1} = 0}^1 \cdots \sum_{b_{s_m} = 0}^1 \Psi( \cdot, w \odot \sum_{i = 1}^m \indicator{ b_{s_i} = 1 } \mathbf{1}_{I_{s_i}} ) 
		\probability{ \cap_{i = 1}^m \{ B_{s_i} = b_{s_i} \} }
		\nonumber \\ &
		\phantom{=}
		\times
		\underbrace{ \sum_{b_{t_1} = 0}^1 \cdots \sum_{b_{t_{r-m}} = 0}^1 \probability{ \cap_{j=1}^{r-m} \{ B_{t_j} = b_{t_j} \} } }_{= 1 \textnormal{ as an axiom of the pdf}}.
	\end{align}
	Substitute 
	\begin{equation}
		\probability{ \cap_{i = 1}^m \{ B_{s_i} = b_{s_i} \} }
		\eqcom{i}= \prod_{i=1}^m	\probability{ B_{s_i} = b_{s_i} } 
		= \prod_{i=1}^m q_{s_i}^{b_{s_i}} (1-q_{s_i})^{1-b_{s_i}}
	\end{equation}
	and then apply the change of variables $K(b_{s_1},\ldots,b_{s_m}) = \cup_{i=1}^m \{ s_i : b_{s_i} = 1 \}$ to identify the right-hand side of \eqref{eqn:Expectation_of_Psi_w_g}.
\end{proof}

We can now prove \refProposition{prop:Coefficients_aU_under_independence_conditions}:

\begin{proof}(of \refProposition{prop:Coefficients_aU_under_independence_conditions})
	In the same notation as in \refLemma{lemma:example_mu_identity} and \refLemma{lemma:example_expectation_subset}, we use that $q_s = 1- p_s$ is the success probability. Then, we can write
	\begin{align}
		& 
		\sum_{V \in 2^r} \prod_{i \in V} \Bigl(\frac{1}{q_i}\Bigr) \prod_{i \in \intint{r} \setminus V} \Bigl(1 - \frac{1}{q_i} \Bigr) \E{} ( \Psi( \cdot , (w \odot \mathbf{1}_{\iota(V)}) \odot f)
		\nonumber \\
		& 
		\eqcom{\refLemma{lemma:example_expectation_subset}} = \sum_{V \in 2^r} \prod_{i \in V} \Bigl(\frac{1}{q_i}\Bigr) \prod_{i \in \intint{r} \setminus V} \Bigl(1 - \frac{1}{q_i} \Bigr) 
		\sum_{K: K \subseteq V \in 2^{r}} \prod_{i \in K} q_i \prod_{i \in V \setminus K} (1 - q_i) \Psi(\cdot,  w \odot \mathbf{1}_{\iota(K)}) 
		\nonumber \\
		& 
		= \sum_{V \in 2^r} \sum_{K: K \subseteq V \in 2^{r}}  \prod_{i \in V} \Bigl(\frac{1}{q_i}\Bigr) \prod_{i \in \intint{r} \setminus V} \Bigl(1 - \frac{1}{q_i} \Bigr)  \prod_{i \in K} q_i \prod_{i \in V \setminus K} (1 - q_i) \Psi(\cdot,  w \odot \mathbf{1}_{\iota(K)}) 
		\nonumber \\
		& 
		= \sum_{K \in 2^r} \Psi(\cdot,  w \odot \mathbf{1}_{\iota(K)}) \sum_{V: K \subseteq V \in 2^{r}} \Bigl(\frac{1}{q_i}\Bigr) \prod_{i \in \intint{r} \setminus V} \Bigl(1 - \frac{1}{q_i} \Bigr)  \prod_{i \in K} q_i \prod_{i \in V \setminus K} (1 - q_i)   
		\nonumber \\
		& 
		= \sum_{K \in 2^r} \Psi(\cdot,  w \odot \mathbf{1}_{\iota(K)}) \mu_{K} 
		\nonumber \\
		& 
		\eqcom{\refLemma{lemma:example_mu_identity}} =  \Psi(\cdot,  w \odot \mathbf{1}_{\iota(\intint{r})}) 
		\nonumber \\ 
		& 
		= \Psi(\cdot,  w).
	\end{align}
	Note finally that we obtain \refProposition{prop:Coefficients_aU_under_independence_conditions} after substituting $q_s = 1 - p_s$.
\end{proof}

\section{Proofs of Section~\ref{sec:Use_of_average_filter_variables_for_prediction}}

\subsection{Proof of \refLemma{lem:Propprob_inheritance}}
\label{sec:Proof_of_lemma_propprob_inheritance}

Let $\Gamma$ be a dropout tree. 
Let $\ell $ be a leaf of $\Gamma$ at level $k > 1$.
Let $\mu$ be the distribution of a random matrix $F \in \{0,1\}^{ d_k \times d_{k-1} }$ that satisfies for all $r,c$, $\probability{ F_{rc} = 1 } \geq \beta > 0$.
Assume that $\Gamma$ satisfies $\propprob_\Gamma(\delta,\epsilon)$, i.e.,
\[
	\mathbb{P}
	\Bigl[
	\sup_{x \in \overline{B(0,R)}} \sup_{\tilde{x}\in \overline{B(\transfer_\Gamma(x), \delta)}}
	\bigl| \Phi^{v_0}_\Gamma((\tilde{x}^\ell)_\ell ) - \Phi^{v_0}_{\Gamma_{\det}}(\transfer_\Gamma(x)) \bigr|
	> \frac{\epsilon}{2}
	\Bigr]
	< \Bigl( \frac{\epsilon}{4 R_L} \Bigr)^q.
\]
Define $\kappa > 0$ by
\[
	\kappa 
	:= 
	\Bigl( \frac{\epsilon}{4 R_L} \Bigr)^q - \mathbb{P}
	\Bigl[
	\sup_{x \in \overline{B(0,R)}} \sup_{\tilde{x}\in \overline{B(\transfer_\Gamma(x), \delta)}}
	\bigl| \Phi^{v_0}_\Gamma(x) - \Phi^{v_0}_{\Gamma_{\det}}(\tilde{x}) \bigr|
	> \frac{\epsilon}{2}
	\Bigr].
\]

By \refLemma{lem:copy-input}, there exists an $\eta > 0$ and an $N_0 \in \mathbb{N}$ such that for all $N \geq N_0$, if $F^i$ are independent, identically distributed filter matrices distributed according to $\mu$, and if $V$ is given by \eqref{eqn:Definition_of_Vegamma}, then 
\[
	\mathbb{P}
	\Bigl[
		\sup_{z \in \overline{B(0, R_k)}} \sup_{(\tilde{z})^i \in B(x, \eta)^N} 
		\Bigl|
		\sigma_k
		\bigl(\frac 1N
		\sum_{i=1}^N (V \odot F^i) \tilde{z}^i + b^{(k)}
		\bigr)
		- \sigma_k\bigl( W^{(k)} z + b^{(k)} \bigr)
		\Bigr| > \delta
	\Bigr] 
	< 
	\kappa.
\]
Now let $\Gamma'$ be a $\mu$-input-copy of $\Gamma$ at $\ell $ of size $N \geq N_0$. Choose $\delta' := \min (\delta, \eta)$.

Consider the event $\mathcal{A}$ that
\[
	\sup_{x \in \overline{B(0,R)}} \sup_{\tilde{x} \in \overline{B(\transfer_\Gamma(x), \delta)}}
	\bigl| 
		\Phi^{v_0}_\Gamma((\tilde{x}^\ell)_\ell ) - \Phi^{v_0}_{\Gamma_{\det}}(\transfer_\Gamma(x)) 
	\bigr|
	> \frac{\epsilon}{2},
\]
which (informally) means that the tree $\Gamma$ provides a bad approximation. Consider also the event $\mathcal{B}$ that 
\[
	\sup_{z \in \overline{B(0, R_{k-1})}} 
	\sup_{(\tilde{z}^m) \in B(z, \eta)^N} 
	\Bigl|
		\sigma_k
		\bigl(
			\sum_{e \in \into(\ell)} (V^e \odot F^e) \tilde{z}^m + b^e 
		\bigr) 
		- \sigma_k( W^e z + b^e ) 
	\Bigr| 
	> 
	\delta.
\]
Here, $\into(\ell)$ refers to the dropout-tree $\Gamma'$. This event (informally) means that the added part provides a bad approximation.
Note that
\[
	\mathbb{P}[\mathcal{A} \cup \mathcal{B}] \leq \mathbb{P}[\mathcal{A}] + \mathbb{P}[\mathcal{B}] 
	< \mathbb{P}[\mathcal{A}] + \kappa 
	= \Bigl( \frac{\epsilon}{4 R_L} \Bigr)^q.
\]
Next, let  us show that on $(\mathcal{A}\cup \mathcal{B})^c$ one has 
\begin{equation}
	\label{eqn: BoundOnComplementAB}
\sup_{x \in \overline{B(0,R)}} \sup_{\tilde{x}\in \overline{B(\transfer_{\Gamma'}(x), \delta')}}
\bigl| \Phi^{v_0}_{\Gamma'}((\tilde{x}^\ell)_\ell ) - \Phi^{v_0}_{\Gamma_{\det}}(\transfer_{\Gamma'}(x)) \bigr|
\leq \frac{\epsilon}{2}.
\end{equation}

To do this, suppose that $(\mathcal{A}\cup \mathcal{B})^c$ holds and $x\in \overline{B(0,R)}$. For every leaf $m \in \children(\ell)$ in $\Gamma'$ define $z:=\transfer_{\Gamma^{'}}^m(x)$ (recall the definition from \eqref{eqn: InGamma})  and note that $z\in \overline{B(0, R_{k-1})}$.
Let $\tilde{x}\in B(\transfer_{\Gamma'}(x), \delta')$. For every leaf $m \in \children(\ell)$ define $\tilde{z}^m:=\tilde{x}^m \in B(z,\eta)$ by the choice of $\delta'$.
Since $\mathcal{B}^c$ holds, 
\[
	\Bigl|
		\sigma_k
		\Bigl(
			\sum_{e \in \into(\ell)} (V^e \odot F^e) \tilde{z}^i + b^e 
		\Bigr) 
		- \sigma_k( W^e z + b^e ) 
	\Bigr| 
	\leq
	\delta,
\]
or, in other words,
\[
\sigma_k
\bigl(
\sum_{e \in \into(\ell)} (V^e \odot F^e) \tilde{z}^i + b^e 
\bigr) 
\in\overline{B(\transfer_{\Gamma}^\ell(x), \delta)}.
\]
Together with $\mathcal{A}^c$ this implies \eqref{eqn: BoundOnComplementAB}.

Finally, by the law of total probability, we estimate
\begin{align}
	&
	\mathbb{P}
	\Bigl[
		\sup_{x \in \overline{B(0,R)}} \sup_{\tilde{x}\in \overline{B(\transfer_{\Gamma'}(x), \delta')}}
		\bigl| \Phi^{v_0}_{\Gamma'}((\tilde{x}^\ell )_\ell ) - \Phi^{v_0}_{\Gamma_{\det}}(\transfer_{\Gamma'}(x)) \bigr|
		> \frac{\epsilon}{2}
	\Bigr]
	\nonumber \\ & 
	= \mathbb{P}
	\Bigl[
		\sup_{x \in \overline{B(0,R)}} \sup_{\tilde{x}\in \overline{B(\transfer_{\Gamma'}(x), \delta')}}
		\bigl| \Phi^{v_0}_{\Gamma'}((\tilde{x}^\ell )_\ell ) - \Phi^{v_0}_{\Gamma_{\det}}(\transfer_{\Gamma'}(x)) \bigr|
		> \frac{\epsilon}{2}
		\Big| \mathcal{A} \cup \mathcal{B} 
	\Bigr] 
	\mathbb{P}[\mathcal{A} \cup \mathcal{B}]
	\nonumber \\ &
	\phantom{=} 
	+ \mathbb{P}
	\Bigl[
		\sup_{x \in \overline{B(0,R)}} \sup_{\tilde{x}\in \overline{B(\transfer_{\Gamma'}(x), \delta')}}
		\bigl| \Phi^{v_0}_{\Gamma'}((\tilde{x}^\ell )_\ell ) - \Phi^{v_0}_{\Gamma_{\det}}(\transfer_{\Gamma'}(x)) \bigr|
		> \frac{\epsilon}{2}
		\Big| 
		(\mathcal{A} \cup \mathcal{B})^c 
	\Bigr] 
	\mathbb{P}[(\mathcal{A} \cup \mathcal
	{B})^c]
	\nonumber \\ &
	< \Bigl( \frac{\epsilon}{4 R_L} \Bigr)^q + 0 
	= \Bigl( \frac{\epsilon}{4 R_L} \Bigr)^q.
\end{align}
This completes the proof of \refLemma{lem:Propprob_inheritance}.

\QuodEratDemonstrandum

\subsection{Proof of \refLemma{lem:copy-input}}
\label{sec:Proof_of_lemma_copy-input}

Let $0 \leq K < \infty$ and $\rho > 0$ be fixed. 
For every $N < \infty$, the suprema in \eqref{eqn:Probability_bound_for_the_last_layer} over $x, (\tilde x^i)$ are in fact attained---say at $X_*, (\tilde{X}_*^i)$---because $\sigma$ is continuous and the optimization domain is closed and bounded. We need to now be careful because $X_*, (\tilde{X}_*^i)$ depend on the collection $\{ F^i \}_{i \in \intint{N}}$. 

Recall that the continuity of $\sigma$ implies that $\sigma$ is also uniformly continuous on each compact set, i.e., for every $\zeta > 0$ there exists an $\eta_\zeta > 0$ such that 
for all $x,y$ from this compact set
\begin{equation}
	\label{eqn: uniform_continuity}
	|x-y| < \eta_\zeta \Rightarrow |\sigma(y)- \sigma(x)| < \zeta.
\end{equation}

Define $\bar{X}_* := (1/N) \sum_{i=1}^N \tilde{X}_*^i$. Then uniform continuity of $\sigma$ implies that 
\begin{equation}
	\bigl| 
		\sigma \bigl( W \bar{X}_* + b \bigr) - \sigma( W x + b ) 
	\bigr|
	\leq 
	\sup_{x \in \overline{B(0, K)}} \sup_{y \in \overline{B(x,\delta)} } 
	\bigl| 
		\sigma \bigl( W y + b \bigr) - \sigma( W x + b )
	\bigr|
	=: 
	\gamma_{\delta} 
	< 
	\infty.
	\label{eqn:Continuity_of_sigma_applied_to_the_average_point_which_must_be_close_around_x}
\end{equation}
Moreover, $\gamma_{\delta}$ is independent of $N$. Remark also that
\begin{equation}
	\sup_{f \in \{0,1\}^{m \times n}} \sup_{ y \in \overline{B(0,\delta)} } 
		\frac{1}{\expectation{F}} \bigl| \bigl( W \odot f - W \odot \expectation{F} \bigr) y \bigr| 
	=: 
	c_\delta
	< \infty
	\label{eqn:Definition_of_c_delta}
\end{equation}
where $f$ stands for all possible deterministic realizations of the filters $F$. Again, $c_\delta$  is independent of $N$. 
Finally, by construction there exists a compact set $\mathcal{C} \subset \R m$ such that the points
\begin{equation}
	\frac{1}{N} \sum_{i=1}^N (V \odot F^i) \tilde{X}_*^i + b,
	\quad
	W \bar{X}_* + b
\end{equation}
lie in $\mathcal{C}$ with probability one.

First, fix $\zeta = \rho / 2$. From uniform continuity of $\sigma$ on the compact set $\mathcal{C}$ there exists $\eta_\zeta > 0$ such that \eqref{eqn: uniform_continuity} holds for all $x,y\in\mathcal{C}$.  Second, observe that $\gamma_{\delta} \to 0$ and $c_\delta \to 0$ as $\delta \to 0$. Hence we can choose $\delta$ and fix it such that
\begin{equation}
	0 
	< \zeta	
	< \rho - \gamma_\delta 
	\quad
	\textnormal{and}
	\quad
	c_\delta 
	< \eta_\zeta.
	\label{eqn:Choice_of_delta_for_Lemma_18}
\end{equation}

Combining \eqref{eqn:Continuity_of_sigma_applied_to_the_average_point_which_must_be_close_around_x} with the triangle inequality and using \eqref{eqn:Choice_of_delta_for_Lemma_18},  we arrive at
\begin{align}
	&
	\textnormal{LHS } \eqref{eqn:Probability_bound_for_the_last_layer}
	\leq 
	\probabilityBig{
		\underbrace{	
		\Bigl|
		\sigma
		\bigl(
			\frac{1}{N} \sum_{i=1}^N (V \odot F^i) \tilde{X}_*^i + b 
		\bigr) 
		- \sigma \bigl( W \bar{X}_* + b \bigr) 
		\Bigr| 
		}_{ =: Z }
		> \rho - \gamma_{\delta}.
	},
\end{align}
Consider now the event
\begin{equation}
	\mathcal{E} 
	=
	\Bigl\{
		\Bigl| 
			\frac{1}{N} \sum_{i=1}^N (V \odot F^i) \tilde{X}_*^i - W \bar{X}_* 
		\Bigr|
		< \eta_\zeta
	\Bigr\}.
	\label{eqn:Event_of_small_deviation}
\end{equation}
Then by the law of total probability and uniform continuity,
\begin{align}
	\probability{ Z > \rho - \gamma_{\delta} }
	&
	= 
	\probability{ Z > \rho - \gamma_{\delta} | \mathcal{E} } \probability{\mathcal{E}} +  \probability{ Z > \rho - \gamma_{\delta} | \mathcal{E}^{\mathrm{c}} } \probability{ \mathcal{E}^{\mathrm{c}} }
	\nonumber \\ &
	\leq 
	\indicator{ \zeta > \rho - \gamma_{\delta} } \probability{\mathcal{E}} + \probability{ \mathcal{E}^{\mathrm{c}} }
	\eqcom{\ref{eqn:Choice_of_delta_for_Lemma_18}}= 
	\probability{ \mathcal{E}^{\mathrm{c}} }.
	\label{eqn:Intermediate_Lemma_18__Law_of_total_probability}
\end{align}

We proceed by bounding $\probability{ \mathcal{E}^{\mathrm{c}} }$. Use the triangle inequality twice to establish that for any $x \in \overline{B(0,K)}$,
\begin{align}
	&
	\Bigl| 
		\frac1N \sum_{i=1}^N (V \odot F^i) \tilde{X}_*^i - W \bar{X}_* 
	\Bigr|
	\nonumber \\ &
	=
	\Bigl| 
		\frac{1}{N \expectation{F}} \sum_{i=1}^N \bigl( (W \odot F^i) - W \odot \expectation{F} \bigr) \bigl( x + ( \tilde{X}_*^i - x ) \bigr)
	\Bigr|
	\nonumber \\ &
	\leq
	\Bigl| 
		\frac{1}{N \expectation{F}} \sum_{i=1}^N \bigl( (W \odot F^i) - W \odot \expectation{F} \bigr) x
	\Bigr|	
	+ 
	\Bigl| 
		\frac{1}{N \expectation{F}} \sum_{i=1}^N \bigl( (W \odot F^i) - W \odot \expectation{F} \bigr) ( \tilde{X}_*^i - x )
	\Bigr|		
	\nonumber \\ &
	\eqcom{\ref{eqn:Definition_of_c_delta}}\leq 		
	\Bigl| 
		\frac{1}{N \expectation{F}} \sum_{i=1}^N \bigl( (W \odot F^i) - W \odot \expectation{F} \bigr) x
	\Bigr|			
	+ 
	c_\delta.		
	\label{eqn:Bound_on_the_filtered_optimizer}
\end{align}
Note now additionally that by Khinchin's weak law of large numbers,
\begin{equation}
	\frac{1}{N} \sum_{i=1}^N W \odot F^i 
	\overset{\mathbb{P}}\to
	W \odot \expectation{F}
	\quad
	\textnormal{as}
	\quad
	N \to \infty.
	\label{eqn:Khincins_WLLN_applied_to_sum_W_odot_Fi}
\end{equation}
Therefore, using \eqref{eqn:Choice_of_delta_for_Lemma_18}, we get 
as $N \to \infty$
\begin{equation}
	\probability{ \mathcal{E}^{\mathrm{c}} } 
	\eqcom{\ref{eqn:Bound_on_the_filtered_optimizer}}\leq \probabilityBig{  
		\Bigl| 
			\frac{1}{N \expectation{F}} \sum_{i=1}^N \bigl( (W \odot F^i) - W \odot \expectation{F} \bigr) x
		\Bigr|
		\geq
		\eta_\zeta - c_\delta 
		}
	\eqcom{\ref{eqn:Khincins_WLLN_applied_to_sum_W_odot_Fi}}\to 
	0
	\label{eqn:Intermediate_Lemma_18__Final_step}.
\end{equation}

Bounding \eqref{eqn:Intermediate_Lemma_18__Law_of_total_probability} by \eqref{eqn:Intermediate_Lemma_18__Final_step} completes the proof.
\QuodEratDemonstrandum

\subsection{Proof of \refTheorem{th:main-average}}
\label{sec:Proof_of_theorem_main-average}

We start by showing that for $\alpha$ small enough and for $N$ large enough,
\begin{equation}
	\label{eq:stronger_prob_bound_indnn}
	\mathbb{P}
		\Bigl[ \sup_{x \in \overline{B(0, R)}}
		\bigl| \indnn_{\Gamma, \Xi}(x) - \Psi(x,w) \bigr|
		> \frac{\epsilon}{2}
		\Bigr]
	< \Bigl(\frac{\epsilon}{4 R_L}\Bigr)^q.
\end{equation}
Afterwards, we deduce the three assertions \eqref{eq:main-average-prob}--\eqref{eq:main-average-avg-filt} from \eqref{eq:stronger_prob_bound_indnn}.

\medskip
\noindent
\emph{Proof of \eqref{eq:stronger_prob_bound_indnn}.}
Recall that the assumption $\propprob_\Gamma(\delta,\epsilon)$ means that
\[
	\mathbb{P}
	\Bigl[
		\sup_{x \in \overline{B(0, R)}} \sup_{\tilde{x} \in \overline{B(\transfer_\Gamma(x), \delta)}}
		\bigl| \Phi^{v_0}_{\Gamma,\Xi}(\tilde{x}) - \Phi^{v_0}_{\Gamma_{\det},\Xi}(\transfer_{\Gamma}(x)) \bigr|
		> 
		\frac{\epsilon}{2}
	\Bigr]
	< \Bigl( \frac{\epsilon}{4 R_L} \Bigr)^q.
\]
Define therefore $\kappa > 0$ by
\begin{equation}
	\kappa 
	:= 
	\frac{1}{\# \mathsf{leaves}(\Gamma)} 
	\Bigl( 
		\Bigl( \frac{\epsilon}{4 R_L} \Bigr)^q 
		- 
		\mathbb{P}
		\Bigl[
			\sup_{x \in \overline{B(0, R)}} \sup_{\tilde{x} \in \overline{B(\transfer_\Gamma(x), \delta)}}
			\bigl| \Phi^{v_0}_{\Gamma,\Xi}(\tilde{x}) - \Phi^{v_0}_{\Gamma_{\det},\Xi}(\transfer_{\Gamma}(x)) \bigr|
			> \frac{\epsilon}{2}
		\Bigr]
	\Bigr)
	.
	\label{eqn:Definition_of_kappa}
\end{equation}

Observe now that the function $\Psi_1$ is continuous, and the function $\Phi_{\Gamma_{\det}}^{v_0}$ is  continuous on $(\mathbb{R}^{d_1})^{\mathsf{leaves}(\Gamma)}$. Since this implies uniform continuity on compact sets (see \eqref{eqn: uniform_continuity}), there exists a $\zeta > 0$ such that whenever a function $g: \overline{B(0,R)} \to \inp_{\Gamma}$ satisfies
\[
	\sup_{x \in \overline{B(0, R)}}\;\sup_{\ell\in \mathsf{leaves}(\Gamma)}\;
	\bigl| 
		g^\ell (x) - \Psi_1(x ; (W^{(1)}, b^{(1)})) 
	\bigr| 
	< \zeta,
\]
then we also have
\begin{equation}
	\label{eq:bound_for_g_to_input}
	\sup_{x \in \overline{B(0, R)}}
	\Bigl| \Phi^{v_0}_{\Gamma_{\det}}(g(x)) - \Psi(x, w) \Bigr|
	< \epsilon.
\end{equation}

Now choose
\[
\rho := \min \left( \delta /2, \kappa, \zeta \right)
\qquad\text{and}\qquad
K := R,
\]
which we use as parameters for \refLemma{lem:Main-average_stepping_stone}. This choice ensures that for $\alpha$ small enough and $N$ large enough, for all leaves $\ell $ of $\Gamma$, 
by inequality (\ref{eqn:Probability_bound_on_alpa_N_precompositions_associated_to_a_distribution})
\begin{equation}
	\mathbb{P}
	\Bigr[
	\sup_{x \in \overline{B(0, R)}} 
	\bigl| \Xi^\ell (x) - \Psi_1(x ; (W^{(1)}, b^{(1)})) \bigr| 
	> \delta/2 
	\Bigr] 
	< \kappa
	\label{eqn:Theorem_15_proof_kappa_bound}
\end{equation}
and by inequality (\ref{eqn:sup_bound_on_alpa_N_precompositions_associated_to_a_distribution})
\begin{equation}
	\label{eq:deterministic-bound-for-avg-filt}
	\sup_{x \in \overline{B(0,R)}} 
	\bigl|  \Xi^{\ell ,\mathsf{avg-filt}}(x) - \Psi_1(x ; (W^{(1)}, b^{(1)})) \bigr| 
	< \zeta.
\end{equation}

Consider now the event $\mathcal{A}$ that there exists a leaf $\ell$ of $\Gamma$ such that
\[
	\sup_{x \in \overline{B(0, R)}} | \Xi^\ell (x) - \Psi_1(x ; (W^{(1)}, b^{(1)})) | > \delta/2
	.
\]
From the law of total probability, it follows that
\begin{align}
	&
	\mathbb{P}
	\Bigl[ 
		\sup_{x \in \overline{B(0, R)}}
		\bigl| \indnn_{\Gamma, \Xi}(x) - \Psi(x,w) \bigr|
		> \frac{\epsilon}{2}
	\Bigr]
	\nonumber \\ &
	= 
	\mathbb{P}
	\Bigl[
		\sup_{x \in \overline{B(0, R)}}
		\bigl| \indnn_{\Gamma, \Xi}(x) - \Psi(x,w) \bigr|
		> \frac{\epsilon}{2} 
		\Big| 
		\mathcal{A} 
	\Bigr] 
	\mathbb{P}[\mathcal{A}]
	\nonumber \\ &
	\phantom{=} 
	+ 
	\mathbb{P}
	\Bigl[
		\sup_{x \in \overline{B(0, R)}}
		\bigl| 
		\indnn_{\Gamma, \Xi}(x) - \Psi(x,w) \bigr|
		> 
		\frac{\epsilon}{2} \Big| \mathcal{A}^c 
	\Bigr] 
	\probability{ \mathcal{A}^c }
	.
	\label{eqn:Theorem_15_law_of_total_probability_applied_with_event_A}
\end{align}
Observe that
\begin{equation}
	\mathbb{P}
	\Bigl[
		\sup_{x \in \overline{B(0, R)}}
		\bigl| \indnn_{\Gamma, \Xi}(x) - \Psi(x,w) \bigr|
		> \frac{\epsilon}{2} 
		\Big| 
		\mathcal{A} 
	\Bigr] 
	\leq 1
	,	
	\label{eqn:Theorem_15__Intermediate_bound_1}
\end{equation}
and by (i) Boole's inequality
\begin{align}
	\probability{ \mathcal{A} }
	&
	= 
	\probabilityBig{ 
		\cup_{\ell \in \mathsf{leaves}(\Gamma)} 
		\bigl\{
			\sup_{x \in \overline{B(0, R)}} | \Xi^\ell (x) - \Psi_1(x ; (W^{(1)}, b^{(1)})) | > \delta/2 
		\bigr\}
	}
	\nonumber \\ &
	\eqcom{i}\leq 
	\sum_{\ell \in \mathsf{leaves}(\Gamma)} 
	\probabilityBig{ 
		\sup_{x \in \overline{B(0, R)}} | \Xi^\ell (x) - \Psi_1(x ; (W^{(1)}, b^{(1)})) | > \delta/2 
	}
	\eqcom{\ref{eqn:Theorem_15_proof_kappa_bound}}
	\leq \kappa \cdot \# \mathsf{leaves}(\Gamma)
	.
	\label{eqn:Theorem_15__Intermediate_bound_2}
\end{align}
Furthermore,
\begin{align}
	&
	\mathbb{P}
	\Bigl[
		\sup_{x \in \overline{B(0, R)}}
		\bigl| 
		\indnn_{\Gamma, \Xi}(x) - \Psi(x,w) \bigr|
		> 
		\frac{\epsilon}{2} \Big| \mathcal{A}^c 
	\Bigr] 
	\nonumber \\ &
	\eqcom{\ref{eq:bound_for_g_to_input}}\leq 
	\probabilityBig{
		\sup_{x \in \overline{B(0, R)}} \sup_{\tilde{x} \in \overline{B(\transfer_\Gamma(x), \delta)}}
		\bigl| \Phi^{v_0}_{\Gamma,\Xi}(\tilde{x}) - \Phi^{v_0}_{\Gamma_{\det},\Xi}(\transfer_{\Gamma}(x)) \bigr|
		> \frac{\epsilon}{2}
	}	
	.
	\label{eqn:Theorem_15__Intermediate_bound_3}
\end{align}
By bounding \eqref{eqn:Theorem_15_law_of_total_probability_applied_with_event_A} using \eqref{eqn:Theorem_15__Intermediate_bound_1}--\eqref{eqn:Theorem_15__Intermediate_bound_3} and $\probability{ \mathcal{A}^c } \leq 1$, we find that
\begin{align}
	&
	\mathbb{P}
	\Bigl[ \sup_{x \in \overline{B(0, R)}}
		\bigl| \indnn_{\Gamma, \Xi}(x) - \Psi(x,w) \bigr|
		> \frac{\epsilon}{2}
	\Bigr]
	\nonumber \\ &
	< \kappa \cdot \# \mathsf{leaves}(\Gamma) + 
	\probabilityBig{
		\sup_{x \in \overline{B(0, R)}} \sup_{\tilde{x} \in \overline{B(\transfer_\Gamma(x), \delta)}}
		\bigl| \Phi^{v_0}_{\Gamma,\Xi}(\tilde{x}) - \Phi^{v_0}_{\Gamma_{\det},\Xi}(\transfer_{\Gamma}(x)) \bigr|
		> \frac{\epsilon}{2}
	} \cdot 1
	\nonumber \\ &
	\eqcom{\ref{eqn:Definition_of_kappa}}= 
	\Bigl(\frac{\epsilon}{4 R_L}\Bigr)^q
\end{align}
This shows \eqref{eq:stronger_prob_bound_indnn}.

Next, we prove that \eqref{eq:main-average-prob}--\eqref{eq:main-average-avg-filt} follow from \eqref{eq:stronger_prob_bound_indnn}.

\medskip
\noindent
\emph{Proof of \eqref{eq:main-average-prob}.}
This inequality follows from \eqref{eq:stronger_prob_bound_indnn} since $R_L \geq 1$ by construction and $q \geq 1$ by assumption.

\medskip
\noindent
\emph{Proof of \eqref{eq:main-average-avg-filt}.} 
This inequality is a direct consequence of inequality \eqref{eq:bound_for_g_to_input} by choosing $g^\ell  := \Xi^{\ell ,\mathsf{avg-filt}}$ and using \eqref{eq:deterministic-bound-for-avg-filt}.

\medskip
\noindent
\emph{Proof of \eqref{eq:main-average-Lq}.}
We will prove that by the definition of $R_j$ in (\ref{eq:definition_R_j}), for all $x \in \overline{B(0, R)}$
\begin{equation}
	\bigl| \indnn_{\Gamma, \Xi}(x)\bigr|^q < R_L^q
	\quad
	\textnormal{w.p.\ one,}
	\quad
	\textnormal{and}
	\quad
	\bigl| \Psi(x,w) \bigr|^q < R_L^q.
	\label{eqn:Inductive_claim_for_the_proof_of_Theorem_15}
\end{equation}
Namely, if \eqref{eqn:Inductive_claim_for_the_proof_of_Theorem_15} holds true, then \eqref{eq:main-average-Lq} follows. 

To see the implication, consider the event $\mathcal{D}$ for which 
\begin{equation}
	\sup_{x \in \overline{B(0, R)}}
	\bigl| \indnn_{\Gamma, \Xi}(x) - \Psi(x,w) \bigr|
	> \frac{\epsilon}{2}
	,
	\label{eqn:Event_D_for_Theorem_15}
\end{equation}
and apply the law of total expectation:
\begin{align}
	\label{eq:computation_expectation_main_average}
	&
	\mathbb{E}
	\Bigl[
		\sup_{x \in \overline{B(0,R)}}
		\bigl| \indnn_{\Gamma, \Xi}(x) - \Psi(x,w) \bigr|^q
	\Bigr]
	\\ &
	= \mathbb{E}
	\Bigl[
		\sup_{x \in \overline{B(0,R)}}
		\bigl| \indnn_{\Gamma, \Xi}(x) - \Psi(x,w) \bigr|^q
		\Big| 
		\mathcal{D} 
	\Bigr] 
	\mathbb{P}[\mathcal{D}]
	+ 
	\mathbb{E}
	\Bigl[
		\sup_{x \in \overline{B(0,R)}}
		\bigl| \indnn_{\Gamma, \Xi}(x) - \Psi(x,w) \bigr|^q
		\Big| 
		\mathcal{D}^c 
	\Bigr] 
	\mathbb{P}[ \mathcal{D}^c ].
	\nonumber
\end{align}
By the triangle inequality,
\begin{equation}
	\mathbb{E}
	\Bigl[
		\sup_{x \in \overline{B(0,R)}}
		\bigl| \indnn_{\Gamma, \Xi}(x) - \Psi(x,w) \bigr|^q
		\Big| \mathcal{D} 
	\Bigr] 
	\eqcom{\ref{eqn:Inductive_claim_for_the_proof_of_Theorem_15}}\leq 
	(2 R_L)^q.
	\label{eqn:EsupcondD_bound}
\end{equation}
On the other hand,
\begin{equation}
	\mathbb{E}
	\Bigl[
		\sup_{x \in \overline{B(0,R)}}
		\bigl| \indnn_{\Gamma, \Xi}(x) - \Psi(x,w) \bigr|^q
		\Big| \mathcal{D}^c 
	\Bigr] 
	\eqcom{\ref{eqn:Event_D_for_Theorem_15}}\leq 
	\Bigl( \frac{\epsilon}{2} \Bigr)^q.
	\label{eqn:EsupcondDc_bound}
\end{equation}
Bound now \eqref{eq:computation_expectation_main_average} using \eqref{eq:stronger_prob_bound_indnn}, \eqref{eqn:EsupcondD_bound}, \eqref{eqn:EsupcondDc_bound}, and the elementary bound $\probability{ \mathcal{D}^c } \leq 1$ to obtain
\[
	\mathbb{E}
	\Bigl[
		\sup_{x \in \overline{B(0,R)}}
		\bigl| \indnn_{\Gamma, \Xi}(x) - \Psi(x,w) \bigr|^q
	\Bigr] 
	< 
	( 2 R_L )^q 
	\Bigl( \frac{\epsilon}{4 R_L} \Bigr)^q 
	+ \Bigl( \frac{\epsilon}{2} \Bigr)^q 
	\leq 
	\epsilon^q.
\]
That would prove \eqref{eq:main-average-Lq}. What remains is to prove \eqref{eqn:Inductive_claim_for_the_proof_of_Theorem_15}.

\medskip
\noindent
\emph{Proof of \eqref{eqn:Inductive_claim_for_the_proof_of_Theorem_15}.}
Observe immediately that the right inequality in \eqref{eqn:Inductive_claim_for_the_proof_of_Theorem_15} follows immediately as $0 < \beta < 1$ and $Q > 1$ (recall the definition of $\Psi$ in \eqref{eqn:definition_Psi}). Next, we will prove the left inequality in \eqref{eqn:Inductive_claim_for_the_proof_of_Theorem_15} by mathematical induction (recall the recursion in \eqref{eqn:Seed_of_the_recursion} and \eqref{eqn:Recursion_inducing_a_NN_of_a_full_dropout_tree_and_a_precomposition} that defines $\indnn_{\Gamma, \Xi}$).

\medskip
\noindent
\emph{Base case.}
Recall from \eqref{eqn:Seed_of_the_recursion} and \eqref{eqn:Recursion_inducing_a_NN_of_a_full_dropout_tree_and_a_precomposition} that the induction starts with the functions
\begin{equation}
	\Xi^\ell(x) 
	:= 
	\sigma_1
	\Bigl( 
	\frac{1}{N}
	\sum_{i=1}^{2N} 
	(-1)^i (V^\ell \odot F^{\ell ,i}) 
	\sigma_0
	\bigl(
	(-1)^i \alpha (I \odot G^{\ell ,i}) x
	\bigr) 
	+ b^\ell 
	\Bigr)
\end{equation}
where element-wise
\begin{equation}
	V^\ell_{rc}
	= 
	\frac{ W^{(1)}_{rc} }{ \alpha \bigl( \sigma_- + \sigma_+ \bigr) \mathbb{E}[F_{rc}^\ell] \expectation{G^\ell_{cc}} }
	.
\end{equation}
We are now going to prove that for every $x \in \overline{B(0, R)}$, the point
\begin{equation}
	\frac{1}{N}
	\sum_{i=1}^{2N} 
	(-1)^i (V^\ell \odot F^{\ell ,i}) 
	\sigma_0
	\bigl(
	(-1)^i \alpha (I \odot G^{\ell ,i}) x
	\bigr) 
	\in \overline{ B \bigl( 0, \beta^{-1} \|W^{(1)} \|_{\mathrm{HS}} R_0 \bigr) }
	\quad
	\textnormal{w.p.\ one}	
	.
	\label{eqn:Argument_in_ball}
\end{equation}
In particular, by the definition of $R_1$ in \eqref{eq:definition_R_j} (which implicitly deals with the bias $b^\ell$), this implies that for all $x \in \overline{B(0, R)}$,
\[
	\bigl|\Xi^{\ell} (x) \bigr| \leq R_1 - 1.
\]

Start by noting that there exists an $\alpha_0 >0$ such that for all $0 < \alpha \leq \alpha_0$ and all $\xi \in [-R, R] \subset \mathbb{R}$ we have
\begin{equation}
	\label{eq:lipschitz_like_bound}
	|\sigma_0( \alpha \xi ) | \leq 2 (|\sigma_-| + |\sigma_+|) \alpha |\xi|.
\end{equation}
It follows from the bound (\ref{eq:lipschitz_like_bound}) that for all $x \in \overline{B(0, R)}$ and all $i \in \intint{2N}$,
\[
	\Bigl| 
		\frac{1}{\alpha (\sigma_- +\sigma_+) \expectation{G^\ell_{cc}}} \sigma_0
		\bigl(
		\alpha (I \odot G^{\ell ,i}) x
		\bigr)_c
	\Bigr| 
	< 
	2 \frac{|\sigma_-| + |\sigma_+|}{|\sigma_- + \sigma_+|} \frac{1}{\beta} |x_c|
	\quad
	\textnormal{w.p.\ one}	
	.
\]
Since we assumed in \eqref{eq:rel-dir-der-to-Q} that
\[
4\frac{| \sigma_- | + |\sigma_+| }{|\sigma_- + \sigma_+|} < Q
\]
it follows that for all $x \in \overline{B(0, R)}$ and all $i \in \intint{2N}$,
\[
	\Bigl|
		2 \frac{\expectation{I \odot G^{\ell}}^{-1}}{\alpha (\sigma_- +\sigma_+)}  \sigma_0
		\bigl(
			\alpha (I \odot G^{\ell ,i}) x
		\bigr)
	\Bigr| 
	< \frac{Q}{\beta}R 
	< R_0
	\quad
	\textnormal{w.p.\ one}
	.
\]
Therefore, for every $x \in \overline{B(0, R)}$,
\[
	\begin{split}
		&
		\Bigl|
			\frac{1}{N}
			\sum_{i=1}^{2N} 
			(-1)^i (V^\ell \odot F^{\ell ,i}) 
			\sigma_0
			\bigl(
			(-1)^i \alpha (I \odot G^{\ell ,i}) x
			\bigr) 
		\Bigr|
		\\ & 
		= 
		\Bigl| 
			\frac{1}{2N}
			\sum_{i=1}^{2N} 
			(-1)^i ((W^{(1)} \odot F^{\ell,i}) \div \expectation{F^{\ell,i}}) \frac{2\expectation{I \odot G^{\ell,i}}^{-1}}{\alpha(\sigma_- + \sigma_+) }
			\sigma_0
			\bigl(
			(-1)^i \alpha (I \odot G^{\ell ,i}) x
			\bigr) 
		\Bigr|
		\\ &
		\leq 
		\frac{1}{2N} \sum_{i=1}^{2N} \| (W^{(1)} \odot F^{\ell,i}) \div \expectation{F^{\ell,i}} \|_{\mathrm{HS}} R_0 
		\leq \frac{1}{\beta} \|W^{(1)} \|_{\mathrm{HS}} R_0
		\quad
		\textnormal{w.p.\ one}		
		.
	\end{split}
\]
This proves \eqref{eqn:Argument_in_ball}.

\medskip
\noindent
\emph{Inductive step.}
In the definition of $\indnn_{\Gamma, \Xi}$ we defined for $v$ not a leaf in $\Gamma$,
\[
	\Phi^v_{\Gamma, \Xi} = \sigma_v
	\Bigl(
	\frac{1}{ \# \into(v) }
	\sum_{e \in \into(v)} (V^e \odot F^e) \Phi^{\source(e)}_{\Gamma, \Xi} + b^e
	\Bigr)
	.
\]
By an inductive argument we find that for all $x \in \overline{B(0, R)}$, it holds that
\[
	\Bigl| 
		\frac{1}{ \# \into(v) }
		\sum_{e \in \into(v)} (V^e \odot F^e) \Phi^{\source(e)}_{\Gamma, \Xi} (x) 
	\Bigr| 
	\leq 
	\beta^{-1} \|W^{e} \|_{\mathrm{HS}} R_{\level(v)- 1}
	\quad
	\textnormal{w.p.\ one}
	.
\]
so that by definition of $R_{\level{(v)}}$ it holds that
\[
	| \Phi^v_{\Gamma, \Xi}(x) | 
	< R_{\level{(v)}}
	\quad
	\textnormal{w.p.\ one}
	.
\]
In particular,
\[
	\bigl| \indnn_{\Gamma, \Xi}(x) \bigr|^q 
	= \bigl|\Phi^{v_0}_{\Gamma, \Xi}(x) \bigr|^q < R_L^q
	\quad
	\textnormal{w.p.\ one}
	.
\]
This proves \eqref{eqn:Inductive_claim_for_the_proof_of_Theorem_15}. With that, \refTheorem{th:main-average} is proven.
\QuodEratDemonstrandum

\subsection{Proof of \refLemma{lem:Main-average_stepping_stone}}
\label{sec:Proof_of_lemma_main-average_stepping_stone}

\noindent
\emph{Proof of \eqref{eqn:Probability_bound_on_alpa_N_precompositions_associated_to_a_distribution}.}
	Let $0 \leq K < \infty$, $\ell$ be a leaf of $\Gamma$, and $\rho > 0$. Recall that 
	\begin{equation}
		\Psi_1(x ; (W^{(1)}, b^{(1)}))
		= \sigma_1\bigl( W^{(1)} x + b^{(1)} \bigr),
	\end{equation}		
	and for $x \in \overline{B(0,K)}$, define
	\begin{equation}
		Z_N^\ell (x)
		=
		\frac{1}{N}
		\sum_{i=1}^{2N} 
		(-1)^i (V^\ell  \odot F^{\ell ,i}) 
		\sigma_0
		\bigl(
			(-1)^i \alpha (I \odot G^{\ell ,i}) x + b 
		\bigr) 
		+ b^\ell 
		\label{eqn:Definition_ZNx}
	\end{equation}
	so that $\Xi^\ell (x) = \sigma_1( Z_N^\ell (x) )$.
	Note that the weights $W$ are fixed and therefore uniformly bounded.

	Continuity of $\sigma_0$ and $\sigma_1$, boundedness of $F$ and $G$, positivity of  $\expectation{F^\ell _{rc}}$ and $\expectation{G^\ell _{rc}}$, and compactness of the optimization domain imply that  the supremum of the optimization problem is attained---say at $X_* \in \overline{B(0,K)}$. Just like in Appendix~\ref{sec:Proof_of_lemma_copy-input}, note that $X_*$ is random and depends on the collections $\{ F^{\ell ,i} \}_{i \in \intint{2N}}$, $\{ G^{\ell ,i} \}_{ i \in \intint{2N} }$. Summarizing:
	\begin{align}
		&
		\probabilityBig{ 
			\sup_{ x \in \overline{B(0,K)} } 
			\bigl| 
			\sigma_1( Z_N^\ell (x) ) - \sigma_1( W^{(1)} x + b^{(1)} ) 
			\bigr|
			> \rho
		}
		\nonumber \\ &
		=
		\probabilityBig{ 
			\bigl| 
			\sigma_1( Z_N^\ell (X_*) ) - \sigma_1( W^{(1)} X_* + b^{(1)} ) 
			\bigr|
			> \rho
		}.		
	\end{align}
	Note that here we slightly abuse the notation by using $|\cdot|$ sign not only for absolute value of numbers, but also, as in the last formula, for the Euclidean norm of the vector.

	By construction, there exists a compact set $\mathcal{C}$ so that the points
	\begin{equation}
		Z_N^\ell(X_*),
		\quad
		W^{(1)} X_* + b^{(1)}
	\end{equation}
	lie in $\mathcal{C}$ with probability one. The uniform continuity of  $\sigma_1$ on $\mathcal{C}$ implies that for each $\zeta>0$  there exists  $\eta_\zeta>0$ such that \eqref{eqn: uniform_continuity} holds for $\sigma_1$ and all $x,y\in \mathcal{C}$. Fix $\zeta= \rho$ and introduce the event 
	\begin{equation}
		\mathcal{D}(X_*)
		= 
		\bigl\{ 
			\|	
				Z_N^\ell (X_*) 
				- ( W^{(1)} X_* + b^{(1)} ) 
			\|_2
			< \eta_\rho
		\bigr\}.
	\end{equation}
	By the law of total probability 
	\begin{align}
		&
		\probabilityBig{		
			| 
				\sigma_1( Z_N^\ell (X_*) ) 
				- \sigma_1( W^{(1)} X_* + b^{(1)} ) 
			|
			> \rho
		}
		\nonumber \\ &
		=
		\probabilityBig{		
			| 
				\sigma_1( Z_N^\ell (X_*) ) 
				- \sigma_1( W^{(1)} X_* + b^{(1)} ) 
			|
			> \rho
			\,|\,
			\mathcal{D}(X_*)  
		}		
		\probabilityBig{ \mathcal{D}(X_*) }
		\nonumber \\ &		
		\phantom{=}
		+ 
		\probabilityBig{		
			| 
				\sigma_1( Z_N^\ell (X_*) ) 
				- \sigma_1( W^{(1)} X_* + b^{(1)} ) 
			|
			> \rho
			\,|\,
			\mathcal{D}^{\mathrm{c}}(X_*) 
		}		
		\probabilityBig{ \mathcal{D}^{\mathrm{c}}(X_*) }
		\leq 
		\probabilityBig{ \mathcal{D}^{\mathrm{c}}(X_*) }.
		\label{eqn:Upper_bound_after_law_of_total_probability}
	\end{align}
	We will next prove that for all $x \in \overline{B(0,K)}$, 
	\begin{equation}
		\probabilityBig{ \mathcal{D}^{\mathrm{c}}(x) }
		\to 0
	\end{equation}
	as $\alpha \downarrow 0$ and $N \to \infty$. Together with \eqref{eqn:Upper_bound_after_law_of_total_probability}, this implies the result.

	Let $x \in \overline{B(0,K)}$. Component-wise,
	\begin{align}
		&	
			\bigl(
				Z_N^\ell (x)
				- ( W^{(1)} x + b^{(1)} )
			\bigr)_r
		\label{eqn:After_symmetrical_split_these_are_the_absolute_components}			
		\\ &
		= 
			\Bigr(	
				\sum_{i=1}^{2N} 
				\frac{(-1)^i}{N}
				(V^\ell  \odot F^{\ell ,i}) 
				\sigma_0
				\bigl(
					(-1)^i \alpha (I \odot G^{\ell ,i}) x 
				\bigr)
				+ b^\ell 
				- ( W^{(1)} x + b^{(1)} )
			\Bigr)_r
		\nonumber \\ &
		=
			\sum_{i=1}^{2N}
			\sum_{c=1}^{d_0} 
			\frac{(-1)^i}{N}
			V^\ell _{rc} F^{\ell ,i}_{rc} 
			\sigma_0
			\bigl(
				(-1)^i \alpha (I \odot G^{\ell ,i}) x 
			\bigr)_{c}
			+ b^{(1)} _r
			- \sum_c W^{(1)}_{rc} x_c 
			+ b^{(1)}_r
		.	
		\nonumber	
	\end{align}	
	Substituting \eqref{eqn:Definition_Vlrc_and_alr_component_wise} into \eqref{eqn:After_symmetrical_split_these_are_the_absolute_components}, using the triangle inequality, and rearranging terms, we find that
	\begin{align}
		&
		\bigl| 
			\bigl(
				Z_N^\ell (x)
				- ( W^{(1)} x + b^{(1)} )
			\bigr)_r
		\bigr|
		\nonumber \\ &
		\leq
		\sum_{c=1}^{d_0} 
		\Bigl|	
			\frac{ W^{(1)}_{rc} }{ \tfrac{1}{2} ( \sigma_- + \sigma_+ ) \expectation{G^\ell _{cc}} }
			\Bigl(
				\tfrac{1}{2N}
				\sum_{i=1}^{2N}
					\frac{ F^{\ell i}_{rc} }{ \expectation{F^\ell _{rc}} }
					(-1)^i							
					\sigma_0
					\bigl( (-1)^i \alpha (I \odot G^{\ell ,i}) x 
					\bigr)_{c}
			\Bigr)
		- W^{(1)}_{rc} x_c 
		\Bigr| 
		.
		\label{eqn: bracketsIn24}		
	\end{align}	
	Note that the assumptions of the lemma imply that
	\[
		\Bigl|	
		\frac{ W^{(1)}_{rc} }{ \tfrac{1}{2} ( \sigma_- + \sigma_+ ) \expectation{G^\ell _{cc}} }
		\Bigr|
		\leq 
		C_{w,G,\sigma} 
		< 
		+\infty
		.
	\]

	We focus now on the term within brackets in \eqref{eqn: bracketsIn24}. Let $\delta_1 > 0$ and consider the event
	\begin{equation}
		\mathcal{E}_{F,N}(\delta_1)
		= 
		\Bigl\{ 
			\Bigl|
			\tfrac{1}{2N} \sum_{i=1}^{2N} \frac{ F^{\ell i} }{ \expectation{F^\ell _{rc}} }	
			- 1
			\Bigr|
			< \delta_1
		\Bigr\}.
	\end{equation}
	There exists $C_1>0$ such that, conditional on $\mathcal{E}_{F,N}(\delta_1)$,
	\begin{align}
		\Bigl|
			&
			\tfrac{1}{2N}
			\sum_{i=1}^{2N} 
			\frac{ F^{\ell i} }{ \expectation{F^\ell _{rc}} }
			(-1)^i
			\sigma_0
			\bigl(
				(-1)^i \alpha (I \odot G^{\ell ,i}) x 
			\bigr)_c
			-
			\tfrac{1}{2N} 
			\sum_{i=1}^{2N} (-1)^i
			\sigma_0
			\bigl(
				(-1)^i \alpha (I \odot G^{\ell ,i}) x 
			\bigr)_c	
		\Bigr|
		\nonumber \\ &		
		\leq 
		\Bigl|
			\tfrac{1}{2N}
			\sum_{i=1}^{2N} 
			\Bigl(\frac{ F^{\ell i} }{ \expectation{F^\ell _{rc}} }-1\Bigr)
			(-1)^i
			\sigma_0
			\bigl(
			(-1)^i \alpha (I \odot G^{\ell ,i}) x 
			\bigr)_c 
		\Bigr|
		\leq  
		C_1 \delta_1 
	\end{align}
	since the argument of $\sigma_0$ is uniformly bounded, and $\sigma_0$ is continuous. Moreover, there exists $C_2>0$ such that  conditional on  $\mathcal{E}_{F,N}(\delta_1)$,
	\begin{align}
		&
		\bigl| 
			\bigl(
				Z_N^\ell (x)
				- ( W^{(1)} x + b^{(1)} )
			\bigr)_r
		\bigr|
		\\ &
		\leq
		\sum_{c=1}^{d_0} 
		\Bigl|	
			\frac{ W^{(1)}_{rc} }{ \tfrac{1}{2} ( \sigma_- + \sigma_+ ) \expectation{G^\ell _{cc}} }
			\Bigl(
				\tfrac{1}{2N}
				\sum_{i=1}^{2N} 
				\tfrac{1}{\alpha}
				\Bigl(
					(-1)^i					
					\sigma_0
					\bigl(
						(-1)^i \alpha (I \odot G^{\ell ,i}) x 
					\bigr)_{c}
				\Bigr)
			\Bigr)
			- W^{(1)}_{rc} x_c 
		\Bigr| 					
		+ C_2 \delta_1.
		\nonumber
	\end{align}
	Note that  $C_1, C_2$ are independent of $N, \delta_1$.

	Recall now that by \eqref{eqn:Directional_derivative} and \eqref{eqn:Condition_on_the_directional_derivative} we can find for each $\gamma > 0$ an $\alpha > 0$ such that for all $y \in \realNumbers^{d_0}$, $c$,
	\begin{equation}
		\Bigl|
			\frac{1}{\alpha} 
			\Bigl( 
				\sigma_0 \bigl(\alpha y \bigr) 
			\Bigr)_c
			- \sigma_{S(y_r)} y_c
		\Bigr|
		< \gamma.
	\end{equation}
	Recall furthermore that $\max_{rc} G^\ell _{rc} \leq M < \infty$ with probability one by assumption. Together, this implies that we can find  for each $\gamma > 0$ an $\alpha > 0$ such that for all $x \in \overline{B(0,K)}$, $c$,
	\begin{equation}
		\probabilityBig{\Bigl|
			\frac{1}{\alpha} 
			\Bigl( 
				\pm
				\sigma_0 \bigl(  \pm \alpha I \odot G^{\ell ,i} x \bigr)  
			\Bigr)_c
			-
			\sigma_{S( \pm G^{\ell i}_{cc} x_c )} G^{\ell i}_{cc} x_c
		\Bigr|
		< \gamma}=1
		\label{eqn:Derivative_bound_wp_one}
	\end{equation}		
	
	Fix $\gamma \in (0,\delta_1)$ and corresponding $\alpha > 0$. Then there exists a constant $C_3>0$, independent of $\delta_1, \gamma, N$,  such that conditional on $\mathcal{E}_{F,N}(\delta_1)$,
	\begin{align}
		&
		\bigl| 
			\bigl( Z_N^\ell (x)
			- ( W^{(1)} x + b^{(1)} ) \bigr)_r
		\bigr|
		\\ &
		\leq
		\sum_{c=1}^{d_0} 
		\Bigl|	
			\frac{ W^{(1)}_{rc} }{ \tfrac{1}{2} ( \sigma_- + \sigma_+ ) \expectation{G^\ell _{cc}} }
			\Bigl(
				\tfrac{1}{2N}
				\sum_{i=1}^{2N}
				\sigma_{S( (-1)^i G^{\ell i}_{cc} x_c )} G^{\ell i}_{cc} x_c
			\Bigr)
			- W^{(1)}_{rc} x_c 
		\Bigr| 					
		+ C_3 \delta_1 
		\nonumber \\ &
		\eqcom{i}= 
		\sum_{c\in \overline{1, d_0} : x_c > 0}
		\Bigl|	
			\frac{ W^{(1)}_{rc} }{ \tfrac{1}{2} ( \sigma_- + \sigma_+ ) \expectation{G^\ell _{cc}} }
			\Bigl(
				\tfrac{1}{2N}
				\sum_{i=1}^{2N}
				\sigma_{S( (-1)^i G^{\ell i}_{cc} x_c )} G^{\ell i}_{cc} x_c
			\Bigr)
			- W^{(1)}_{rc} x_c 
		\Bigr| 					
		+ C_3 \delta_1 .
		\nonumber 		
	\end{align}
	To conclude (i), we used the fact that $\sigma_\pm \cdot 0 = 0$. By assumption $G^{\ell i}_{cc} \geq 0$ with probability one, so if moreover $| x_c | > 0$, then  $S( (-1)^i G^{\ell i}_{cc} x_c ) = S( (-1)^i x_c)$ with probability one---recall its definition in \eqref{eqn:Condition_on_the_directional_derivative}. 
	Thus there exists $C_4$ independent of $\delta_1,\gamma, N$ such that  conditional on the event $\mathcal{E}_{F,N}(\delta_1) \cap \mathcal{E}_{G,N}(\delta_1)$,
		\begin{align}
			&
			\bigl| 
				Z_N^\ell (x)
				- ( W^{(1)} x + b^{(1)} )
			\bigr|_r
			\nonumber \\ &
			\leq
			\sum_{c  \in \overline{1, d_0}: | x_c | > 0}
			\Bigl|	
				\frac{ W^{(1)}_{rc} }{ \tfrac{1}{2} ( \sigma_- + \sigma_+ ) } 
				\frac{1}{2N} \sum_{i=1}^{2N} \sigma_{S( (-1)^i x_c)} x_c
				- W^{(1)}_{rc} x_c 
			\Bigr| 	
			+ C_4 \delta_1
			= C_4 \delta_1.
		\end{align}
	The last equality holds because  the sum is only of over $c$ such that $|x_c| > 0$.

	All that remains is to prove that 
	\begin{equation}
		\probability{ \mathcal{E}_{F,N}(\delta_1) \cap \mathcal{E}_{G,N}(\delta_1) } 
		\to 1 
		\quad
		\textnormal{as}
		\quad
		N \to \infty.
	\end{equation}
	This fact follows immediately from the independence of $F, G$, and a subsequent application of Khinchin's weak law of large numbers (which may be applied since $F, G$'s expectations are bounded).
	Note that $\delta_1$ is an arbitrary parameter: choosing it such that $C_4 \delta_1 < \eta_\zeta$, and then choosing $N$ sufficiently large completes the proof of \eqref{eqn:Probability_bound_on_alpa_N_precompositions_associated_to_a_distribution}.

	\medskip	
	\noindent
	\emph{Proof of \eqref{eqn:sup_bound_on_alpa_N_precompositions_associated_to_a_distribution}.}
	The assertion \eqref{eq:main-average-Lq} is proven for any $(\alpha,N)$-precomposition associated with some distributions $\mu^\ell, \nu^\ell$  with finite nonzero mean.  In particular, the same argument shows that \eqref{eq:main-average-Lq} holds when $F^\ell$ and $G^\ell$ are taken deterministic and equal to the expectations of the corresponding random variables (see also the discussion following \eqref{eqn:Definition_Vlrc_and_alr_component_wise}).
	This proves \eqref{eqn:sup_bound_on_alpa_N_precompositions_associated_to_a_distribution}.
\QuodEratDemonstrandum

\end{document}